\documentclass[ejsv2]{imsart}
\usepackage{algorithmic}  
\usepackage{algorithm}        
\usepackage{microtype}
\usepackage{booktabs} 
\usepackage{mathtools}
\usepackage{dsfont} 
\usepackage{caption}
\usepackage{subcaption}
\usepackage{bbm}

\usepackage{appendix}
\RequirePackage[numbers]{natbib}
\RequirePackage[colorlinks,citecolor=blue,urlcolor=blue]{hyperref}
\RequirePackage{graphicx}

\startlocaldefs
\numberwithin{equation}{section}
\theoremstyle{plain}                 
\theoremstyle{plain}

\newtheorem{theorem}{Theorem}[section]
\newtheorem{lemma}[theorem]{Lemma}
\newtheorem{proposition}[theorem]{Proposition}
\newtheorem{corollary}[theorem]{Corollary}
\theoremstyle{definition}            
\theoremstyle{definition}
\newtheorem{definition}[theorem]{Definition}
\newtheorem{condition}[theorem]{Condition}

\theoremstyle{remark}

\DeclareMathOperator*{\argminA}{argmin}
\DeclareMathOperator{\Var}{Var}
\newcommand\myeq{\mathrel{\overset{\makebox[0pt]{\mbox{\normalfont\tiny\sffamily def}}}{=}}}

\newcommand\myeqD{\mathrel{\overset{\makebox[0pt]{\mbox{\normalfont\tiny\sffamily d}}}{\sim}}}

\newcommand\myApproxD{\mathrel{\overset{\makebox[0pt]{\mbox{\normalfont\tiny\sffamily d}}}{\approx}}}

\newcommand{\R}{\mathbb{R}^{d}}

\newcommand{\w}{\textbf{w}}

\endlocaldefs

\begin{document}
\begin{frontmatter}
\title{High-Dimensional Change Point Detection via Graph Spanning Ratios}
\runtitle{High-dimensional CPD using graph spanning ratio}

\begin{aug}
\author[A]{\fnms{Katerina}~\snm{Papagiannouli}\ead[label=e1]{aikaterini.papagiannouli@unipi.it}},
\author[B]{\fnms{Yangwen}~\snm{Sun}\ead[label=e2]{yangwen.sun@hu-berlin.de}}
\and
\author[C]{\fnms{Vladimir}~\snm{Spokoiny}\ead[label=e3]{spokoiny@wias-berlin.de}}
\address[A]{Department of Mathematics,
University or Pisa\printead[presep={,\ }]{e1}}

\address[B]{Department of Mathematics,
Humboldt University of Berlin\printead[presep={,\ }]{e2}}
\address[C]{Weierstrass Institut\printead[presep={,\ }]{e3}}
\runauthor{K. Papagiannouli et al.}
\end{aug}

\begin{abstract}
    \sloppy
    Inspired by graph-based methodologies, we introduce a novel graph-spanning algorithm designed to identify changes in both offline and online data across low to high dimensions. This versatile approach is applicable to Euclidean and graph-structured data with unknown distributions, while maintaining control over error probabilities. Theoretically, we demonstrate that the algorithm achieves high detection power when the magnitude of the change surpasses the lower bound of the minimax separation rate, which scales on the order of $\sqrt{nd}$. Our method outperforms other techniques in terms of accuracy for both Gaussian and non-Gaussian data. Notably, it maintains strong detection power even with small observation windows, making it particularly effective for online environments where timely and precise change detection is critical.
\end{abstract}
\end{frontmatter}
\section{Introduction}\label{Sec:Intro}
Since the 1950s, as quality control became an integral part of continuous mass production processes, change-point detection (CPD) has gained prominence across various fields. The seminal works by~\cite{girshick1952bayes},~\cite{page1954continuous},~\cite{shiryaev1961problem}, and~\cite{lorden1971procedures} contributed significantly to the development of CPD methodologies. Over time, CPD expanded its scope of application to include areas such as finance~\cite{spokoiny2009multiscale}, biology~\cite{chen2011parametric}, and engineering.
In addition to traditional CPD problems involving data structured in vector space, the detection of changes in graph-structured data has gained popularity in recent decades. Network change point detection, which focuses on identifying distributional changes in dynamic graph-structured data, has applications in monitoring and analyzing network evolution. Nowadays, as sensing and communication technologies evolves, high-dimensional data are generated seamlessly. Hence, high dimensionality, online (timely), and algorithm robustness constitute major challenges to modern change-point detection problem, and are reshaping change-point detection methodology. Motivating high dimensional change-point detection problems are, for example (a) adaptive learning in online machine learning and adaptive systems; (b) indicator of changes in financial structures and market movements \cite{grundy2020high} and detect changes in diagnostic data \cite{kuleshov2016synthetic}; (c) track abrupt changes in the dynamic evolution of network data, such as social networks, brain connectivity, and the electric grid. \par

Statistically, a change-point can be characterized as a point in sequential observations $Y_i, i=1, 2, \dots$, $Y_i \in \mathbb{R}^d$ where the probability distribution before and after the point in the sequence differs, that is $\exists \tau > 0, H_0: Y_i \sim \mathcal{F}_0$, for $ i< \tau$, otherwise $H_1$: $Y_i \sim \mathcal{F}_1$. Traditional parametric approaches face limitations with high-dimensional data, as the number of parameters to be estimated can exceed the available observations. Examples include Hotelling's $T^2$ test \cite{baringhaus2017hotelling}, and generalized likelihood ratio test \cite{james1992asymptotic}. The assumptions required for the distribution of each individual dimension are challenging to establish, as the underlying distributions are typically highly context-specific \cite{siegmund2011detecting}. In contrast, nonparametric approaches, such as the kernel-based method \cite{harchaoui2009kernel}, offer advantages for high-dimensional data. However, as the dimensionality increases, selecting an appropriate kernel function and bandwidth becomes an optimization challenge.

To address the complexity of the change-point detection (CPD) problem in high-dimensional settings, a common strategy is to project the multi-dimensional data into a lower-dimensional metric space and then apply univariate CPD methods to identify change points. For instance, \cite{wang2018high} investigate the optimal projection of CUSUM statistics to enhance the detection of changes in the mean. Similarly, \cite{boracchi2018quanttree} propose a method for detecting distributional changes in multivariate data streams using histograms. Notably, the graph-based CPD approach, initially introduced by \cite{friedman1979multivariate}, employs a two-sample test based on the minimum spanning tree (MST), which effectively captures the similarity structure between observations. Additionally, \cite{rosenbaum2005exact} introduce a test based on minimum-distance pairing (MDP), which relies on the rank of distances within pairs, thereby confining the approach to the MDP graph. More recently, \cite{chen2015graph} utilize both MST and MDP graph representations to construct a test statistic by counting the number of edges connecting data points before and after a potential change point. This approach demonstrates enhanced detection power in high-dimensional data compared to parametric methods. However, its sensitivity to variance changes is relatively limited, and it is specifically designed for offline retrospective detection within a fixed dataset.
In summary, the key features of our proposed framework are as follows.
\begin{enumerate}
    \item \textbf{Adaptability and Data-Orientation}: The method is designed to detect both mean and variance changes, making it versatile for various applications.
    \item \textbf{Generality}: It can be applied to low- and high-dimensional vector or network data, even when the underlying distribution is unknown.
    \item \textbf{Timeliness and Efficiency}: The framework maintains high detection power with small scanning windows, enabling prompt identification of change points in high-dimensional online data.
\end{enumerate}
  Inspired by the graph structure, we devise graph-spanning ratios to map the dimensional data into metrics that have distributions corresponding to the mean and variance change of the original data. The detection of variance change can be applied to many practical problems where the volatility is an important factor.

We demonstrate that the GSR method can be extended to i.i.d. data with unknown distributions through permutation or bootstrap procedures. These procedures can be applied to determine an appropriate quantile, provided that a training dataset with no (or a low probability of) change-points is available. A suitable multiplicity correction can also be derived from the resampling distribution. Under mild moment conditions, the distribution of each test statistic can be well approximated by a generalized F-distribution. Through theoretical analysis, we establish that the lower bound of the minimax separation rate for testing over the alternative hypothesis is of the order $\sqrt{nd}$, which aligns with the rates identified in \cite{enikeeva2019high} and \cite{liu2021minimax}.

We adapt the spanning-ratio CPD framework for online detection in real-world settings. Multiple scanning windows are employed to capture incoming data, enabling timely detection. Our proposed graph spanning-ratio framework facilitates online change-point detection while maintaining accuracy even with small scanning windows. The structure of this paper is as follows: Section \ref{Sec:Online} details the graph spanning-ratio algorithms for both static and online change-point detection, Section \ref{Sec:Theo} provides the theoretical foundation for the algorithm, and Section \ref{Sec:Simu} presents empirical validation of these results.
\section{Method: graph spanning ratio CPD}\label{Sec:Online}
We now introduce the test statistics for the change-point detection taking into account the similarity properties from a graph. Then we define the $\alpha$-quantiles for the test statistics and provide the algorithms for the estimation of the critical values.\par
\subsection{Notation}\label{subsec:graph_simi}
We observe data: $\{Y_i\}_{i=1, \ldots, \infty}$, where $Y_i \in \mathbb{R}^d$, and $n \in \mathbb{N}$ and denote $\mathcal{N}(\mu, \Sigma^2)$ the Gaussian distribution with mean $\mu$, and variance $\Sigma^2$; $\chi^2_{df}$ as the chi-squared distribution with $df$ degrees of freedom; $F_{df_1, df_2}$ as the Fisher distribution with $df_1$, and $df_2$ degree of freedom.
Let us consider an undirected graph $G=(V,E)$, in which vertices $V=[n]$ represent a block of $n$ consecutive observations $\{1,\ldots,n\}$ from the sequential data. Edges set $E$ indicates the connectivity of two nodes. We define edge weight $W_{ij} $ as the Euclidean distance between the nodes, that is $W_{ij} = \sqrt{\| Y_i - Y_j\|^2} $.  The graph spanning distance of a graph $G$ with nodes $\{1,\ldots,n\}$  is defined as $ \| W_{G} \|^2 = \sum_{\{ij\} \in E} \|Y_i -Y_j\|^2$, where $\| W_G \|^2 $ is the sum of squared distance between nodes in a graph $G$.
For a timestamp $t$, we define a scanning window that covers $n$ data points before and after $t$, i.e., $\{t-n, \ldots, t+n-1\}$. Given a reference $k \in [2, 2n-2]$, where $k \in \mathbb{N}$, let $G_{2n}(t)$ denote the graph constructed using the data points $\{Y_{t-n}, \ldots, Y_{t+n-1}\}$. Similarly, let $G^l_k(t)$ represent the graph constructed using the data points $\{Y_{t-n}, \ldots, Y_{t-n+k-1}\}$, and let $G^r_{2n-k}(t)$ represent the graph constructed using the data points $\{Y_{t-n+k}, \ldots, Y_{t+n-1}\}$.

At a reference point $k \in (2, 2n-2)$, the graphs $G_{2n}$, $G^l_k$, and $G^r_{2n-k}$ are constructed based on the specified data. The graph choices include the minimum spanning tree (MST), nearest neighbor graph (NNG), complete graph (CG), among others. Figure \ref{fig:CG} illustrates $G_{2n}$ constructed using the graph choices of CG, MST, and NNG, corresponding to (a) a change in mean and (b) a change in variance.
\begin{figure}[ht]
  \centering
  \textbf{CG} \hspace{60pt} \textbf{MST} \hspace{60pt} \textbf{NNG}\\
  \includegraphics[width=.16\textwidth]
   {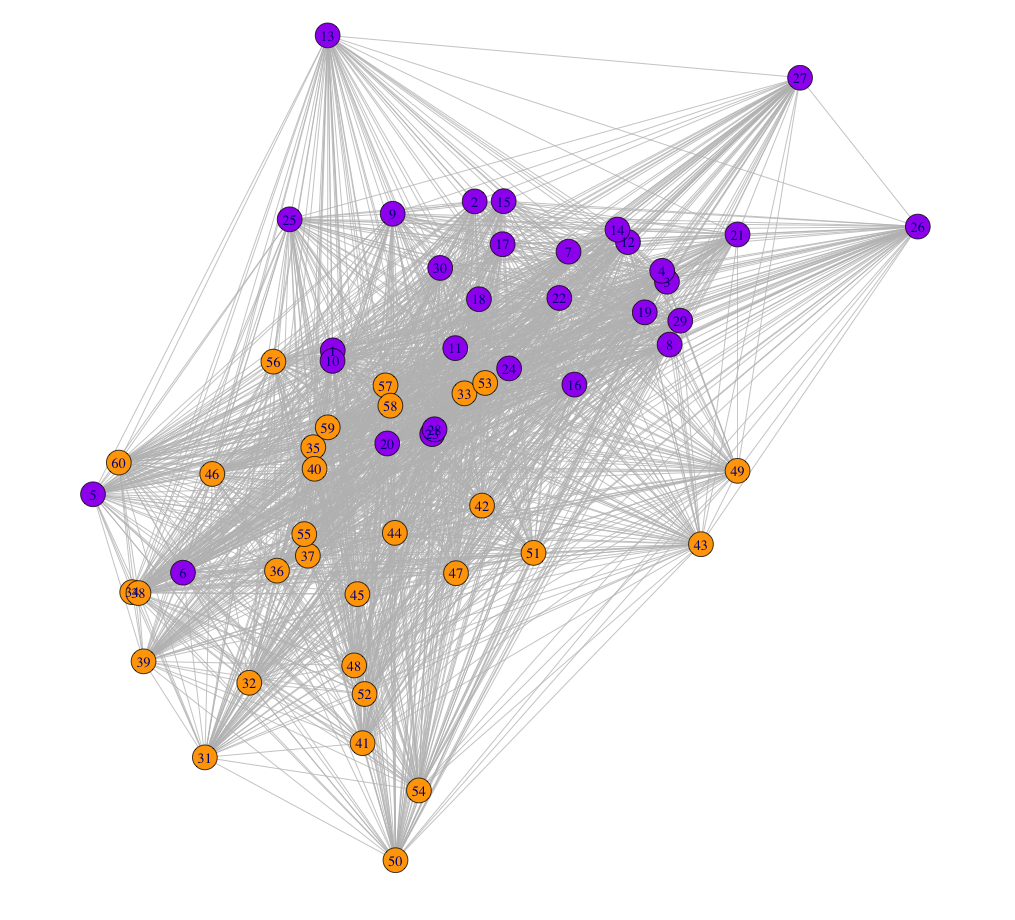}
   \includegraphics[width=.15\textwidth]
   {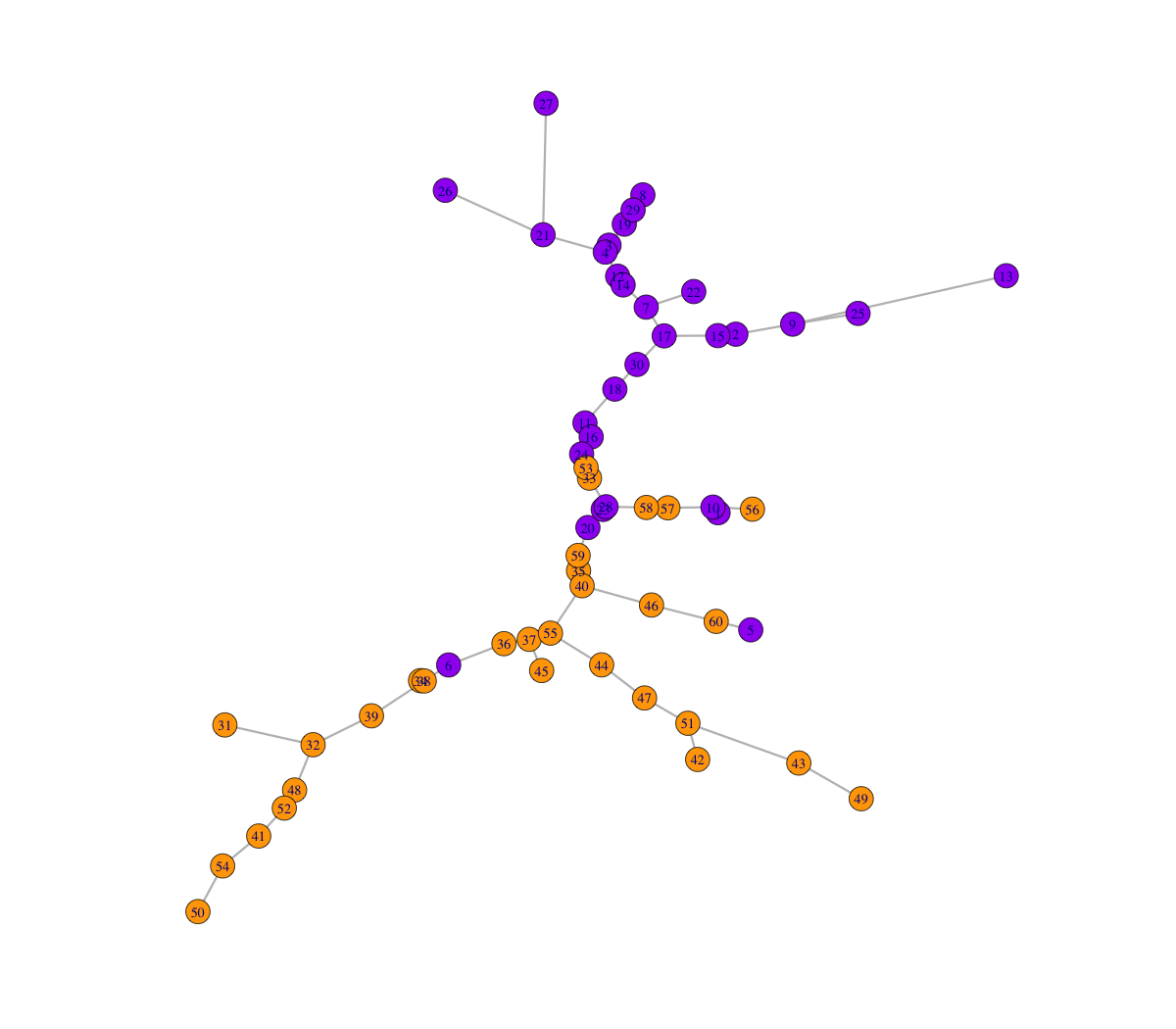}
   \includegraphics[width=.16\textwidth]
   {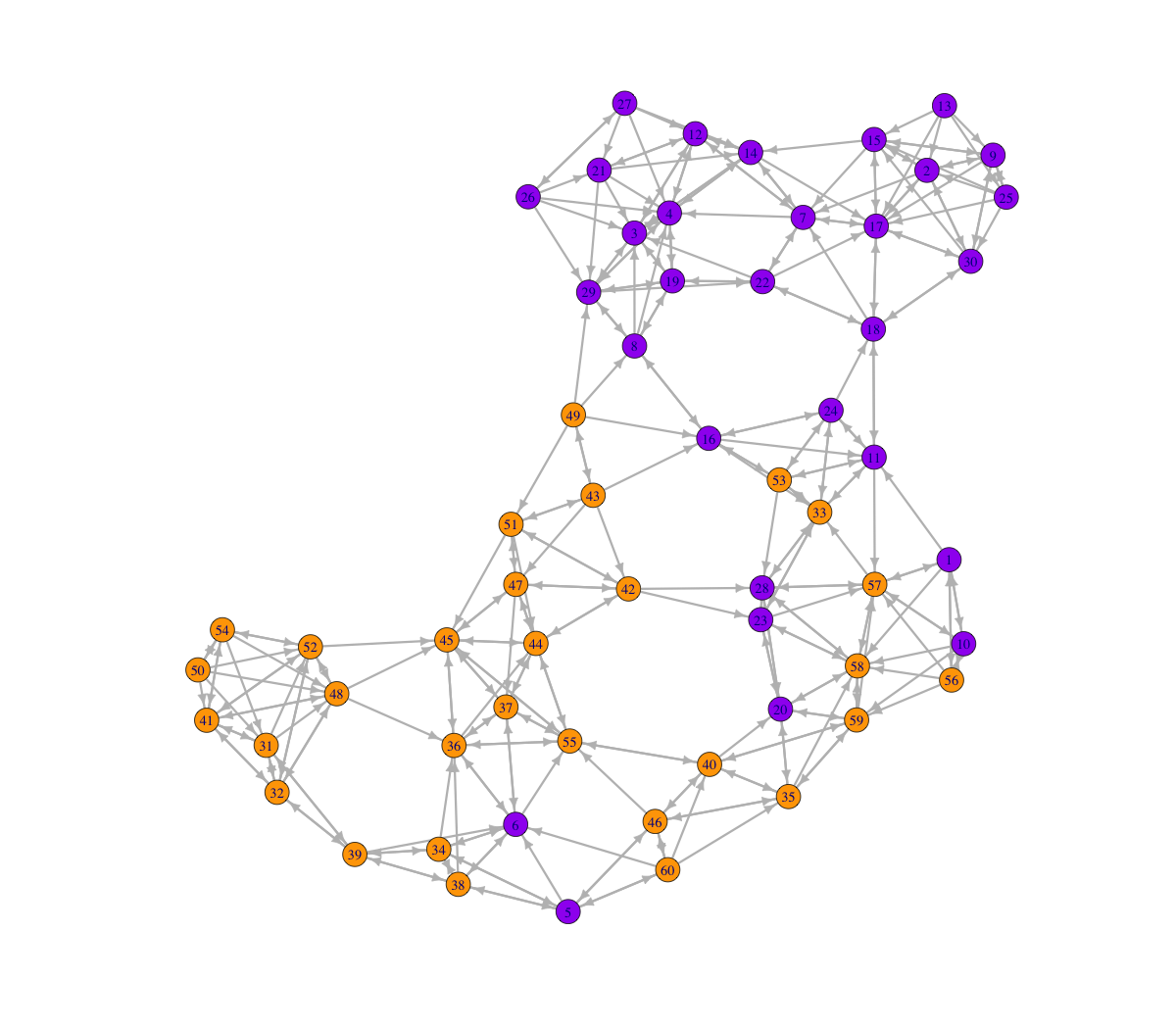}\\
  \includegraphics[width=.16\textwidth]
   {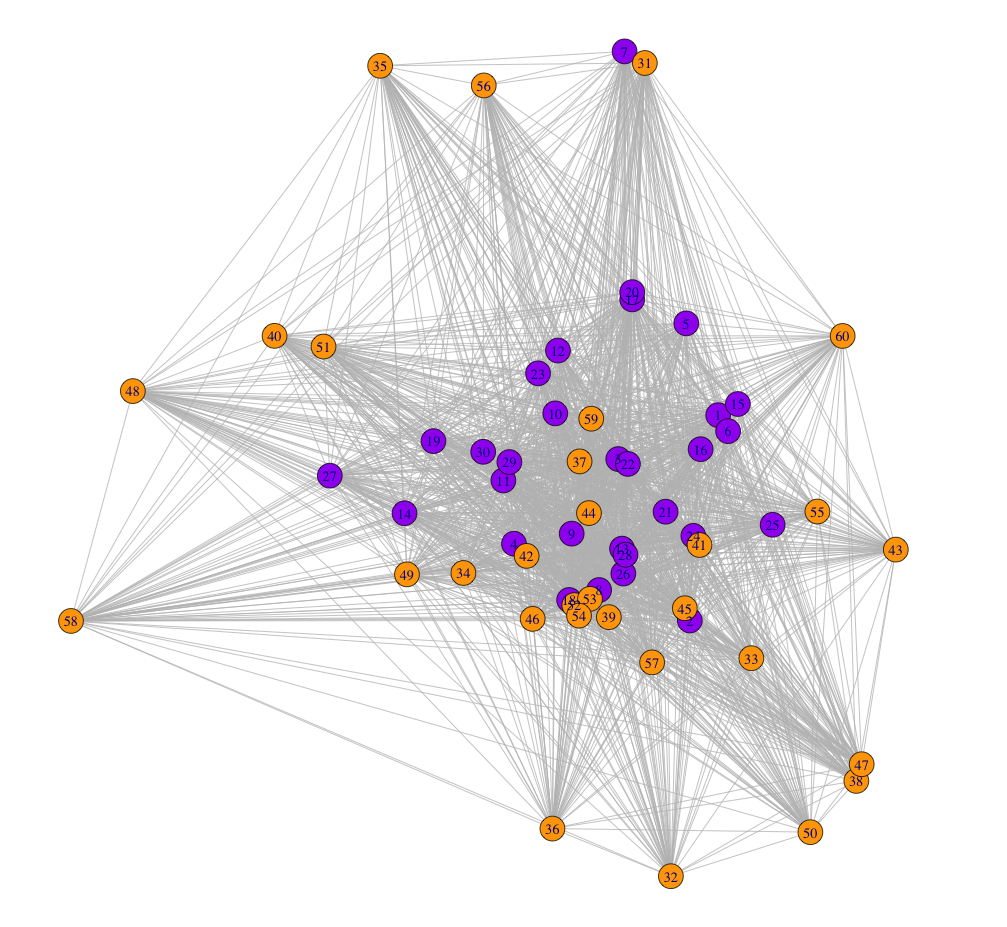}
   \includegraphics[width=.15\textwidth]
   {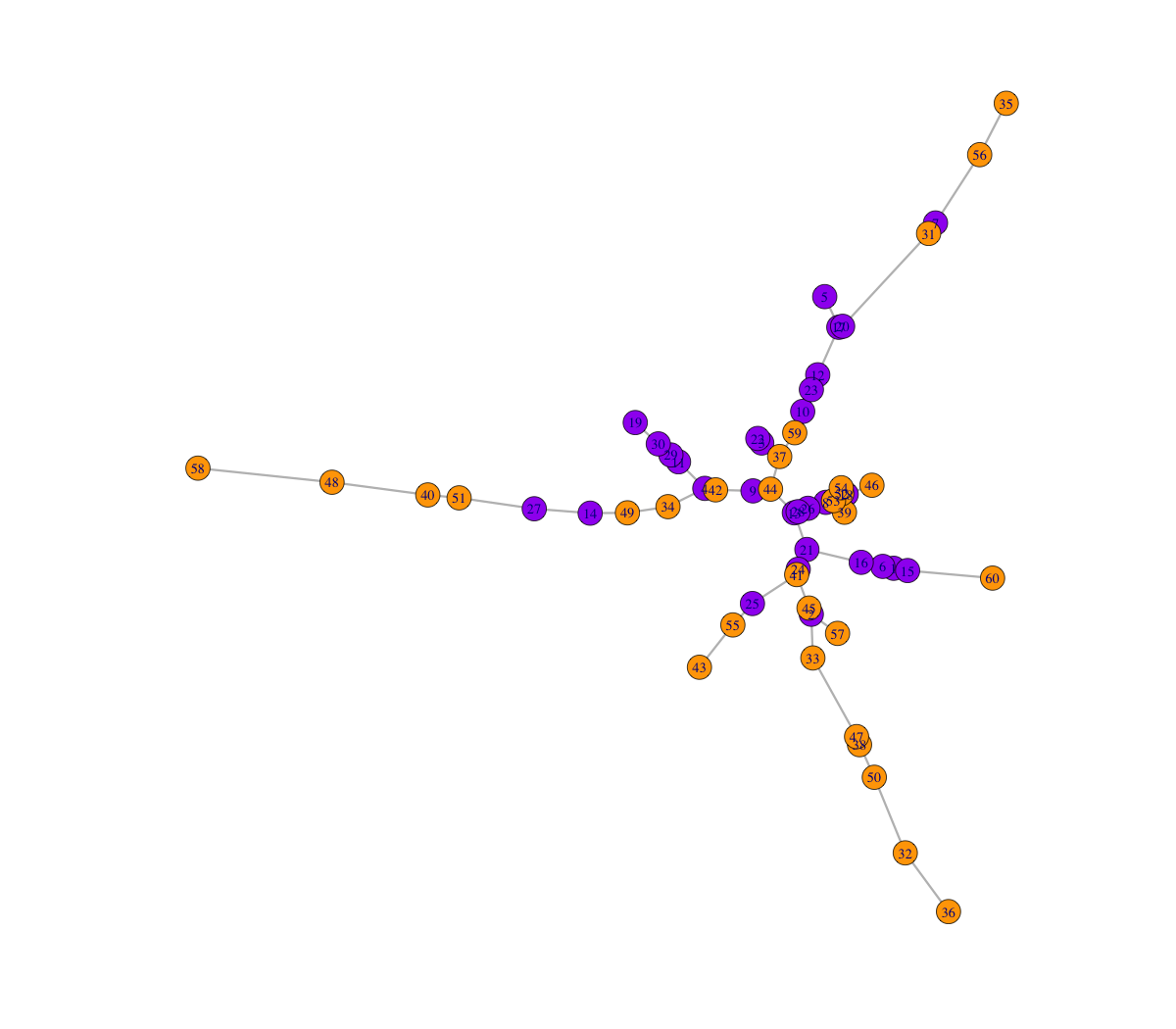}
   \includegraphics[width=.16\textwidth]
   {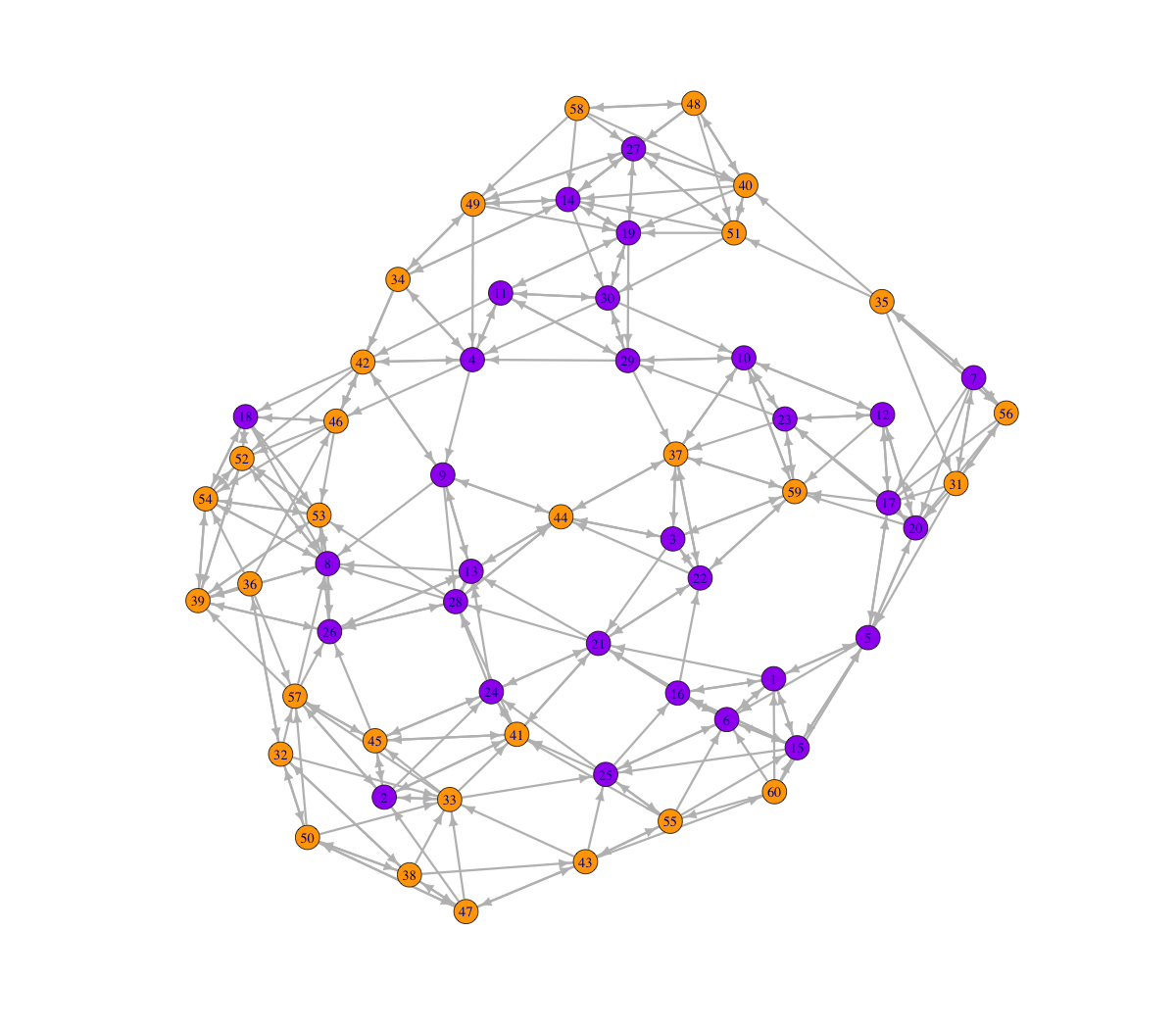}\\   
   \setlength{\belowcaptionskip}{0pt}
  \caption{Graph representation of a two-dimensional sequential data. Complete graphs, MST graphs, and NNG graphs are constructed from 60 i.i.d. normal distributed observations with first 30 observations (in orange) from standard normal, the second 30 observations (purple) with (upper row) change in mean, (lower row) change in variance.}
 \label{fig:CG}
\end{figure}  
We now define the GSR for test of the local mean change:
\begin{align*} 
R_{\mu,n,k}(t) 
 =  \frac{\| W_{G_{2n}(t)}\|^2-\frac{2n}{k} \|W_{G^l_k(t)} \|^2-\frac{n}{n-k} \|W_{G^r_{2n-k}(t)}\|^2}{\frac{2n}{k}\| W_{G^{l}_{k}(t)} \|^2+\frac{2n}{2n-k} \|W_{G^r_{2n-k}(t)}\|^2},
\end{align*}
where $n$ and $d$ are the window length and dimension of the data, respectively.

Similarly, for the detection of the local variance change, we have 
\[
	R_{\sigma+, n,k}(t)
	= \frac{(k-1) \| W_{G^r_{2n-k}(t) }\|^2}{(2n-k-1) \|W_{G^l_k(t)} \|^2}
\] 
\[
	R_{\sigma-, n,k}(t)
	= \frac{(2n-k-1) \|W_{G^l_k(t)} \|^2}{(k-1) \| W_{G^r_{2n-k}(t) }\|^2},
\]
	where $\| W_{G^{l}_{k}(t)}\|^{2}, \| W_{G^{r}_{2n-k}(t)}\|^{2}$ are the distances spanned by graphs before and after reference point $k$ within the scanning window.
Note that the graph-spanning ratio of the graphical mean $R_{\mu,n}(t)$ is devised in such a way that it increases when a change of mean occurs. Similarly for $R_{\sigma+,n}(t)$ and $R_{\sigma+,n}(t)$. Figure \ref{fig:AsymTstat} illustrates how the GSR varies as the location $k$ changes within the range ${2, \ldots, 2n-2}$. The red dotted line corresponds to the GSR of a dataset with a change point located in the middle, whereas the blue dotted line represents the GSR of a dataset without a change point.
 \begin{figure}[ht]
  \centering
  \includegraphics[width=.45\textwidth]
   {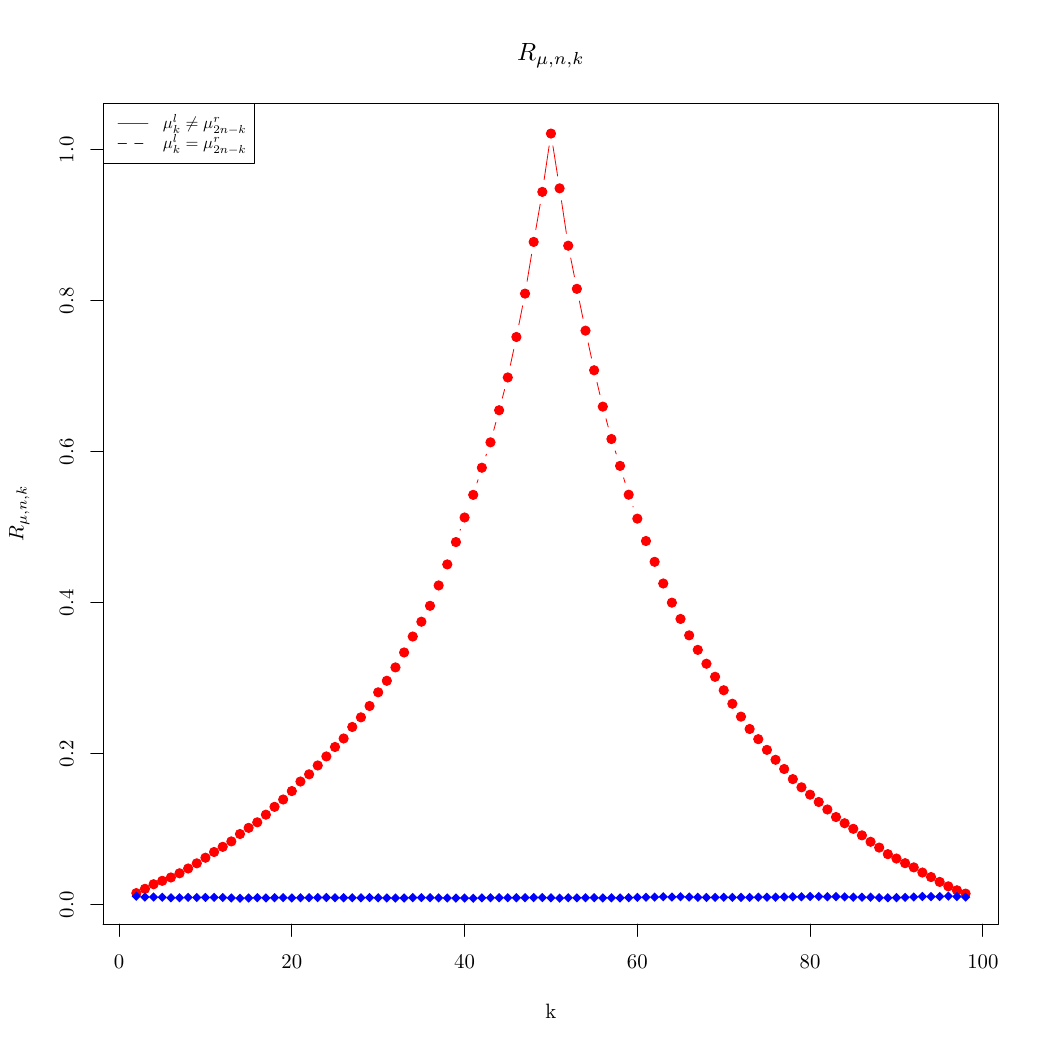}
   \includegraphics[width=.45\textwidth]
   {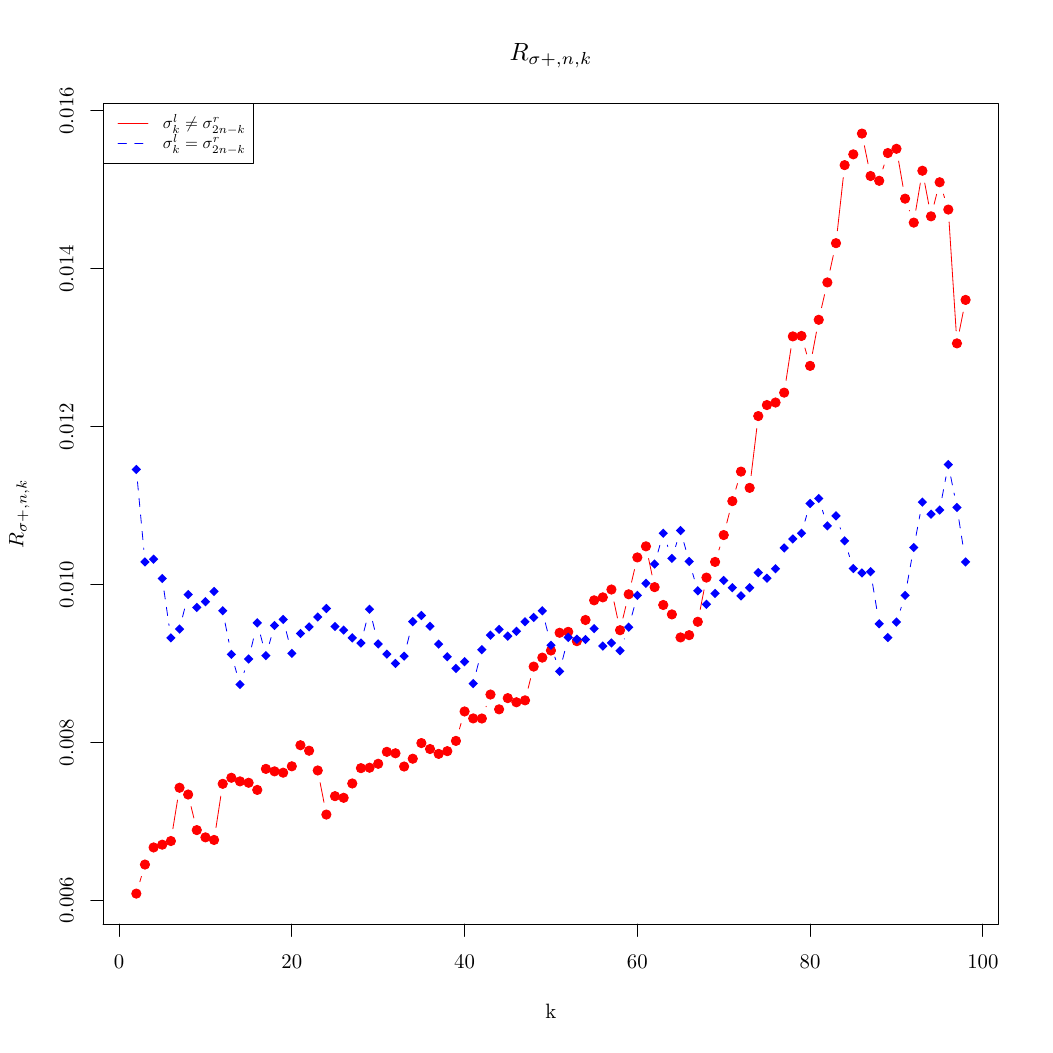}
   \setlength{\belowcaptionskip}{0pt}
  \caption{The GSRs are calculated for a data with a dimensionality of $d = 300$ and a length of $2n = 100$, as $k$ varies within the range ${2, \ldots, 2n-2}$. The red dotted line corresponds to a dataset with a change point located at the midpoint, while the blue dotted line represents a dataset without a change point.}
 \label{fig:AsymTstat}
\end{figure}
\subsection{Threshold for detection}
In the online setting, we need to take into account the small sample dependency structure, see for example \cite{kirch2008bootstrapping}. More precisely, the consecutive scanning statistics $R_{\mu, n}(t), R_{\mu, n}(t+1)$ are correlated due to the fact that we receive the data sequentially. To circumvent this problem, one can apply permutation or bootstrap procedure to determine a proper quantile provided that we are given a training data set that has no (or low probability) of change-point. Assume that we receive a sequence of $N$ i.i.d. random variables, $Y_1, \ldots, Y_N$, with $N \geq 2n$, as our training sample. We define a zone $A_n = {n+1, \ldots, N-n+1}$, where $n$ represents the size of the scanning window. For a fixed window size $n$, we perform the Bootstrap procedure (resample $Y_i$ with replacement) or the Permutation procedure (resample $Y_i$ with replacement). Base on $Y^b_1, \ldots, Y^b_N$ from the resampling, we calculate the GSR metrics ${R^b_{\mu,n,k}}(t)$, ${R^b_{\sigma+,n,k}}(t)$, and ${R^b_{\sigma-,n,k}}(t)$ for each $t \in A_n$ as follows:
         \[ R_{\mu,n,k}^{max}:= \max_{t \in A_n} R^b_{\mu,n,k}(t), \]
	    \[ R_{\sigma+,n,k}^{max}= \max_{t \in A_n} R^b_{\sigma+,n,k}(t),\]
	    \[ R_{\sigma-,n,k}^{max}= \max_{t \in A_n} R^b_{\sigma-,n,k}(t).\]

By resampling, we generate $Y^b_1, \ldots, Y^b_N$ and then repeat the procedure multiple times to estimate the quantile function of $R^b_{\mu,n,k}$. That is for $z \in  [0, 1]$,  
         \[ \rho^b_{\mu,n,k}(z) :=\inf\{ x: \mathds{P}^{b} 	\big( R^{max}_{\mu,n,k}  \geq x \big) \leq z \},\]
where $ \mathds{P}^b$ denotes the probability measure under resampling.
Resampling calibration can be used for online setting to control the false alarm rate \cite{avanesov2018change}. To lower the false alarm rate, we calibrate for all $k \in \{2,\ldots, (2n-2)\}$, 
    \[  \alpha^*_{\mu} := sup\{ z : \exists k \in \{2, \ldots, (2n-2)\},\]
    \[\mathds{P}^b\big(R^{max}_{\mu,n,k} > \rho^b_{\mu,n,k}(z) \big) <  \alpha\}, \]
The $\alpha$-quantile, $\rho^b_{\mu,n,k}(\alpha^*_\mu)$, serves as a critical value to test the change of mean. Similarly, we can calibrate the test statistics for the variance. Similarly, we estimate $\rho^b_{\sigma+,n,k}(\alpha^*_{\sigma+}) $ and $\rho^b_{\sigma-,n,k}(\alpha^*_{\sigma-}) $.
Detailed online detecting procedures are specified in Algorithm \ref{algo.onlineRaoSimple} and Algorithm \ref{algo.OnlineAsymDetectionSimple}.
\paragraph{\textbf{Asymmetric and symmetric window.}}
The reference point $k$ serves as an moving cursor to segregate the data frame to left partition and right partition which graphs $G^l_k$ and $G^r_{2n-k}$ are based on. We denote this as an asymmetric window. For case when $k = n$, when the left and right partition has exact same size of data, we denote the case as symmetric window. The symmetric window is useful when one wish to discover the seasonal structural change of the data such as month-to-month, quarter-to-quarter change. Multiple window family-wise test can be in-place to detect the change with respect to desired time frame. For simplicity, we denote the GSR test statistics for the symmetric window as $R_{\mu,n}(t)$, $R_{\sigma+,n}(t)$, and $R_{\sigma-,n}(t)$ with respect to their thresholds $\rho_{\mu,n}$,  $\rho_{\sigma+,n}$, and $\rho_{\sigma-,n}$.
\paragraph{\textbf{Offline vs. online detection}}
The above method is applicable to both offline and online data. In the online setting, the timestamp $t$ indicates the middle point of the scanning window $Y_{t-n}, \ldots, Y_{t+n-1}$. While the offline detection, the timestamp can be leave out for a closed set data  so that it is data points are labeled as $Y_1, \ldots, Y_{2n}$ for further detection.
\begin{algorithm}[ht]
   \caption{Critical value Estimation ($Y, B, n, \alpha$)}\label{algo.onlineRaoSimple}
\begin{algorithmic}
   \FOR{$b=1$ {\bfseries to} $B$}
        \STATE Generate $Y^b_1,\ldots, Y^b_{N}$ by resampling 
        \FOR{$k=t_0$ {\bfseries to} $N-t_0$}
            \FOR{$t=n+1$ {\bfseries to} $N-n+1$}
           		\STATE Calculate  test statics: $R^b_{\mu,n,k}(t)$
            \ENDFOR
            \STATE Calculate $R^b_{\mu,n,k}= \max_t R^b_{\mu,n,k}(t)$
        \ENDFOR
   \ENDFOR
\STATE  Calibrate the critical value $\rho^b_{\mu,n,k}$
\end{algorithmic}
\end{algorithm}
\begin{algorithm}[ht]
   \caption{CPD ($Y$, $\rho^b_{\mu,n,k})$}\label{algo.OnlineAsymDetectionSimple}
\begin{algorithmic}
   \STATE \textbf{Initialize} $t_0=2$, $t_L=2n-2t_0+1$, $I_{\mu}=I_{\sigma+}=I_{\sigma-}=0$    
   \REPEAT 
   \FOR{$k=t_0$ {\bfseries to} $t_L+1$ }
   		\STATE  Calculate  test statics:$R_{\mu,n,k}(t)$
        \IF{$R_{\mu,n,k}(t)> \rho^b_{\mu,n,k}$}
               \STATE $I_{\mu}=1$ \textbf{return} Mean change at $t-n+k$
        \ENDIF
    \ENDFOR
   \UNTIL{$ I_{\mu}>0$}
\end{algorithmic}
\end{algorithm}
\section{Theoretical validation}\label{Sec:Theo}
The quality of a test $\phi$ is typically measured by the type I (false positive) and type II (false negative) error probabilities. Under the null hypothesis $H_0$, the type I error probability is $\alpha = \textbf{P}_0 (\phi = 1)$. $\alpha$ is defined as the level of the test, representing the probability of rejecting $H_0$ when $H_0$ is true. This value is specified when estimating the detection threshold to control the type I error probability. If there is a change point, that is, $Y_i \sim \mathcal{F}_1$ for $i \geq t$, the type II error probability is defined as $\beta = \textbf{P}_1(\phi = 0)$, which is the probability of not rejecting $H_0$ when it is false. The quantity $1 - \beta$ is referred to as the power of the test $\phi$ at $\mathcal{F}_1$.
In this section, the $\alpha$-level and $(1- \beta)$ power are theoretically verified to ensure the quality of the proposed GSR test. Without loss of generality, the notation of test statistics under symmetric window setting, $R_\mu,n$, is used for representing the theoretical property. For a concise expression, we omit the time stamp $t$.
\subsection{ Type I error:
Level of the test for multiple windows} \label{section:P1}

For multiple window and online tests, let us fix some $\alpha \in (0,1)$, and denote the pooled test-statistic $\mathds{T}_{\mu}$ as 
\begin{equation} \label{eq:poolT}
\mathds{T}_{\mu}=\sup_{n \in \mathfrak{N}} \Big\{  T_{\mu,n} \Big\} =\sup_{n \in \mathfrak{N}} \Big\{  R_{\mu,n} - \rho_{\mu,n}(\alpha_{\mu,n}) \Big\},
\end{equation}
where  $\rho_{\mu,n}(\alpha_{\mu,n})= \argminA_{\rho} \{\mathds{P}\big( R_{\mu,n}  \geq \rho \big) \leq \alpha_{\mu,n}\}$.
 $ \{ \alpha_{\mu,n}, n \in \mathfrak{N} \}$ is a collection of numbers in $(0,1)$, such that $\forall Y_i \sim \mathcal{F}_0$, $i \in G_{2n}$, 
$ \mathds{P}_0 ( \mathds{T}_{\mu} > 0 ) \le \alpha_{\mu}.$ We reject the null hypothesis when $ \mathds{T}_{\mu} >0$. 
To verify the consistency and accuracy of the Bootstrap procedure, we apply the result of bootstrap approximation by \cite{zhilova2022new} and the delta theorem for bootstrap by \cite{wellner2013weak}. Let $Y_1, Y_2, \ldots, Y_n \overset{i.i.d}{\sim} \mathcal{F}$, $Y_i \in \mathbb{R}^d$. We denote $Y_i = (Y_{i1}, \ldots, Y_{ij})^T$, then $Y_{ij}$ is the $j$-th coordinate of $Y_i$. Assume $Y_i$ is centered, that is, $E[Y_{ij}] = 0$ and $E[Y_{ij}^2] < \infty$ for all $i=1,\ldots,n$ and $j=1,\ldots,d$. Following the Bootstrap procedure, we resample with replacement from these observation data to generate an ordered bootstrap sample: $Y^b_1, \ldots, Y^b_n$. Note that there is $1/n$ probability that $Y^b_i=Y^b_j$, $i \neq j$.
Let us assume the following condition on the random vector $Y$.
\begin{condition}
\noindent \hypertarget{SubGaussian}{[\textcolor{blue}{Sub-Gaussian condition}]}
Let $Y \in \mathbb{R}^d$ satisfy $\mathbb{E}(Y) =0$. Let $\Var(Y) \le \mathbb{I}_d$. For some $C_{Y} >0$ and $g>0$, assume that the characteristic function of $Y$ is well defined and fulfills:
\begin{equation}
    |\log \mathbb{E}e^{i\langle  u, Y \rangle} | \le \frac{C_{Y}\|u\|^2}{2}, \ \ \ \ \ \ u \in \mathbb{R}^d, \  \|u\| < g,
\end{equation}
where $i=\sqrt{-1}$.
\end{condition}
The sub-Gaussian condition states that the logarithm of the characteristic function is bounded on a ball. 

\begin{theorem}[Bootstrap validity: online]\label{theorem:BootValidOnline}

Suppose that $Y_i$ satisfies the \hyperlink{SubGaussian}{sub-Gaussian condition} and $\mathbb{E}|Y_i^{\otimes 4}| < \infty $, then
\[\left|\mathds{P}(\max_{t \in A_n} R_{\mu,n}(t) \le \rho^b_{max\mu,n}(\alpha)) - (1-\alpha) \right| \xrightarrow{}0\]
where 
\[
\rho^b_{max\mu,n}(\alpha) = \inf\{ x: \mathds{P}^{b} 	\big( R^{max}_{\mu,n}  \geq x \big) \leq \alpha \}
\]
\end{theorem}
See proof in Appendix \ref{Proof:BootValidOnline}. Similar results apply to the test statistics of variance. 
\subsection{Power of test for multiple windows}
Our aim is to determine the test's ability to detect a change when the distribution shift exceeds a threshold. That is, with $\mathds{P}$-probability greater than $1-\beta$, where $\beta \in (0,1)$, the test can detect the change when the mean shift exceeds a specified threshold $\Delta$, where $\beta$ represents the false negative rate.
We first focus on the complete graph with normally distributed observations, then extend the analysis to other graph types, and eventually to unknown distributions. When observations follow a Gaussian distribution, the GSR test statistics based on a complete graph follow a Fisher distribution. We now introduce a definition related to the spanning distance of the gap between $G_{2n}$ and $G^{l}_{n}$, $G^{r}_{n}$. This quantity is essential for determining how far the mean separates the data before and after the change point. For simplicity and without loss of generality, we omit the time stamp $t$ in this section to ensure concise expressions.
\begin{definition}
We define an gap-spanning distance:
\[\| W_{gap,n}\|^2 = \sum_{i \in G^l_n, j\in G^r_n, {i,j}\in E_{G_{2n}}} \| Y_i -Y_j\|^2,\]
which is the total spanning distance between $G_n^l$ and $G_n^r$. Let its mean value be $\| \mu_{gap,n} \| ^2$.
\end{definition}
Next, we show the theoretical separation gap (gap-spanning distance) for the static test to detect a change, and its corresponding type II error rate ($\beta$)

\begin{theorem}[Power of the test]\label{Theo:Delta_mu}
Let $ \mathds{T}_{\mu}$ be the test statistics specified in Equation (\ref{eq:poolT}), and $\beta \in (0,1)$. Then
$\mathds{P} (  \mathds{T}_{\mu} > 0) \ge 1-\beta$, if 
\begin{align*}
  &\sup_{n \in \mathfrak{N}} \{ \| \mu_{gap,n} \| ^2 - \Delta_{\mu}(n) \} \geq 0, \\
  &\Delta_{\mu}(n) = C_1 \Big( \| \mu^{l}_{G_{n}}\|^2+\| \mu^{r}_{G_{n}} \|^2 +  C_2 \sigma^2 \Big),
\end{align*}
where $\| \mu_{gap,n} \| ^2$,   $ \|\mu^{l}_{G_{n}}\|^2$,  and $\| \mu^{r}_{G_{n}} \|^2$ are the expected gap-spanning distance and the expected spanning distance of subgraphs $G^{l}_n$ and $G^{r}_n$, respectively.   
\begin{align*}
 C_1 = &5 \frac{N_n}{D_n} F^{-1}_{N_n, D_n} (\alpha_{\mu,n}),  \\
 C_2 = &
 \Bigg( D_n + 2 \sqrt{D_n \log\Big(\frac{2}{\beta}\Big)} +4 \log\Big(\frac{2}{\beta}\Big) \Bigg) - \frac{5}{4} \Bigg(N_n  - 2 \sqrt{N_n \ log \Big(\frac{2}{\beta}\Big)} - 10 \log \Big(\frac{2}{\beta}\Big) \Bigg),
\end{align*}
where $N_n = d$ and $ D_n = 2(n-1)d$.
\end{theorem}
See proof in Section \ref{Proof:Delta_mu}. This implies that there exists a window size $n \in \mathfrak{N}$ such that when the mean gap-spanning exceeds the threshold $\Delta_{\mu}(n)$, then the power of test: $ \mathds{P} ( \mathds{T}_{\mu} > 0) \ge 1-\beta $ is achieved. This gives the theoretical guarantee that the false negative rate (type II error) is smaller than $\beta$.

Figure~\ref{fig:Delta_mu} illustrates $\Delta_{\mu}$ as a function of $\beta$ with a fixed window length $n=30$, data dimension $d=100$, and significance level $5\%$. As $\beta$ decreases, the mean gap-spanning $\Delta_{\mu}$ must increase to ensure the true positive rate $1 - \beta$.
\begin{figure}[H] 
\centering
  \includegraphics[width=.45\textwidth]{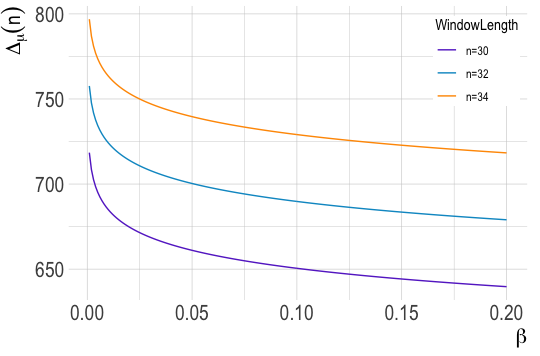}\setlength{\belowcaptionskip}{-5pt}
  \caption{Mean gap-spanning distance $\Delta_\mu$ to ensure $1-\beta$ power in detecting mean change for window size $n=30, 32, 34$, and dimension $d=10$.}
\label{fig:Delta_mu}
\end{figure}
The power of the test for the change of variance $\sigma$ can be shown in a similar way. The $(1-\beta)$-power of the test can thus be obtained. The theorems and proofs are presented in the Appendix \ref{prop:powertest_sigma}.

In the next section, we study the minimum radius for detecting the distributional change with the prescribed error rate of $\alpha$ and $\beta$. 
\subsubsection{Minimum radius of the mean separation}
We denote the quantity
\[ \beta(\mathcal{F}_1 )= \inf_{\phi_{\alpha}} \sup_{ \mathcal{F}_1 } P[\phi_{\alpha}=0], \]
where $\mathcal{F}_1$ is the alternative distribution as stated in $H_1$. $\beta(\mathcal{F}_1)$ is the infimum taking over all the tests $\phi_\alpha$ with values in $\{0,1\}$ satisfying $P_0[\phi_\alpha = 1] \leq \alpha$.

Let $\|\mu_{gap,n} \|^2$ belong to some subset of the Hilbert space,
$\textit{l}_2(n)=\big\{ \|\mu_{gap,n} \|^2 < \infty\big\}.$ For the problem of detecting the mean change, the minimal radius $\rho$ (i.e., lower bound of the minimax separation rate) is a quantity, when $\| \mu_{gap,n} \| \geq \rho$ which is the problem of testing for $i> t$, $H_0$ against the alternative,  $H_1: Y_i \sim \mathcal{F}_1$, with prescribed error probabilities, is still possible \cite{spokoiny1996adaptive}. The test $\phi_0$ is powerful if it rejects the null hypothesis for all $\{Y_i, i> t \} \sim {\mathcal{F}_1} $ outside of a small ball with probability close to 1. 

We derive the minimal radius based on the result from \cite{baraud2003adaptive}, and \cite{baraud2002non}. 

\begin{proposition}[$\left( \alpha,\beta \right)$ minimum radius]\label{Prop:lower_bound}
Let $\beta \in (0, 1-\alpha_{\mu,n})$ and fix some window size $n \in \mathfrak{N}$. Let
\[\theta(\alpha_{\mu,n}, \beta) = \sqrt{2 \log (1+4(1-\alpha_{\mu,n}-\beta)^2)}. \]
If $\| \mu_{gap,n}\|^2 \leq \theta(\alpha_{\mu,n}, \beta) \sqrt{nd} \sigma^2, $
then $\mathds{P}(T_{\mu,n}(t) \geq 0 ) \leq 1- \beta $.
\end{proposition}
See proof in Section \ref{Proof:lower_bound}. Therefore, $\theta(\alpha_{\mu,n}, \beta) \sqrt{nd} \sigma^2$ is the minimum radius $\rho$ with the prescribed error rate of $\alpha$ and $\beta$. The minimum radius $\rho$ is of order $\sqrt{nd}$, which is consistent with the results from \cite{enikeeva2019high} and \cite{liu2021minimax}. The threshold derived for $\| \mu_{gap,n}\|^2$ in Theorem \ref{Theo:Delta_mu}, $\Delta_\mu(n)$ is greater than the lower bound of the minimum radius, so Proposition \ref{Prop:lower_bound} holds. The pooled test based on $\mathbb{T}_\mu$ has power greater than $1-\beta$ over a class of window length $\mathfrak{N}$. Thus, the test of mean change is powerful. Similarly, we can confirm that the test of change of variance is powerful.
\subsection{Extending the power of test to unknown distributions}
For GSR test statistics constructed from data of various graph types and unknown distributions, the study of quadratic forms in Section \ref{Sec:nonGaussian} shows that, under a mild moment constraint, the tail distribution of the GSR statistics with unknown distributions is approximately Fisher distributed.

By Corollary \ref{corollary:F-approx}, when $Y_1, \ldots, Y_n \in \mathbb{R}{^d}$ satisfy the \hyperlink{SubGaussian}{sub-Gaussian condition}, $d \gg 1$, and $d^2 \ll n$, then by contraction, the ratio of the quadratic spanning distances approximates that of the Gaussian case. Consequently, its tail behavior also closely approximates the Gaussian case.

This shows that, with a mild norm constraint, the tail distribution of the GSR test statistic mimics that of an F-distributed random variable. Therefore, the power of the test can be guaranteed in a similar manner.
\section{Experimental analysis}\label{Sec:Simu}
\subsection{Comparison of detection power with various graph structure}  
First, we examine the detection powers between different graphs types: MST and complete graph. Our second step is to test if our proposed method has improved detection power over other methodology. Therefore, we compared with aforementioned graph-based method in Section \ref{Sec:Intro}. 
To quantify the detection power of our proposed method, we consider the scenarios that the observation follow certain parametric distribution. We generate 100 samples for detection power comparison. Each sample is consist of $n$ simulated i.i.d observations, n is even. It follows $d$ dimensional standard normal distribution $Y_i \overset{i.i.d}{\sim} \mathcal{N}(0,I_d), i =1,\ldots, n$. For $i=n+1,\ldots,2n$, with equal probability, $Y_i$ follows $\mathcal{N}(0,I_d)$ or $\mathcal{N}(\Delta, \Sigma)$ distribution.
\begin{figure}[ht]
\centering
\setlength{\tabcolsep}{0pt} 

\begin{tabular}{ccc}
  \makebox[0.15\textwidth][c]{\textbf{GSR$_{CG}$}} &
  \makebox[0.25\textwidth][c]{\textbf{GSR$_{MST}$}} &
  \makebox[0.15\textwidth][c]{\textbf{GSR$_{NNG}$}} \\[0.5ex]

  \multicolumn{3}{c}{
    \includegraphics[width=0.55\textwidth]{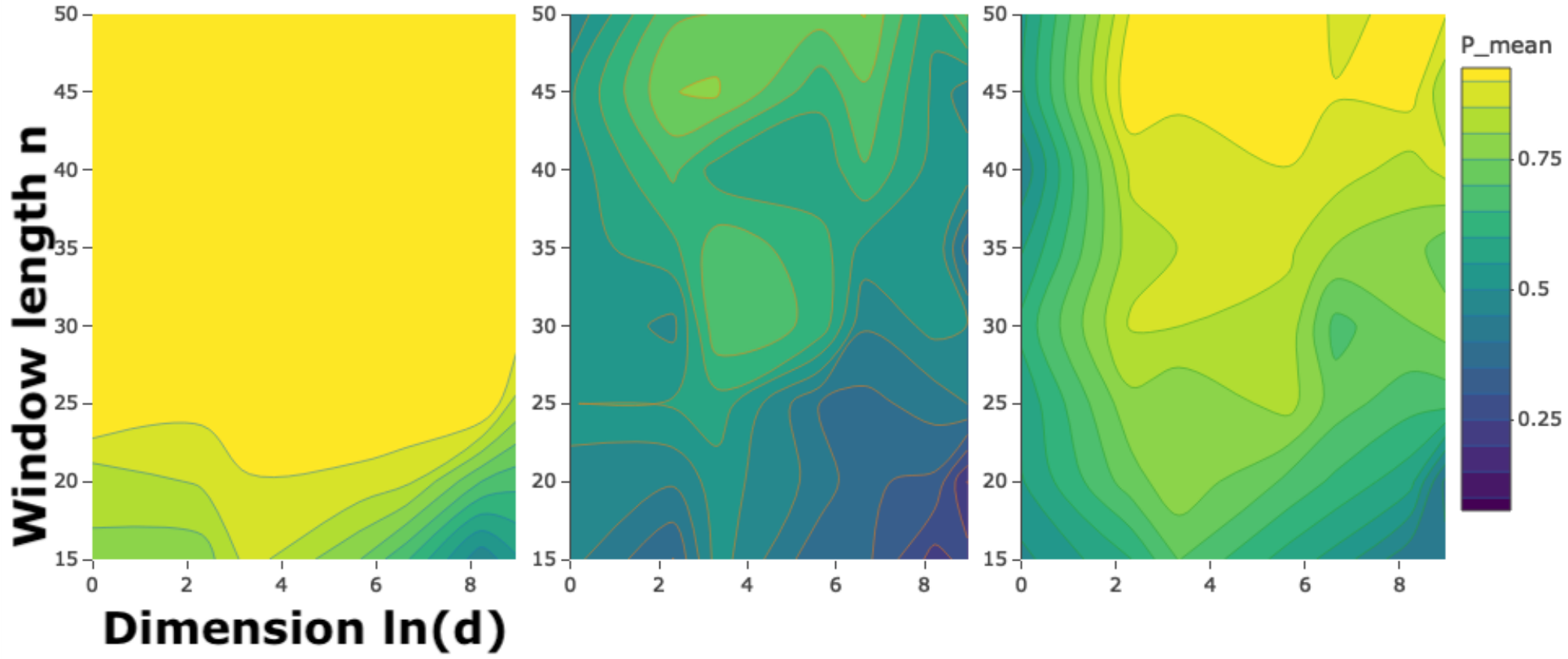}
  } \\[2ex]

  \makebox[0.15\textwidth][c]{\textbf{GSR$_{CG}$}} &
  \makebox[0.25\textwidth][c]{\textbf{GSR$_{MST}$}} &
  \makebox[0.15\textwidth][c]{\textbf{GSR$_{NNG}$}} \\[0.5ex]

  \multicolumn{3}{c}{
    \includegraphics[width=0.55\textwidth]{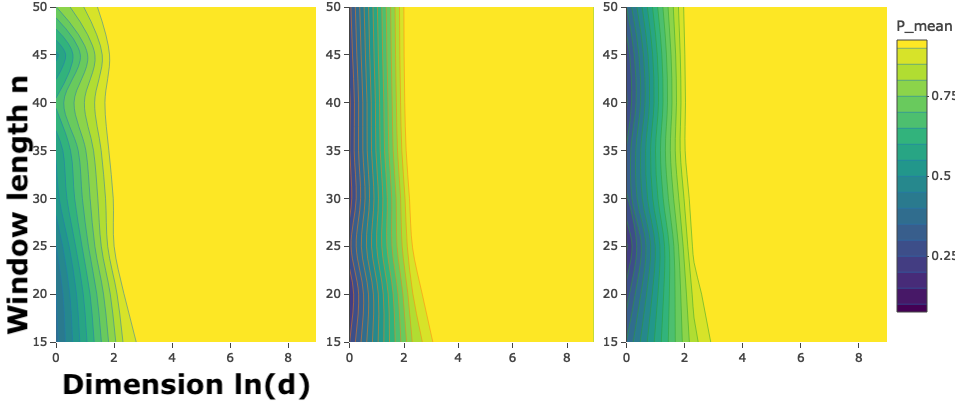}
  }
\end{tabular}
\caption{
Detection power $P_{\text{mean}}$ for a mean change $\Delta = 1/\sqrt[3]{d}$ (top row)
and a variance change $\Sigma = 2 I_d$ (bottom row), as a function of dimension and
window length. Columns correspond to GSR$_{CG}$, GSR$_{MST}$, and GSR$_{NNG}$.
}\label{fig:Contour_Static}
\end{figure}

\begin{table}[ht] 
\centering{
  \caption{Detection power, P\_mean, for mean change of $\Delta = 1/ \sqrt[3]{d}$. Comparison of GLR, $T^2$,  GEC, Kernel, $GSR_{CG}$ with respect to dimension (d) and window length (n), with significance level at 2.5$\%$.}\label{Tb:OtherMethod_SymOffMean}
\begin{tabular}{ccrrrrr}
& & $\Delta = $  & $ 1/ \sqrt[3]{d}$   &    &    &      \\

 \toprule 
  &d              & 1   &10     & 50   &  100  & 500 \\
 \toprule 
$GLR$ & n = 35   & 0.94 & 0.67  & - & - & - \\
     & n = 50    & 0.97 & 0.92  & - & - & - \\
\cmidrule(r){1-7} 
$T^2$ & n = 35   & 0.99 & 0.99  & 0.49 & - & - \\
     & n = 50    & 0.99 & 0.99  & 0.96 & - & - \\
\cmidrule(r){1-7} 
$GEC$ & n = 35   & 0.42 & 0.55  & 0.61 & 0.31 & 0.44 \\
     & n = 50    & 0.47 & 0.80  & 0.59 & 0.58 & 0.43 \\
\cmidrule(r){1-7} 
$Kernel$ & n = 35   & 0.32 & 0.78  & 0.84 & 0.73 & 0.71 \\
     & n = 50    & 0.45 & 0.91  & 0.99 & 0.98 & 0.98 \\
\cmidrule(r){1-7} 
$GSR_{CG}$ & n = 35   & 0.99 & 0.98  & 0.99 & 0.98   & 0.98 \\
     & n = 50    & 0.99 & 0.99  & 0.99 & 0.98 & 0.98 \\
\bottomrule
\end{tabular}
}
\end{table}

\begin{table}[ht] 
\centering{
  \caption{Detection power, P\_mean, for variance change of $\Sigma = 2 I_d$. Comparison of GLR, $T^2$, GEC, Kernel, $GSR_{CG}$ with respect to dimension (d) and window length (n), with significance level at 2.5$\%$.}}\label{Tb:OtherMethod_SymOffVar}
\begin{tabular}{ccrrrrr}
& & $\Sigma=$    &  $2 I_d$   &    &    &    \\

 \toprule 
  &d              & 1   &10     & 50   &  100  & 500 \\
 \toprule 
$GLR$ & n = 35   & 0.49   &0.61    & -   &  -  & - \\
     & n = 50    &  0.67   &0.88     & -   &  -  & -\\
\cmidrule(r){1-7}
$T^2$ & n = 35   & 0.09   &0.13     & 0.14   &  -  & - \\
     & n = 50    & 0.00   &0.10     & 0.17   &  -  & -\\
\cmidrule(r){1-7} 
$GEC$ & n = 35   & 0.19  &0.44     & 0.11   &  0.00  & 0.00\\
     & n = 50    & 0.25   &0.48     & 0.17   &  0.00  & 0.00\\
\cmidrule(r){1-7}
$Kernel$ & n = 35   & 0.00  &0.10     & 0.00   &  0.00  & 0.19\\
     & n = 50    & 0.00   &0.11     & 0.25   &  0.57  & 0.99\\
\cmidrule(r){1-7}
$GSR_{CG}$ & n = 35   & 0.65  & 0.98   &  0.98  & 0.99      & 0.98\\
     & n = 50    & 0.68   &0.97     & 0.97   &  0.99  & 0.98\\
\bottomrule
\end{tabular}

\end{table}
We denote our proposed complete graph-spanning-ratio methodology as GSR, in comparison to general likelihood ratio (GLR), Hotelling's $T^2$ ($T^2$), the in-between-group edge counting (GEC) methods mentioned in Section \ref{Sec:Intro}, and the kernel method. For the kernel method, we utilized the function \texttt{kcpa} from the \texttt{ecp} R package. Since the kernel function does not directly have Type I error rate ($\alpha$) as input, we fine-tuned its parameters using Monte Carlo simulations to ensure an equivalent Type I error rate $\alpha$ for a fair comparison. The \texttt{kcpa} function is based on the kernel CPD algorithm developed by \cite{arlot2019kernel}.

Building on the results from the previous section, we employ a complete graph within our GSR framework to compare the proposed method with other approaches.

We consider both accuracy and sensitivity as a general way of comparing detection power \cite{aminikhanghahi2017survey}. Detection accuracy is defined as how often the detection algorithm make the right decision, that is, to identify change-point when there in reality a true change-point, and identify no change-point when there is true non-change-point. We denote $TP$ as true positive, $FN$ as false negative, and so on. Then we define accuracy$=\frac{TP+TN}{TP+TN+FP+FN}$. We denote FPR $= \frac{FP}{FP+TN}$ as the false positive rate which is rate of giving a false alarm when no change-point present.

For detection sensitivity, we concern about the success rate of identify a change-point when there indeed true change-point exist. Therefore, sensitivity $= \frac{TP}{TP+FN}$. To consider the detection power with both the accuracy and the sensitivity of the detection methods, here we define a power metric as the geometric mean of the accuracy and sensitivity, P\_mean $= \sqrt{accuracy \times sensitivity}$

Each sample contains either with or without change-point in the middle point of sample. Table \ref{Tb:OtherMethod_SymOffMean} shows the comparison of detection power. We can see that GLR and $T^2$ shows good detection power for change of mean but limited to low dimensions. The detection power are higher for GSR compared to GEC method, across all dimension and window length. The kernel method achieves high detection power for larger window lengths, but demonstrates relatively lower power for smaller window lengths. The kernel method requires careful parameter tuning to achieve the desired FPR, which can be computationally intensive as it involves solving an optimization problem. GSR method shows generally good detection power for the detection of mean and variance change. In particular, with small window length under high-dimensional scenarios. This make our proposed algorithm more ideal for further online detection, where a timely detection of change-point is important. Our method can be generalized to distributions other than Gaussian and the result is shown in Appendix \ref{Append: numerical}.    
\subsection{Change of graph structure}
For non-Euclidean graphs such as social networks and power grids, one plausible application is detecting changes in their graph data structures, such as connectivity or power usage. Denote $Y = [y_1, \ldots, y_d]^T$ as the connectivity for graph $V[d]$, where $d$ is the number of nodes. The problem is equivalent to detecting changes in the distribution of $Y$. A simulated CPD experiment demonstrates that the algorithm successfully detects changes in connectivity. As mentioned in the Introduction chapter, the social network can be modeled using an Erd\H{o}s-R'enyi random graph (ER), $ER_n(\lambda /n)$. We generate a graph of $30$ nodes with connectivity probability $p = \lambda/n = 1/2$ and change it to a connectivity probability of $p = 1/3$, as shown in Figure \ref{fig:GraphStructure3}.

Table \ref{Tb:connectivity} shows the detection power with respect to changes in connectivity for a graph with 30 nodes and an observation window of 30. As the connectivity changes from a probability of $1/2$, the detection remains high across all graph types when the change $\Delta p$ is relatively large. However, when the change is smaller than the inverse of the number of edge nodes, the detection power deteriorates significantly for NNG and MST. This reduction in performance may be attributed to fewer edges (less information) being available and the presence of random errors.
 \begin{figure}[ht] 
  \centering
  \hspace{30pt} \textbf{$p =1/2 $} \hspace{60pt} \textbf{$\Delta p = 1/6 $}\\
   \includegraphics[width=.45\textwidth]
   {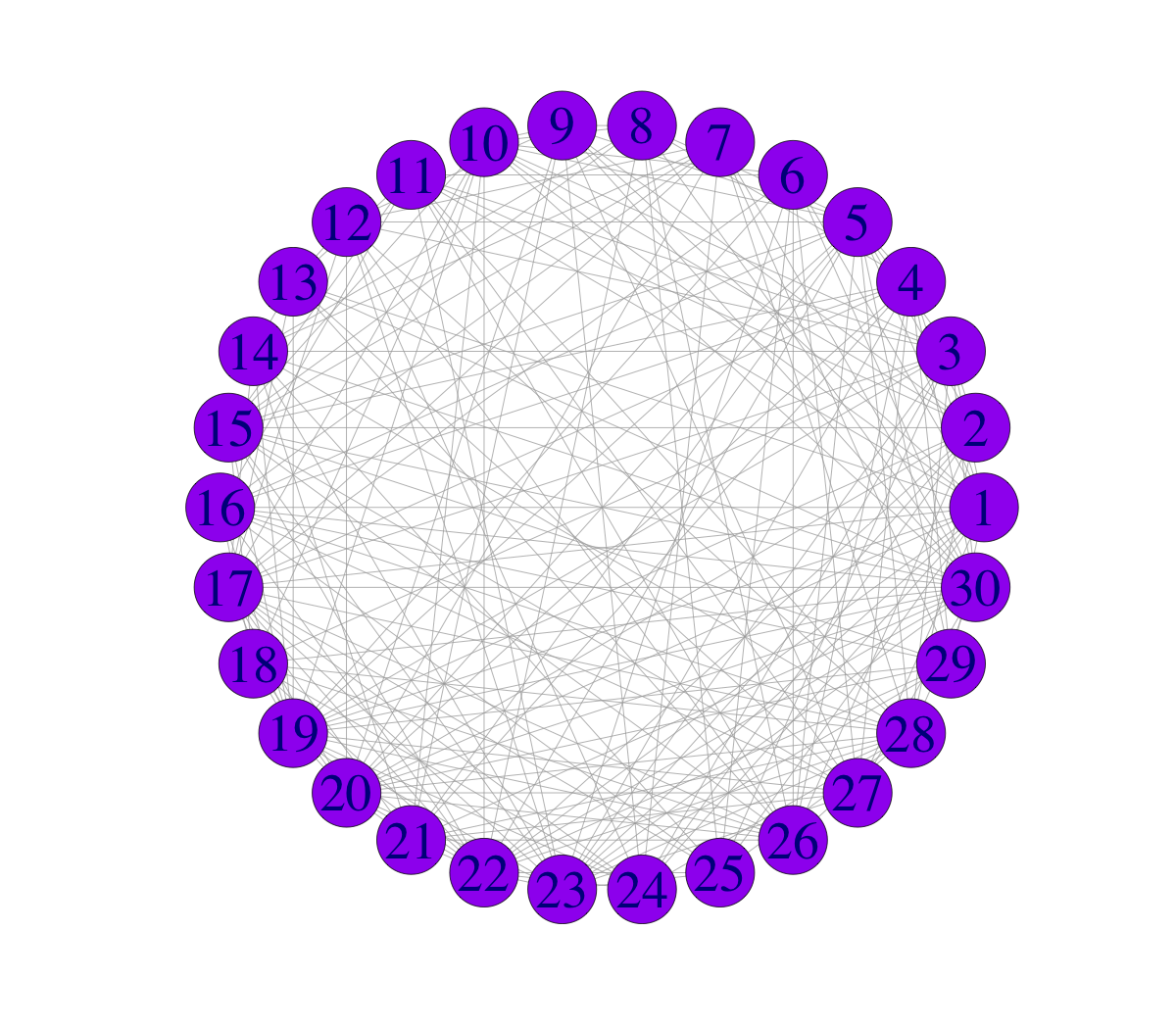}
   \includegraphics[width=.45\textwidth]
   {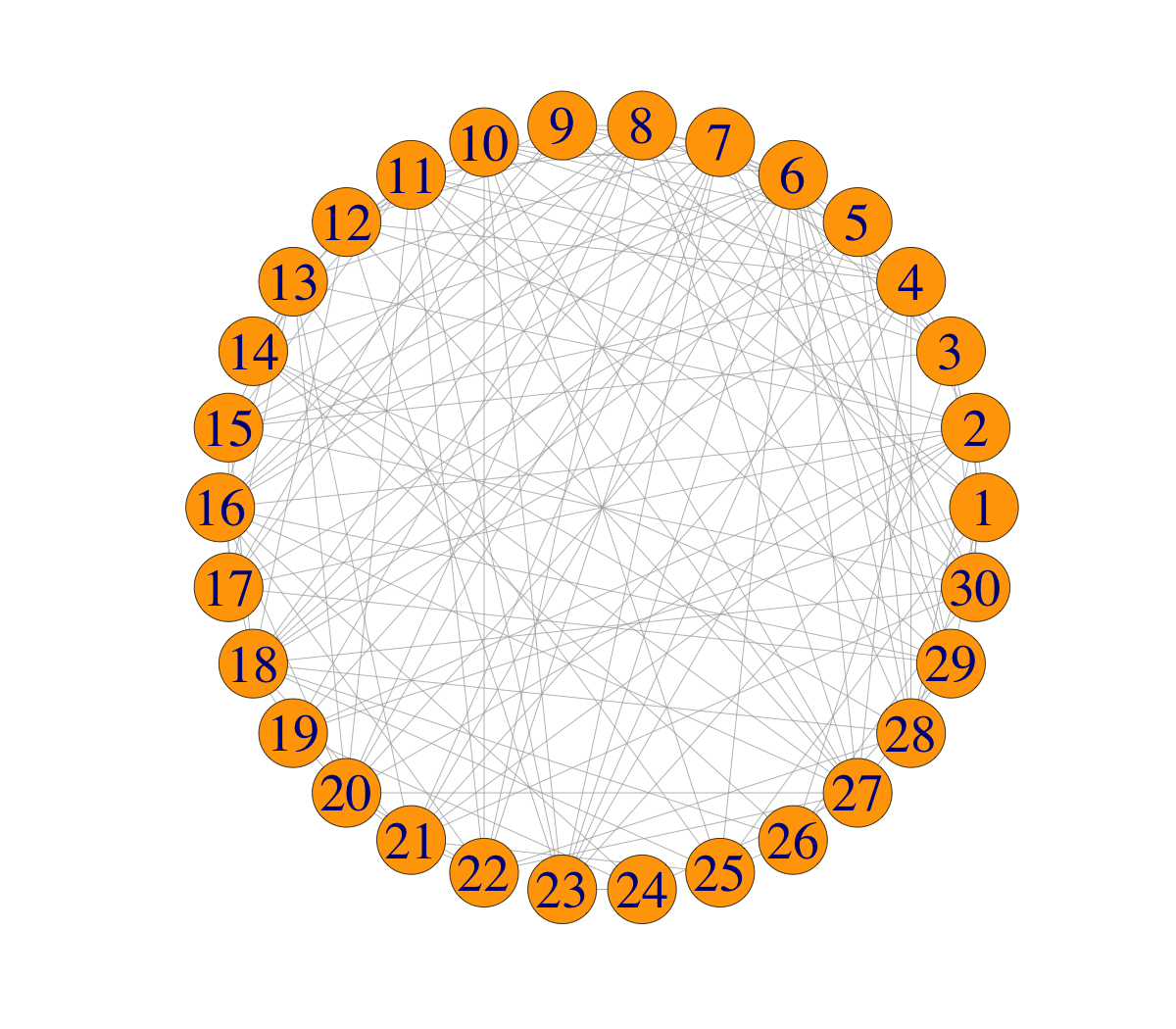}
   \setlength{\belowcaptionskip}{0pt}
  \caption{Change in graph connectivity from $p=1/2$ to $p=1/3$ ($\Delta p =1/6$). Purple nodes represent the graph data before change point, while oranges nodes represent the graph data after change point.}
 \label{fig:GraphStructure3}
\end{figure}\\

\begin{table}[ht] 
\centering
  \caption{Detection power for connectivity changes in an Erd\H{o}s-R'enyi graph with $d=30$ nodes and an observation window length of $n=30$, based on 1000 detection tests. The threshold is estimated using the permutation procedure with a significance level of $2.5\%$.}\label{Tb:connectivity}
\begin{tabular}{ccrrr}
 \toprule 
  &d =30    &   CG   &  MST   & NNG   \\
 \toprule 
$\Delta p =1/6 $ &P\_mean  & 0.995 & 0.995 & 0.997\\
& FPR    & 0.02 & 0.02 & 0.01\\
\bottomrule
$\Delta p= 1/12 $ &P\_mean  & 0.991& 0.992 &0.998 \\
& FPR    &0.03& 0.03 & 0.01\\
\bottomrule
$\Delta p= 1/24 $ &P\_mean  &0.994 & 0.900& 0.904\\
& FPR    &0.02& 0.02 & 0.01\\
\bottomrule
$\Delta p= 1/48 $ &P\_mean  &0.74 &0.33 &0.26 \\
& FPR    &0.02&0.02  &0.01 \\
\bottomrule
\end{tabular}

\end{table}
\paragraph{\textbf{Detection of changes in graph types}.}We illustrate the detection of changes in graph types through examples of structural changes, including transitions from MST to CG, CG to NNG, and MST to NNG, as shown in Figure \ref{fig:GraphStructure}. In all of these examples, the test statistics exceed their respective thresholds, indicating the presence of a change point.
\begin{figure}[ht] 
  
  \hspace{5pt} \textbf{MST-CG} \hspace{100pt} \textbf{MST-NNG} \hspace{80pt} \textbf{NNG-CG}\\
  \includegraphics[width=.30\textwidth]
   {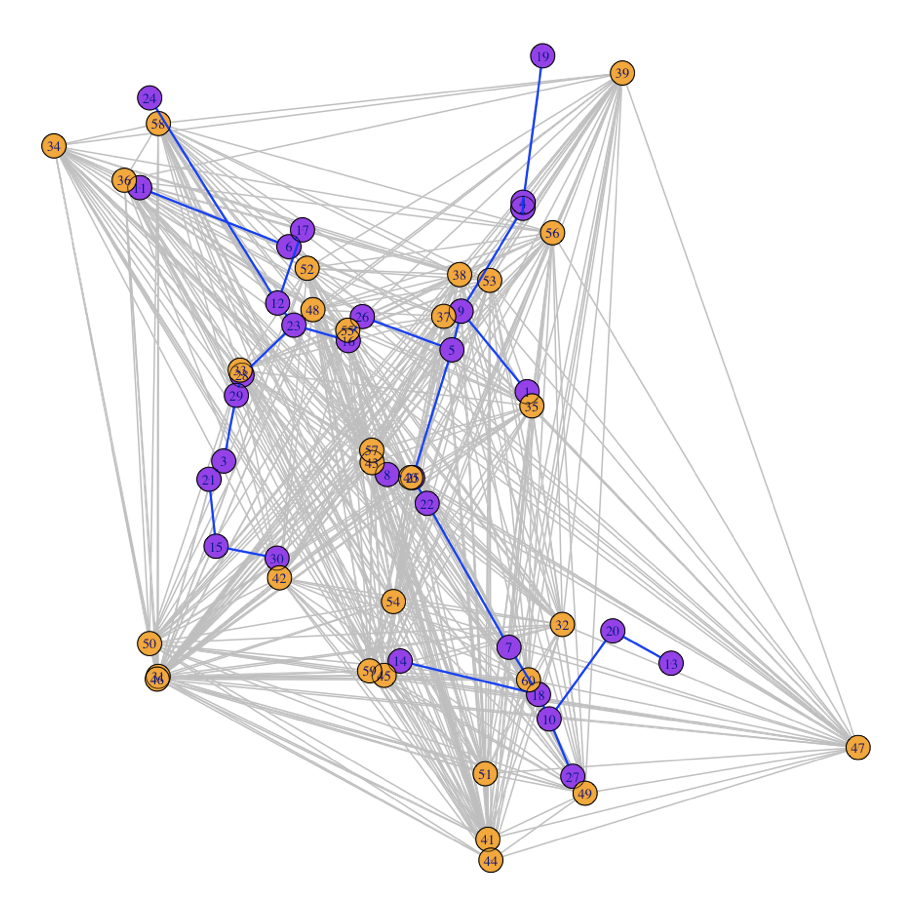}
   \includegraphics[width=.30\textwidth]
   {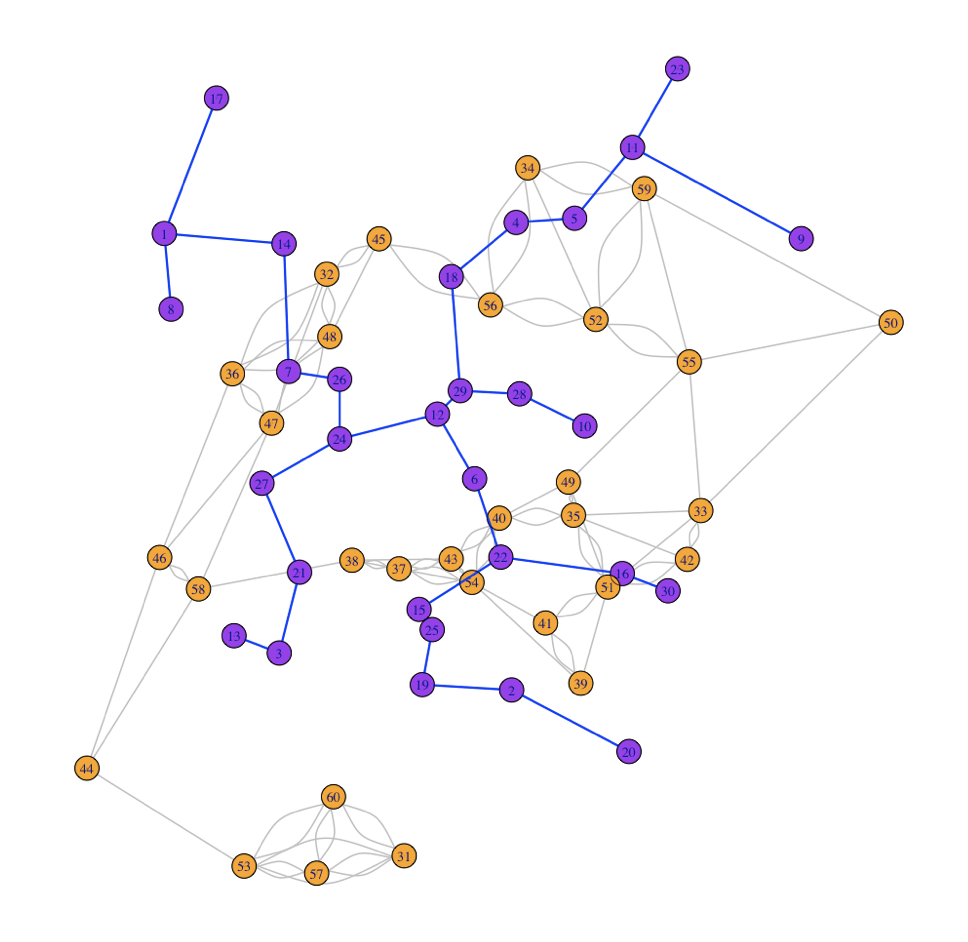}
   \includegraphics[width=.30\textwidth]
   {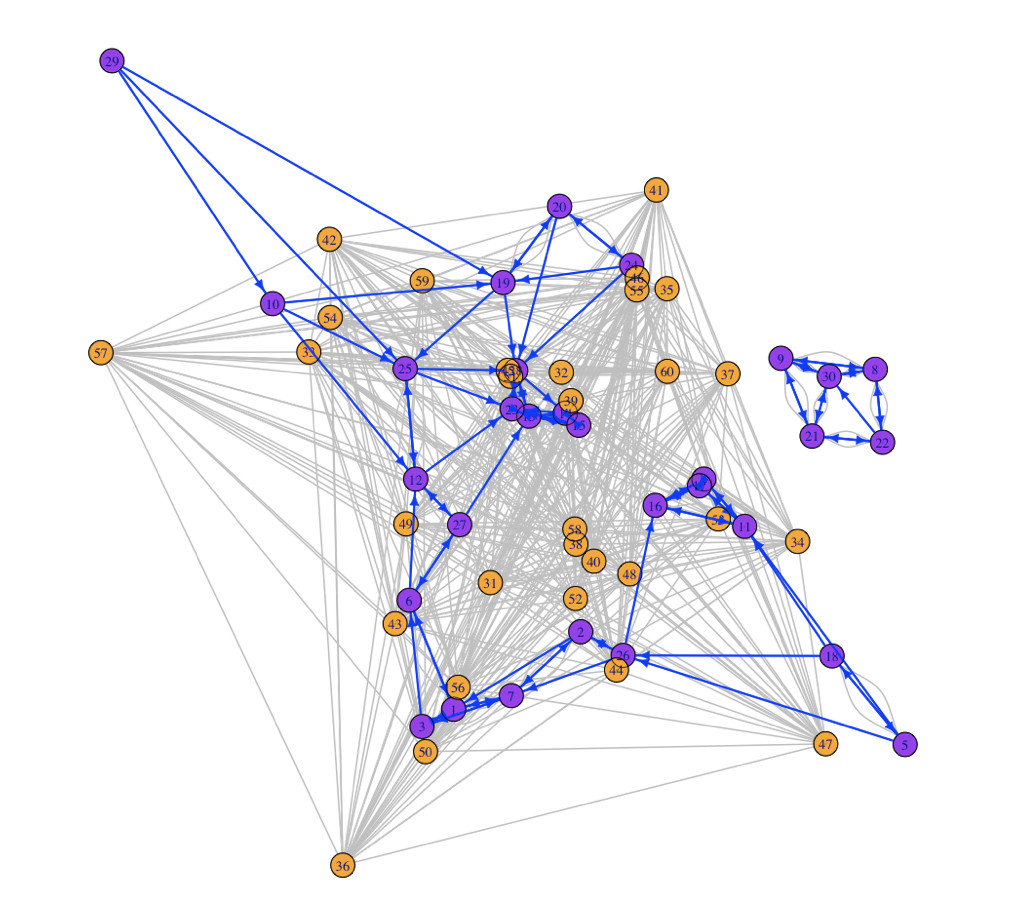}
   \setlength{\belowcaptionskip}{0pt}
  \caption{Detecting changes in graph types: Purple nodes and blue edges represent the graph data before the change point, while orange nodes and gray edges represent the graph data after the change point. The graph structure changes from MST to CG (left), MST to NNG (middle), and NNG to CG (right).}
 \label{fig:GraphStructure}
\end{figure}

\subsection{Change points in S\&P 500 stocks}
The proposed change point detection framework was applied to real-world data from S\&P 500 stocks. Using the online algorithm, we analyzed the closing daily stock prices of companies listed in the S\&P 500 from January 2014 to January 2016. The data were log-returns of stock prices, with approximately 253 trading days per year. In financial markets, changes are typically reported on a month-to-month or quarter-to-quarter basis. To capture quarterly variations, we set the window length at $n = 32$. By adjusting the significance level $\alpha$, the false alarm rate can be controlled.
In Figure \ref{fig:SP500}, mean changes were detected in August 2015, corresponding to the three-day market drop of 7.7\% in the DJIA. This event was reportedly linked to the Greek debt crisis in June 2015 and the Chinese stock market turbulence in July. In early 2016, several mean changes were detected, coinciding with a sharp rise in bond yields during that period.
The variance change analysis in Figure \ref{fig:SP500} revealed that market volatility fluctuated more frequently compared to mean changes. In practice, changes in variance serve as critical risk indicators for market instability.
\begin{figure}[H]
  \centering
  \includegraphics[width=.9\textwidth]{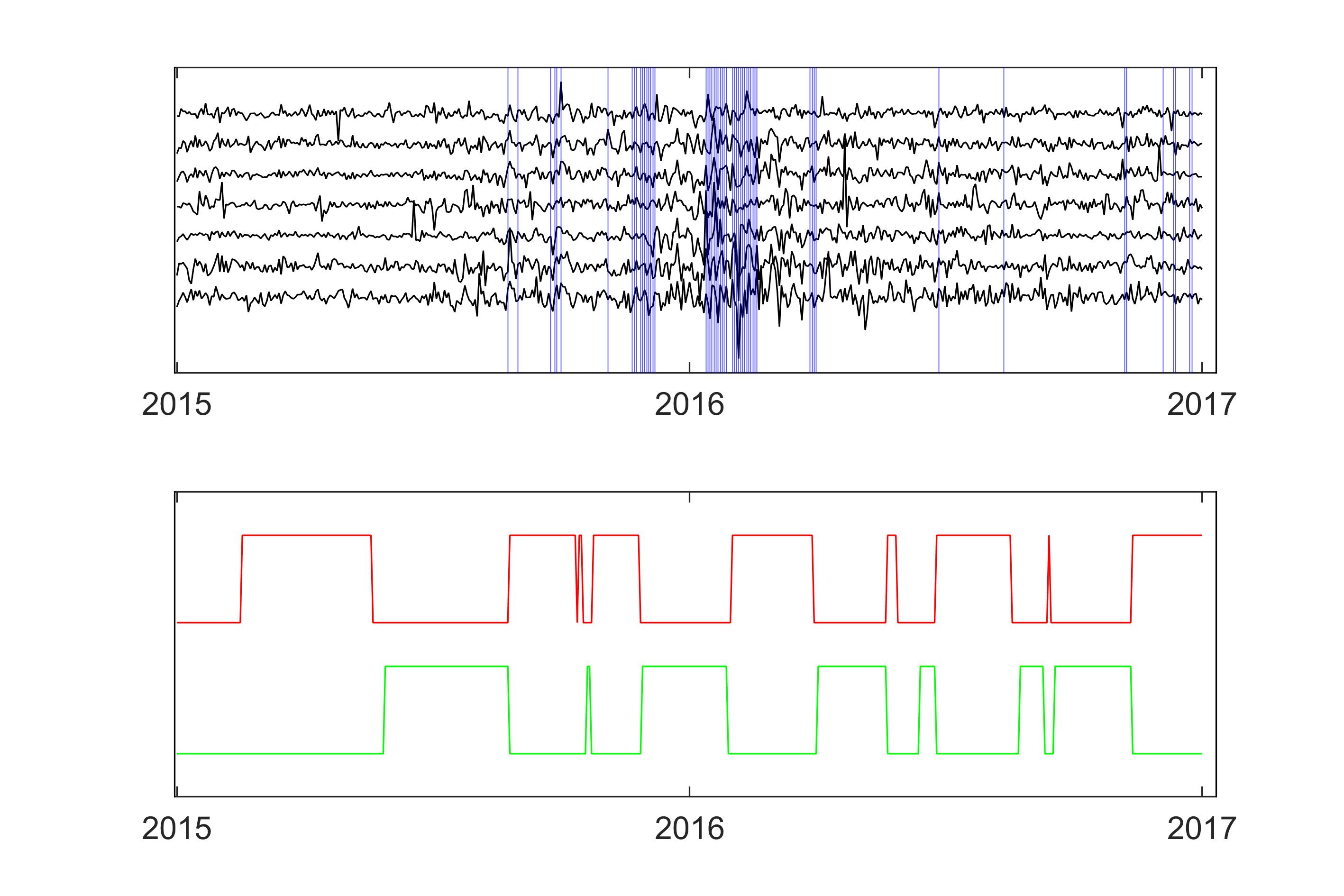}\setlength{\belowcaptionskip}{-5pt}
  \caption{Online detection was performed on the daily closing prices of S\&P 500 stocks from 2015 to 2017, using a window length of $n = 32$. For demonstration purposes, the plot showcases 7 out of the 500 stocks. In the upper figure, the blue lines indicate the detected mean changes. In the lower figure, the red and green lines represent increases and decreases in variance, respectively.}
\label{fig:SP500}
\end{figure}
\section*{Conclusion}
We proposed GSRCPD for low to high-dimensional data. Comparing to a recent literature, numerical studies show that the method has desirable power with small and multiple scanning windows, which enables online timely detection of change-point. The framework's versatility enables its application to both vector-form and graph-structured data, supporting the detection of structural and topological changes in graphs. Moreover, the method is fully adaptive and data-driven, making it an effective tool for identifying inhomogeneity in both online and offline data settings.  We conclude with application to real S\&P500 data from financial industry to make statistical inferences about mean and variance changes in 500 stocks.

\begin{appendix}

\section{Experiment}
\subsection{Application to non-Gaussian data}\label{Append: numerical} 
In this section, we demonstrate that our proposed approach (GSR) can be extended to handle data that are not Gaussian-distributed. Tables \ref{tb:Contour_StaticMean} and \ref{tb:Contour_StaticVar} compare the detection power of various methods when the data distribution does not belong to the exponential family. The comparison includes Hotelling's $T^2$ ($T^2$), the generalized likelihood ratio (GLR), and the in-between-graph edge-counting method (GEC).

We consider scenarios where the observations follow a uniform distribution. To compare detection power, we generate 100 samples, each repeated 100 times. Each sample consists of $2n$ simulated i.i.d. observations. For a change point involving a mean shift, we set up the following scenario: with equal probability, the observations are $d$-dimensional uniformly distributed as $Y_i \overset{i.i.d}{\sim} \mathbf{U_d}(0,1), i = 1, \dots, 2n$, or they are generated from the following distribution:
\begin{equation} 
	 Y_i \overset{i.i.d}{\sim} \left\{  \begin{array}{ll}  \mathbf{U_d}(0,1), \quad i =1,\dots, n; \\ 

          \mathbf{U_d}(1/ \sqrt[3]{d},1+1/ \sqrt[3]{d})^d, \quad i =n+1,\dots, 2n. \end{array} \right.
\end{equation}\label{eq:simu_offlineU}

For a change of variance, we have 
\begin{equation} 
	 Y_i \overset{i.i.d}{\sim} \left\{  \begin{array}{ll}  \mathbf{U_d}(0,1), \quad i =1,\dots, n; \\ 

          2 \mathbf{U_d}(-1/4,3/4), \quad i =n+1,\dots, 2n. \end{array} \right.
\end{equation}\label{eq:simu_offline2}

Compared to other methods, our approach demonstrates superior detection power for both mean and variance changes, particularly in high-dimensional settings where the detection power of most other methods significantly diminishes (approaching zero).

\begin{table}[ht] 
\centering{
  \caption{Detection power $P\_mean^2$ for mean change with a significance level around 2.5$\%$.}\label{tb:Contour_StaticMean}
\begin{tabular}{ccrrrrr}

 \toprule 
  &d              & 1   &10     & 50   &  100  & 500 \\
 \toprule 
$GLR$ & n = 35   & 0.98$\pm$0.01 & 0.89$\pm$0.05  & - & - & -  \\
     & n = 50    & 0.99$\pm$0.01 & 0.98$\pm$0.01  & - & - & - \\
\cmidrule(r){1-7} 
$T^2$ & n = 35   & 0.99$\pm$0.01 & 0.99$\pm$0.01  & 0.08$\pm$0.03 & - & -  \\
     & n = 50    & 0.99$\pm$0.01 & 0.99$\pm$0.01  & 0.36$\pm$0.08 & - & - \\
\cmidrule(r){1-7}
$GEC$ & n = 35   & 0.98$\pm$0.01 & 0.94$\pm$0.04  & 0.07$\pm$0.03 & 0.02$\pm$0.02 & 0.01$\pm$0.01 \\
     & n = 50    & 0.97$\pm$0.02 & 0.99$\pm$0.01  & 0.08$\pm$0.03 & 0.03$\pm$0.02 & 0.01$\pm$0.01\\
\cmidrule(r){1-7} 
$Kernel$ & n = 35   & 0.99$\pm$0.01 & 0.98$\pm$0.01  & 0.75$\pm$0.07 & 0.54$\pm$0.08 & 0.09$\pm$0.03 \\
     & n = 50    & 0.99$\pm$0.01 & 0.99$\pm$0.01  & 0.94$\pm$0.04 & 0.73$\pm$0.07 & 0.19$\pm$0.05\\
\cmidrule(r){1-7}
$GSR$ & n = 35   & 0.98$\pm$0.01 & 0.99$\pm$0.01  & 0.99$\pm$0.01 & 0.97$\pm$0.02  & 0.52$\pm$0.08 \\
     & n = 50    & 0.99$\pm$0.01 & 0.99$\pm$0.01  & 0.99$\pm$0.01 & 0.99$\pm$0.01 & 0.73$\pm$0.07 \\
\bottomrule
\end{tabular}
}
\end{table}

\begin{table}[ht] 
\centering{
  \caption{Detection power $P\_mean^2$ for variance change with a significance level around 2.5$\%$.}\label{tb:Contour_StaticVar}
\begin{tabular}{ccrrrrr}

 \toprule 
  &d              & 1   &10     & 50   &  100  & 500 \\
 \toprule 
$GLR$ & n = 35   & 0.98$\pm$0.02 & 0.99$\pm$0.01  & - & - & -  \\
     & n = 50    & 0.98$\pm$0.01 & 0.99$\pm$0.01  & - & - & - \\
\cmidrule(r){1-7} 
$T^2$ & n = 35   & 0.01$\pm$0.01 & 0.02$\pm$0.02  & 0.07$\pm$0.03 & - & -  \\
     & n = 50    & 0.01$\pm$0.01 & 0.02$\pm$0.02  & 0.07$\pm$0.03 & - & - \\
\cmidrule(r){1-7} 
$GEC$ & n = 35   & 0.26$\pm$0.07 & 0.72$\pm$0.08  & 0.09$\pm$0.04 & 0.06$\pm$0.03 & 0.04$\pm$0.02 \\
     & n = 50    & 0.46$\pm$0.08 & 0.97$\pm$0.02  & 0.26$\pm$0.06 & 0.18$\pm$0.06 & 0.00$\pm$0.00\\
\cmidrule(r){1-7}
$Kernel$ & n = 35   & 0.01$\pm$0.01 & 0.00$\pm$0.01  & 0.00$\pm$0.00 & 0.00$\pm$0.01 & 0.00$\pm$0.00 \\
     & n = 50    & 0.01$\pm$0.01 & 0.01$\pm$0.01  & 0.01$\pm$0.01 & 0.01$\pm$0.01 & 0.01$\pm$0.01\\
\cmidrule(r){1-7}
$GSR$ & n = 35   & 0.99$\pm$0.01 & 0.99$\pm$0.01  & 0.99$\pm$0.01 & 0.99$\pm$0.01   & 0.99$\pm$0.01 \\
     & n = 50    & 0.99$\pm$0.01 & 0.99$\pm$0.01  & 0.99$\pm$0.01 & 0.99$\pm$0.01 & 0.99$\pm$0.01 \\
\bottomrule
\end{tabular}
}
\end{table}
The GLR method is applicable only when the data dimension satisfies $d < n$, while Hotelling's method $T^2$ requires $d < (2n-1)$ to ensure that the test statistics can be computed.

\section{Distribution of the test-statistic}
In this section, we examine the distribution of the GSR test statistics to determine whether it varies with the distributional change of the observed data ($Y_i$). A sample of $2n$ observations is generated, where the first $n$ observations follow the default distribution, and the second $n$ observations follow either the default distribution (scenario of no change point) or an alternative distribution (scenario of a change in distribution). We compute the test statistic based on the simulated data and compare the distributions of the test statistics from the aforementioned scenarios. Figure~\ref{fig:ChgInDistrs} displays histograms of the test statistics as the changes in mean or variance gradually increase. For graph types CG, MST, and NNG, the histograms show that the distribution changes as the change in mean widens. In addition, distributional changes are observed for changes in variance. Note that in this example, the observations ($Y_i$) follow a Gaussian distribution. 

\begin{figure} 
  \hspace{80pt} \textbf{CG} \hspace{210pt} \textbf{CG} \\
  \includegraphics[width=.50\textwidth]
   {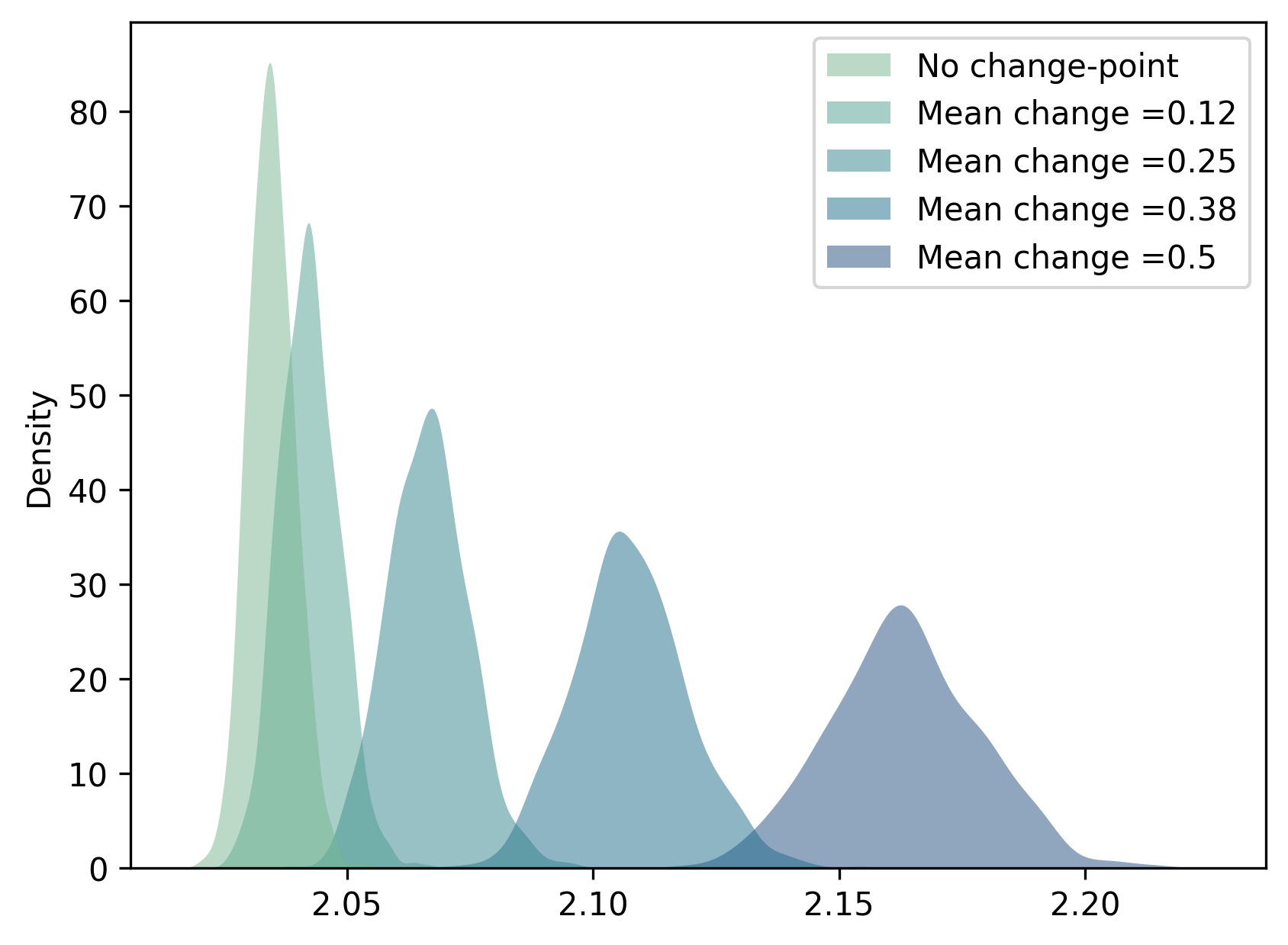}
   \includegraphics[width=.49\textwidth]
   {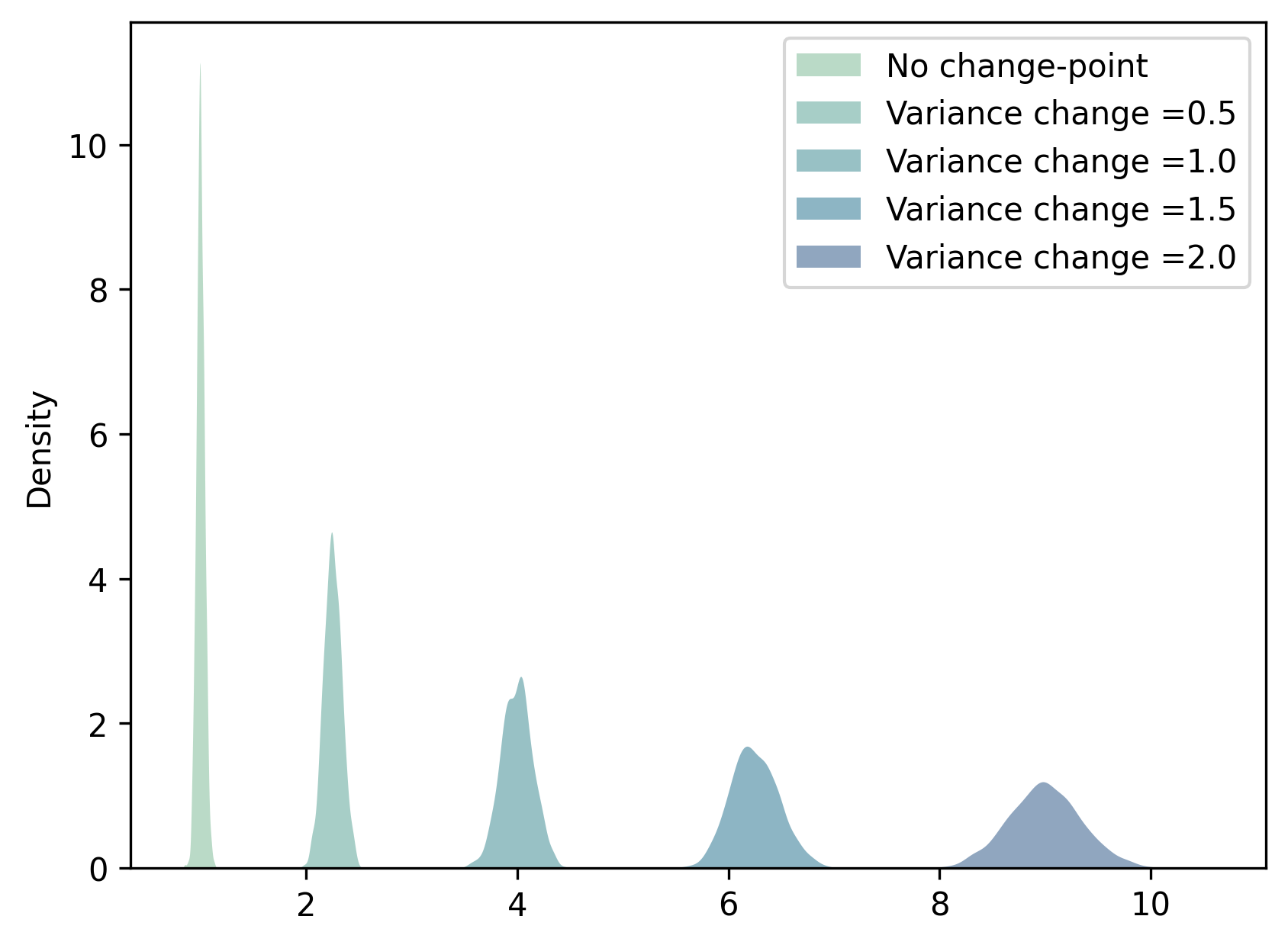}

     \hspace{80pt} \textbf{MST}  \hspace{204pt} 
    \textbf{MST}

   \includegraphics[width=.50\textwidth]
   {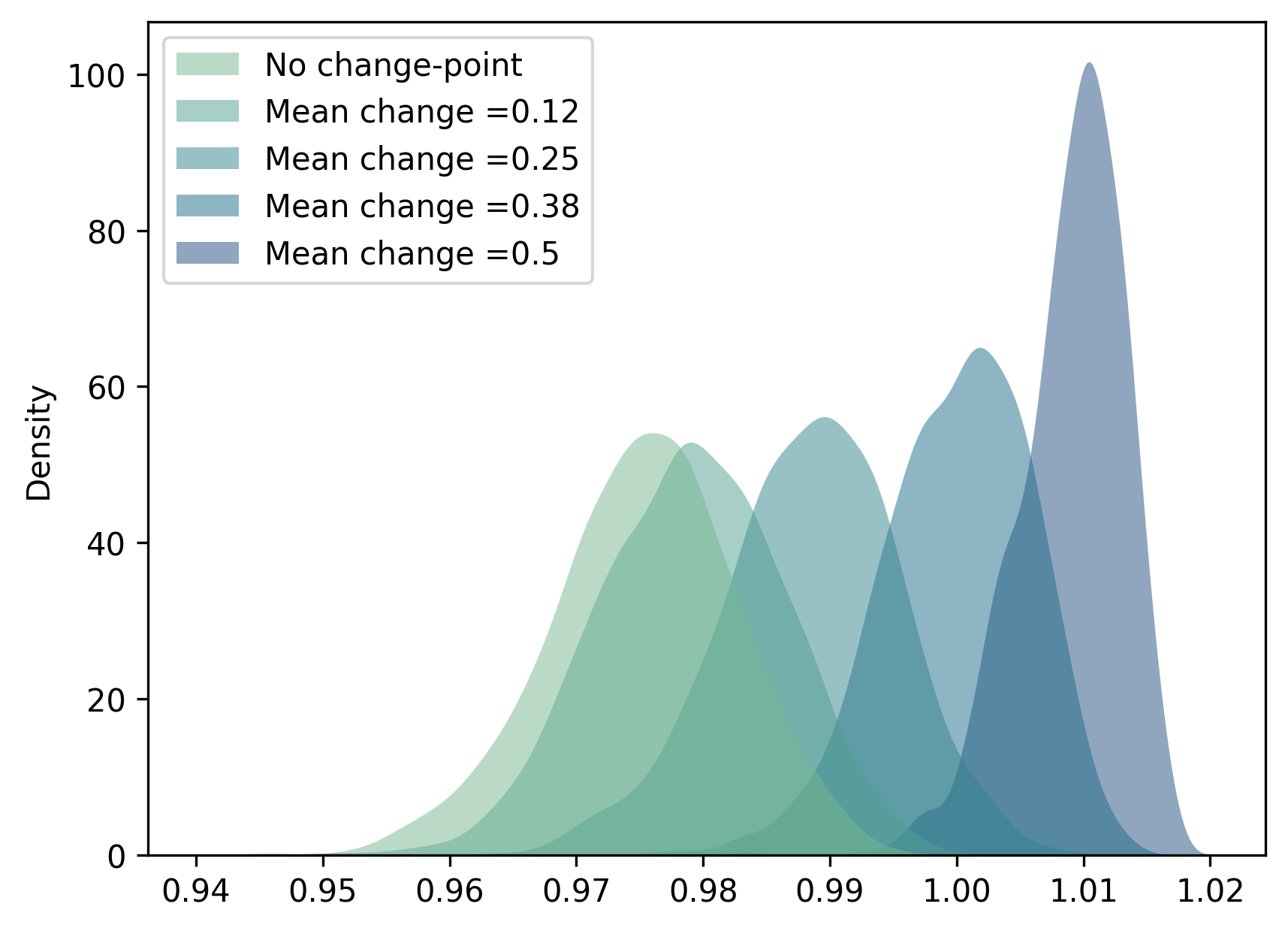}
  \includegraphics[width=.49\textwidth]
     {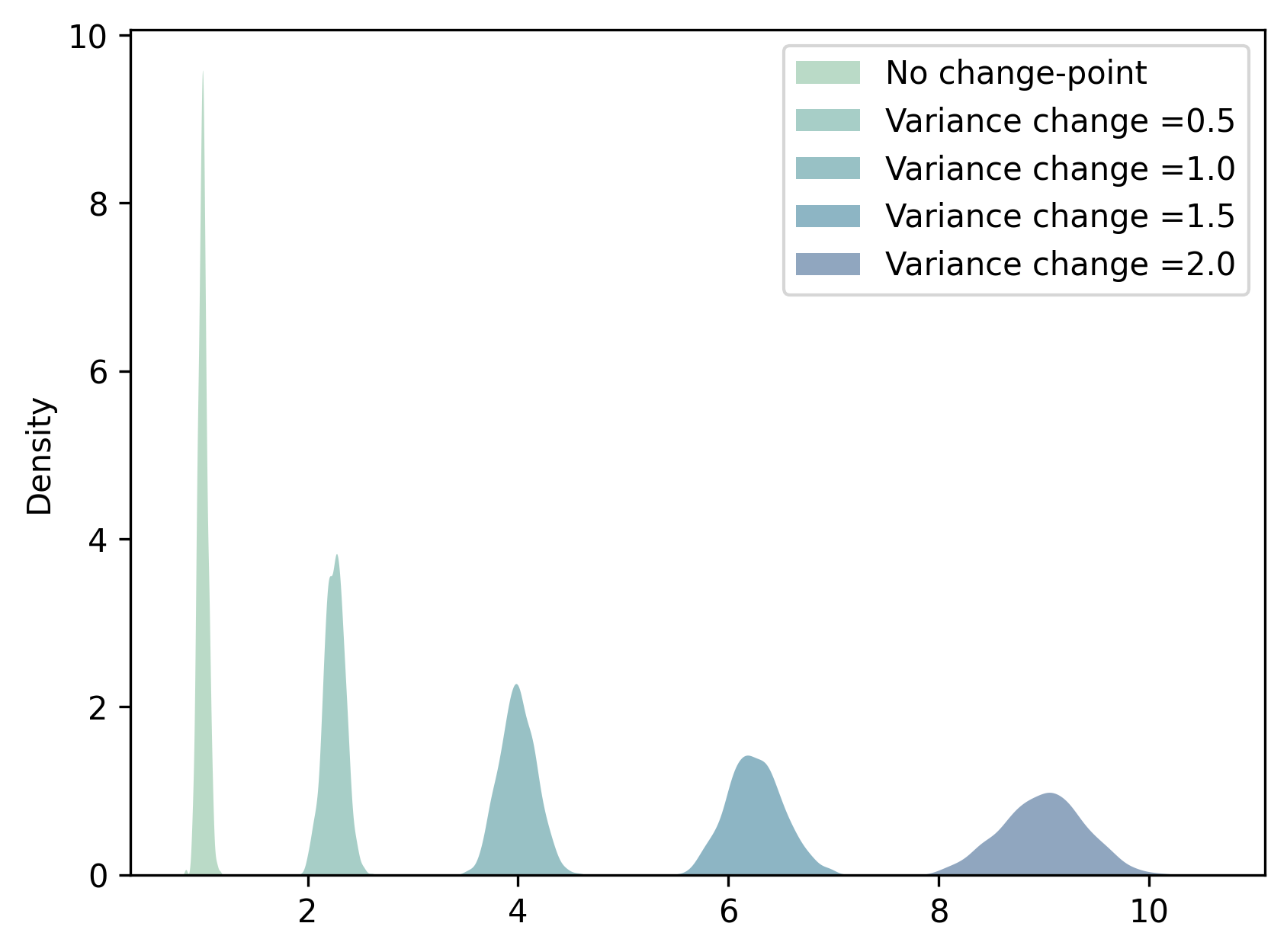}

     \hspace{80pt} \textbf{NNG}  \hspace{204pt} 
    \textbf{NNG}

   \includegraphics[width=.50\textwidth]
   {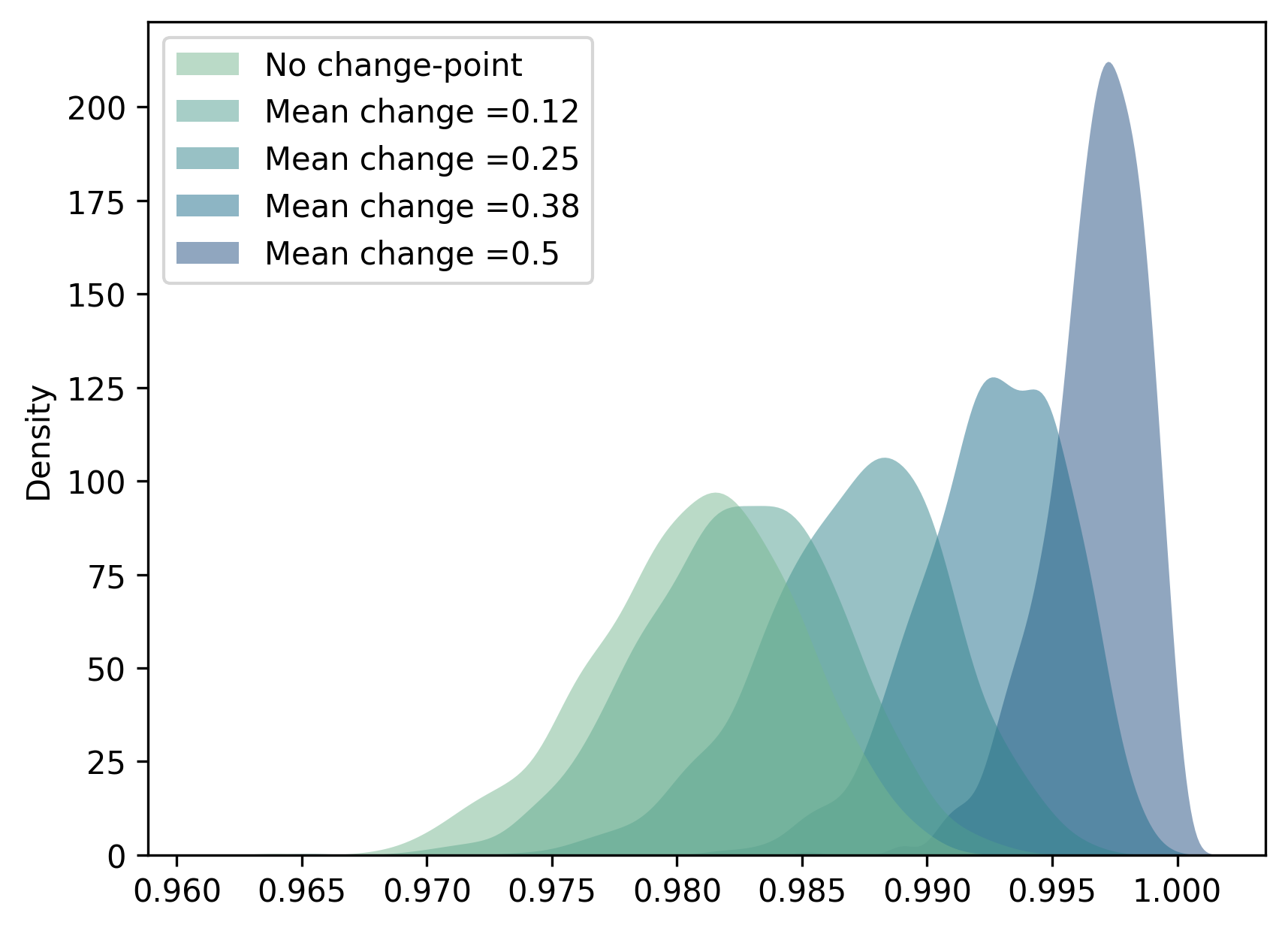}
  \includegraphics[width=.49\textwidth]
     {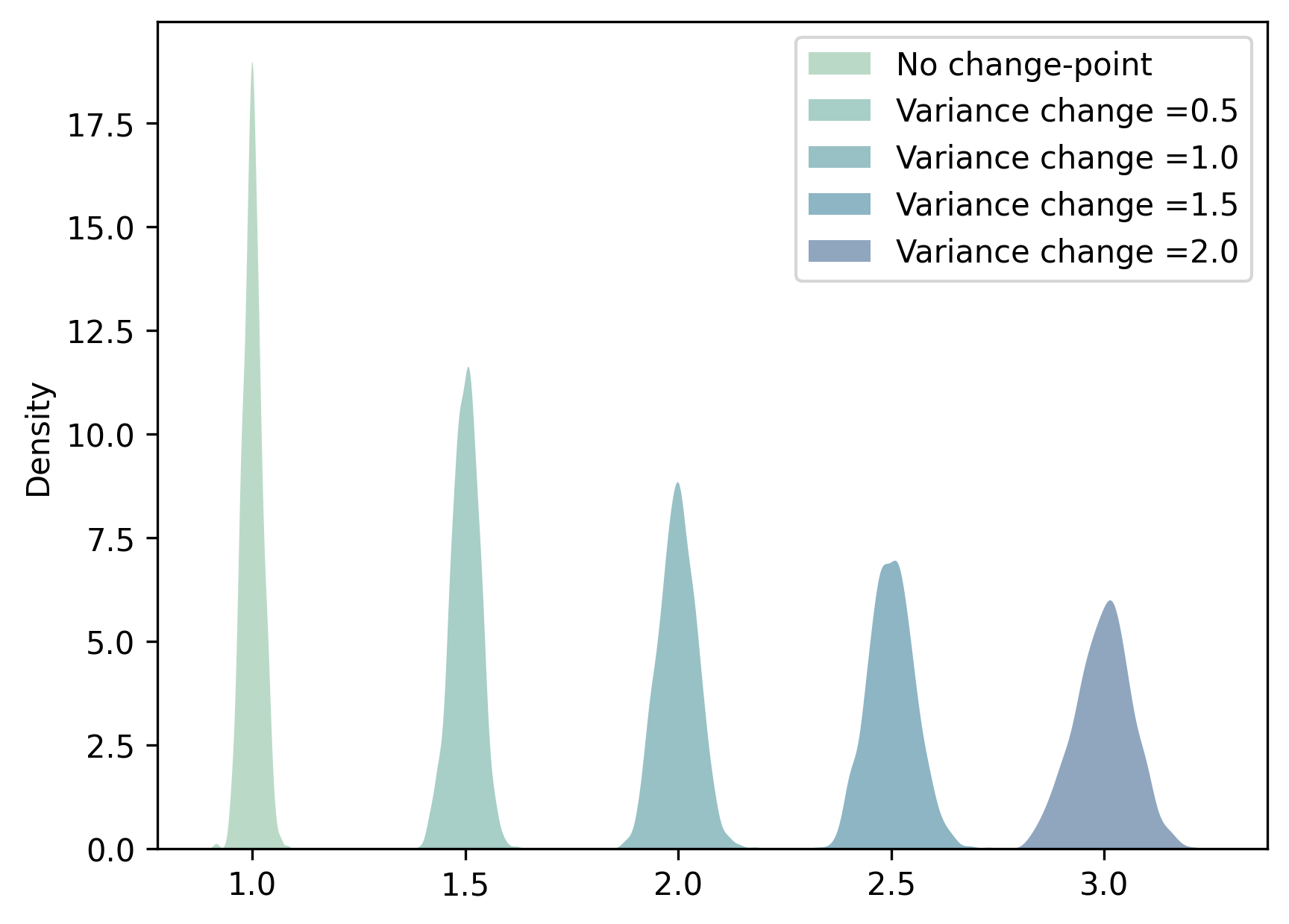}
   \setlength{\belowcaptionskip}{0pt}
  \caption{Distribution shift of the test statistics with respect to the change of mean and variance for various graph types: CG, MST, and NNG. The distribution of the  original observed data follows standard normal distribution with $d=10$, $n=30$.}
\label{fig:ChgInDistrs}
\end{figure}

\section{Algorithm}\label{sec: algo}
We provide the details of the algorithm in this section.
\begin{algorithm}[ht]
   \caption{Critical value by the Bootstrap/Permutation procedure: \\ AsymThreshold($Y_1,\ldots,Y_N, B, n, \alpha$)}\label{algo.onlineRao}
\begin{algorithmic}
   \STATE \textbf{Critical value estimation}: for each window length $n$ and location of the detection $k$, $k < 2n $, we estimate the $\alpha$-level critical value of the CPD test by the Permutation procedure (note: for known distribution, one can apply Monte Carlo simulation instead of Permutation procedure) 
   \STATE STEP 1. Calculate the test statistics
   \STATE {\bfseries Input:} Training data $Y_1,\ldots, Y_N$ of dimension $d$, $N > 2n$, significant level $\alpha$. Let  $t_0=2$, $t_L=2n-2t_0+1$\\
   \FOR{$b=1$ {\bfseries to} $B$}
        \STATE Generate $Y^b_1,\ldots, Y^b_{N}$ by resampling without replacement from $Y_1,\ldots, Y_m$  (Permutation procedure) 
        \FOR{$k=t_0$ {\bfseries to} $N-t_0$}
            \FOR{$t=n+1$ {\bfseries to} $N-n+1$}
           		\STATE Calculate  test statics:  \\$R^b_{\mu,n,k}(t) 
                 =  \frac{\| W_{G_{2n}(t)}\|^2-\frac{2n}{k} \|W_{G^l_k(t)} \|^2-\frac{2n}{2n-k} \|W_{G^r_{2n-k}(t)}\|^2}{\frac{2n}{k}\| W_{G^{l}_{k}(t)}\|^2+\frac{2n}{2n-k} \|W_{G^r_{2n-k}(t)}\|^2}$, \\
        	     $R^b_{\sigma+, n,k}(t)
        	     = \frac{(k-1) \| W_{G^r_{2n-k}(t)}\|^2}{(2n-k-1) \|W_{G^l_k}(t) \|^2}$, \\
        	    $R^b_{\sigma-, n,k}
        	     = \frac{(2n-k-1) \|W_{G^l_k(t)} \|^2}{(k-1) \| W_{G^r_{2n-k}(t) }\|^2}$. \\
            \ENDFOR
            \STATE Calculate $R^b_{\mu,n,k}= \max_t R^b_{\mu,n,k}(t)$, $R^b_{\sigma+,n,k}=\max_t R^b_{\sigma+,n,k}(t)$, $R^b_{\sigma-,n,k}=\max_t R^b_{\sigma-,n,k}(t)$
        \ENDFOR
   \ENDFOR
\end{algorithmic}
\end{algorithm}

\begin{algorithm}[ht]
\begin{algorithmic}
   \STATE STEP 2: Calibrate the critical value 
   \STATE \textbf{Initialize} $\alpha^*_0 = \alpha/{2t_L}$, $\alpha^*_1 =\alpha/2$, and $CP(b) = 0$, for $b= 1$ to $B$ \\
    \WHILE{$|\alpha^*_1 - \alpha| > 0.001$}
       \STATE $\alpha^*_0=\alpha^*_0+(\alpha-\alpha^*_1)/(2*t_L)$
       \STATE $ \rho^b_{\mu,n,k}$ as the $(1-\alpha^*_0)$ quantile of $R^b_{\mu,n,k}$
       \FOR{$b=1$ {\bfseries to} $B$}
          \FOR{$k=t_0$ {\bfseries to} $N-t_0$}  
            \IF {$R^b_{\mu,n,k} > \rho^b_{\mu,n,k}$}              
            \STATE $CP(b)=1$
            \ENDIF
         \ENDFOR
       \ENDFOR
       \STATE $\alpha^*_1=mean(CP)$
   \ENDWHILE
   \STATE Similarly for $ \rho^b_{\sigma+,n,k}$ and $ \rho^b_{\sigma-,n,k}$
   \STATE {\bfseries Output:}  $ \rho^b_{\mu,n,k}$,  $ \rho^b_{\sigma+,n,k}$, and $\rho^b_{\sigma-,n,k}$
\end{algorithmic}
\end{algorithm}

\begin{algorithm}[tb]
   \caption{GSR change point: 
   OnlineAsymDetection($Y_{t-n},\ldots,Y_{t-1}, Y_{t},\ldots,Y_{t+n-1}$, $n$, $\rho^b_{\mu,n,k}$, $\rho^b_{\sigma+,n,k}$ ,$\rho^b_{\sigma-,n,k}$)}\label{algo.onlineDetection}
\begin{algorithmic}
   \STATE Online detection of the change point located any point inside the scanning window
   \STATE {\bfseries Input:} Data $Y_{t-n},\ldots ,Y_{t-1}, Y_{t},\ldots,Y_{t+n-1}$
   \STATE \textbf{Initialize} $t_0=2$, $t_L=2n-2t_0+1$, $I_{\mu}=I_{\sigma+}=I_{\sigma-}=0$    
   \REPEAT 
   \FOR{$k=t_0$ {\bfseries to} $t_L+1$ }
   		\STATE  Calculate  test statics:  \\$R_{\mu,n,k}(t) 
         =  \frac{\| W_{G_{2n}(t)}\|^2-\frac{2n}{k} \|W_{G^l_k(t)} \|^2-\frac{2n}{2n-k} \|W_{G^r_{2n-k}(t)}\|^2}{\frac{2n}{k}\| W_{G^{l}_{k}(t)} \|^2+\frac{2n}{2n-k} \|W_{G^r_{2n-k}(t)}\|^2}$, \\
	     $R_{\sigma+, n,k}(t)
	     = \frac{(k-1) \| W_{G^r_{2n-k}(t) }\|^2}{(2n-k-1) \|W_{G^l_k(t) }\|^2}$, \\
	    $R_{\sigma-, n,k}
	     = \frac{(2n-k-1) \|W_{G^l_k(t)} \|^2}{(k-1) \| W_{G^r_{2n-k}(t) }\|^2}$. 
    
        \IF{$R_{\mu,n,k}(t)> \rho^b_{\mu,n,k}$}
               \STATE $I_{\mu}=1$ \textbf{return} Mean change at $t-n+k$
        \ENDIF
        \IF{$R_{\sigma+,n,k}(t) >  \rho^b_{\sigma+,n,k}$}
               \STATE $I_{\sigma+}=1$ \textbf{return} Variance increased at  $t-n+k$
        \ENDIF
        \IF{$R_{\sigma-,n,k}(t) >  \rho^b_{\sigma-,n,k}$}
               \STATE $I_{\sigma-}=1$ \textbf{return} Variance decreased at $t-n+k$
        \ENDIF
    \ENDFOR
   \UNTIL{$ I_{\mu}+I_{\sigma+}+I_{\sigma-}>0$}
\end{algorithmic}
\end{algorithm}

\end{appendix}
\begin{appendix}
\section{Proof of theorems}
\subsection{Validity of the Bootstrap procedure}
First, we define the normalized sum as

\[S_{Y,n} := \frac{1}{\sqrt{n}} \sum_{i=1}^n Y_i, \ \ \ \ S^b_{Y,n} := \frac{1}{\sqrt{n}} \sum_{i=1}^n Y^b_i.\]

Now we introduce the result of the accuracy of the bootstrap approximation form \cite{zhilova2022new}.

\begin{theorem}\label{theo:HDBoot}[Theorem 5.1 from \cite{zhilova2022new}]
Let $S_{Y,n}$ and $S^b_{Y,n}$ be defined as above. Assume $Y_i$ satisfies \hyperlink{SubGaussian}{sub-Gaussian condition} and $\mathbb{E}|Y_i^{\otimes 4}|< \infty $. Denote $\mathcal{A}$ a class of sets $A$ of all $\textit{l}_2$-balls. Then it hold with probability $\ge 1- n^{-1}$
\[
\sup_{A \in \mathcal{A}} \left| \mathds{P}(S_{Y,n} \in A) - \mathds{P}(S^b_{Y,n} \in A) \right| \le C_* \{ \sqrt{d^2/n} +d^2/n\},
\]
where $C_*$ depends on the moment of $Y_i$.
\end{theorem}

\begin{lemma}[The Delta Method]
Assume that we have a sequence of random variables $Z_1 ,\ldots,Z_n$, $Z_i \in \mathbb{R}^p $ such that
\[\sqrt{n}(\Bar{Z}_n-\mu_Z) \xrightarrow{} \mathcal{N}(0, \Sigma_{Z} ),\]
for some vector $\mu_Z$, and $\Sigma_Z= E[Z_1 Z_1^T] \in \mathbb{R}^{p\times p}$. Let $g: \mathbb{R}^p \xrightarrow{} \mathbb{R}^m$. If $\nabla g (\cdot)$ exists in a neighborhood of $\mu_Z$, $\nabla g(\mu_Z) \ne 0$, and if $\nabla g(\cdot)$ is continuous at $\mu_Z$, then using the Taylor expansion, $\sqrt{n}(g(\Bar{Z} - g(\mu_z)  \approx \sqrt{n}(\nabla g(\mu_z)\cdot(g(\Bar{Z}) - g(\mu_z) )$.
\end{lemma}

\begin{corollary}[Delta theorem for bootstrap]\label{coro:Delta}
    Let $Z_1,Z_2,\ldots,Z_n \overset{i.i.d}{\sim} \mathcal{F}_Z$, $Z_i \in \mathbb{R}^p, i=1
    \ldots, Z_n $. Let $\Sigma_{Z}= E[Z_1 Z_1^T] \in \mathbb{R}^{p\times p}$ be finite. Let $T$ be a function, $T(Z_1,Z_2,\ldots,Z_n) = \sqrt{n}(\Bar{Z}-\mu_Z)$ and for some $m \ge 1$, let $g: \mathbb{R}^p \xrightarrow{} \mathbb{R}^m$. If $\nabla g (\cdot)$ exists in a neighborhood of $\mu_Z$, $\nabla g(\mu_Z) \ne 0$, and if $\nabla g(\cdot)$ is continuous at $\mu_Z$, then the bootstrap is strongly consistent for $\sqrt{n}(g(\Bar{Z}) - g(\mu_Z)) $.   
\end{corollary}

\begin{theorem}[Bootstrap validity]\label{theorem:BootValidOffline}
Suppose that $Y_i$ satisfies the \hyperlink{SubGaussian}{sub-Gaussian condition} and $\mathbb{E}|Y_i^{\otimes 4} < \infty |$, then
\[\left|\mathds{P}( R_{\mu,n} \le \rho^b_{\mu,n}(\alpha)) - (1-\alpha) \right| \xrightarrow{}0,\]
as $n \gg d^2$.
\end{theorem}

\begin{proof}
We first show the consistency of bootstrap for the GSR test statistic of variance on the complete graph, and then the result can be generalized to the test of mean. 
    Let $Y^l_1, \ldots, Y^l_n \overset{i.i.d}{\sim} \mathcal{F}_Y$, $ Y^l_i \in \mathbb{R}^d$. 
    Let $Y^r_1, \ldots, Y^r_n \overset{i.i.d}{\sim} \mathcal{F}_Y$, $ Y^r_i \in \mathbb{R}^d$. 
    Let us define $Z_i \in \mathbb{R}^p$ where $p=4d$, and 
     \begin{equation}\label{eq:Z}
    \mathbf{Z} = \bigg[ \ Y^l_{i,1} \ \ \ldots \ \  Y^l_{i,d} \ \ \ (Y^l_{i,1})^2 \ \ \ldots \ \ (Y^l_{i,d})^2 \ \ \ Y^r_{i,1} \ \ \ldots \  \ Y^r_{i,d} \ \ \ (Y^r_{i,1})^2 \  \ \ldots \ \ (Y^r_{i,d})^2 \ \bigg]^T.  
    \end{equation} 
And the sample average of $Z$ is 
    \begin{equation}\label{eq:Zbar}
	     \mathbf{\Bar{Z}} =
		 \begin{bmatrix}
		 \Bar{Y^l_1} \\
		 \vdots\\
		 \Bar{Y^l_d}\\
		 \frac{1}{n} \sum_{i=1}^{n} (Y^l_{i,1})^2 \\
		 \vdots\\
		 \frac{1}{n} \sum_{i=1}^{n} (Y^l_{i,d})^2\\
    	 \Bar{Y^r_1} \\
		 \vdots\\
		 \Bar{Y^r_d}\\
		 \frac{1}{n} \sum_{i=1}^{n} (Y^r_{i,1})^2 \\
		 \vdots\\
		  \frac{1}{n} \sum_{i=1}^{n} (Y^r_{i,d})^2
		 \end{bmatrix},
    \end{equation} 
where $\Bar{Y^l_j} = \frac{1}{n} \sum_{i=1}^{n} Y^l_{i,j} $, $j=1,\ldots, d$, similarly for $\Bar{Y}^r_j$. 
Denote $\mu_Z := \mathbb{E}(Z_i)$, and let the covariance matrix of $Z$, $\Sigma_{Z}$ be finite. Denote $\bar{Z}^b$, $\mu^b_Z$ as the bootstrap counter part of $\bar{Z}$, $\mu_Z$, and $\mathds{P}^b(\cdot)=\mathds{P}(\cdot|Y_1, \ldots, Y_n)$, $\mathds{P}^b $ denote the probability measure under the bootstrap sample.
By Lemma \ref{theo:HDBoot}, then we have 

\[  \sup_{A \in \mathcal{A}} \left| \mathds{P}(\sqrt{n}(\Bar{Z}-\mu_Z) \in A) - \mathds{P}^b(\sqrt{n}(\Bar{Z^b}-\mu^b_Z) \in A) \right| \le C_b \{ \sqrt{d^2/n} +d^2/n\}.\] 

Let us consider the transformation function $g: \mathbb{R}^{4p} \xrightarrow{}\mathbb{R} $ be $R_{\sigma+,n}$, that is 
\[  g (\Bar{Z}) := R_{\sigma+,n}= \frac{\| W_{G^{r}_n}\|^2}{\| W_{G^{l}_n}\|^2}  =  \frac{ \sum_{j=1}^d \bigg(\sum_{i=1}^n (Y^r_{i,j})^2 - \Bar{Y^r_{j}}^2\bigg)} {\sum_{j=1}^d \bigg(\sum_{i=1}^n (Y^l_{i,j})^2 - \Bar{Y^l_{j}}^2\bigg)}. \]
Since $\nabla g(\mu_Z) \ne 0$, and $\nabla g(\cdot)$ is continuous at $\mu_Z$, which satisfies the condition in Lemma \ref{coro:Delta}, then it follows that the bootstrap is consistent.
Specifically, we define $\mathcal{A}$ a class of sets $A$ of all $\textit{l}_2$-balls. Then
\[  \sup_{A \in \mathcal{A}} \left| \mathds{P}(\sqrt{n}(\Bar{Z}-\mu_Z) \in A) - \mathds{P}^b(\sqrt{n}(\Bar{Z^b}-\mu^b_Z) \in A) \right| \le C_b \{ \sqrt{d^2/n} +d^2/n\}.\] 
The distribution of $\sqrt{n} (\Bar{Z}-\mu_Z)$ asymptotically approaches a Gaussian distribution (in $Y$ world). Similarly for the Bootstrap world, we have the Bootstrap counter part $\sqrt{n} (\Bar{Z^b}-\mu^b_Z)$ asymptotically approaches a Gaussian distribution (in $Y^b$ world).
As the Gaussian distribution is a continuous function, by Continuous Mapping Theorem, the difference between the Gaussian distributions (one in $Y$ world, the other in $Y^b$) is asymptotic zero.
This leads to the result
\[\left|\mathds{P}( R_{\mu,n} \le \rho^b_{\mu,n}(\alpha)) - (1-\alpha) \right| \xrightarrow{}0.\]

Bootstrap consistency for GSR tests statistics of the mean can be shown in a similar way. For graph type other than complete graph, as shown in Section \ref{sec:GaussianApprox} the characteristic function of the test statistics are approximately to that of the Gaussian case when $d \gg 1$, and $d^2 \ll n$. Therefore, based on this Bootstrap procedure, the error of the bootstrap approximation is small if the sample size $n$ is much larger than the square of its dimension $d$.
\end{proof}
\textbf{Proof of Theorem \ref{theorem:BootValidOnline}}\label{Proof:BootValidOnline}
\begin{theorem}[Bootstrap validity: online]
Suppose that $Y_i$ satisfies the \hyperlink{SubGaussian}{sub-Gaussian condition} and $\mathbb{E}|Y_i^{\otimes 4} < \infty |$, then
\[\left|\mathds{P}(\max_{t \in A_n} R_{\mu,n}(t) \le \rho^b_{max\mu,n}(\alpha)) - (1-\alpha) \right| \xrightarrow{}0,\]
where 
\[
\rho^b_{max\mu,n}(\alpha) = \inf\{ x: \mathds{P}^{b} 	\big( R^{max}_{\mu,n}  \geq x \big) \leq \alpha. \}
\]
\end{theorem}
\begin{proof}
Let $A_n = \{ t_1,\ldots,t_k \}$,
\begin{align*}
    \mathds{P}(\max_{t \in A_n} R_{\mu,n}(t) \le z) &= \mathds{P}(R_{\mu,n}(t_1) \le z, \ldots, R_{\mu,n}(t_k) \le z ).
\end{align*}

The relation holds for the bootstrap world
\[
  \mathds{P}^b(\max_{t \in A_n} R^b_{\mu,n}(t) \le z) = \mathds{P}^b(R^b_{\mu,n}(t_1) \le z, \ldots, R^b_{\mu,n}(t_k) \le z ).
\]
Therefore 
\begin{align*}
   \sup_{z > 0} \bigg| & \mathds{P}(\max_{t \in A_n} R_{\mu,n}(t) \le z) - \mathds{P}^b(\max_{t \in A_n} R^b_{\mu,n}(t) \le z) \bigg| \\
   & = \sup_{z > 0} \left| \mathds{P}(R_{\mu,n}(t_1) \le z, \ldots, R_{\mu,n}(t_k) \le z ) -\mathds{P}^b(R^b_{\mu,n}(t_1) \le z, \ldots, R^b_{\mu,n}(t_k) \le z ) \right| \\
   & \le \sup_{z > 0} \left|  \mathds{P}(R_{\mu,n}(t_1) \le z )-\mathds{P}^b (R^b_{\mu,n}(t_1) \le z ) \right|. \\
\end{align*}
We apply the theorem of bootstrap validity from offline (Theorem \ref{theorem:BootValidOffline}) to complete the proof.
\end{proof}
Similar results apply to the test statistics of variance. 

\subsection{Power of the test}\label{Theo:power}
We first construct the theorem on the basis of a complete graph with Gaussian distributed observation. Then we can extend the result to an unknown distribution with other graph type.

\textbf{Proof of Theorem \ref{Theo:Delta_mu}}\label{Proof:Delta_mu}
\begin{theorem}[Power of the test]
Let $ \mathds{T}_{\mu}$ be the pooled test statistics specified in Equation (\ref{eq:poolT}), and $\beta \in (0,1)$. Then
$\mathds{P} (  \mathds{T}_{\mu} > 0) \ge 1-\beta$, if 
 \[\sup_{n \in \mathfrak{N}} \{ \| \mu_{gap,n} \| ^2 - \Delta_{\mu}(n) \} \geq 0,\]
\begin{align*}
\Delta_{\mu}(n) = C_1 \Big( \| \mu^{l}_{G_{n}}\|^2+\| \mu^{r}_{G_{n}} \|^2 +  C_2 \sigma^2 \Big),
\end{align*}
where $\| \mu_{gap,n} \| ^2$,   $ \|\mu^{l}_{G_{n}}\|^2$,  and $\| \mu^{r}_{G_{n}} \|^2$ are the expected gap-spanning distance and the expected spanning distance of subgraphs $G^{l}_n$ and $G^{r}_n$, respectively.   

\begin{align*}
 C_1 = &5 \frac{N_n}{D_n} F^{-1}_{N_n, D_n} (\alpha_{\mu,n}),  \\
 C_2 = 
& \Bigg( D_n + 2 \sqrt{D_n \log\Big(\frac{2}{\beta}\Big)} +4 \log\Big(\frac{2}{\beta}\Big) \Bigg)\\
& - \frac{5}{4} \Bigg(N_n  - 2 \sqrt{N_n \ log \Big(\frac{2}{\beta}\Big)} - 10 \log \Big(\frac{2}{\beta}\Big) \Bigg),
\end{align*}
where $N_n = d$ and $ D_n = 2(n-1)d$.
\end{theorem}
\begin{proof}
\[  \mathds{T}_{\mu}=\sup_{n \in \mathfrak{N}} \Bigg\{ \frac{ \|W_{G_{2n}}\|^2 }{\Big( \|W_{G^{l}_{n}}\|^2 + \|W_{G^{r}_{n}}\|^2 \Big)}   - 2\frac{N_n }{D_n} F^{-1}_{N_n, D_n} (\alpha_{\mu,n}) \Bigg\}. \]

By the definition of $ \mathds{T}_{\mu}$,  $\mathds{P}(  \mathds{T}_{\mu} \le 0) \le \inf_{n \in \mathfrak{N}} P(n)$ where 

\begin{align*}
 P(n) &= \mathds{P} \Bigg(\frac{ \|W_{G_{2n}}\|^2 }{\Big( \|W_{G^{l}_{n}}\|^2 + \|W_{G^{r}_{n}}\|^2 \Big)}  \le 2\frac{N_n }{D_n}F^{-1}_{N_n,D_n}(\alpha_{\mu,n}) \Bigg) \\
 &= \mathds{P} \Bigg( \frac{ \|W_{gap,n}\|^2 }{\Big( \|W_{G^{l}_{n}}\|^2 + \|W_{G^{r}_{n}}\|^2 \Big)}  \le 2\frac{N_n }{D_n} F^{-1}_{D_n,N_n}(\alpha_{\mu,n}) -2  \Bigg). 
\end{align*}
The goal is to show $P(n) \le \beta$.Denote $Q (a, D, u)$ the $1-u$ quantile of a non-central $\chi^2$ random variable with $D$ degree of freedom and non-centrality parameter $a$.
For each $n  \in \mathfrak{N}$, we have
\[ 
	\| W_{G^{l}_{n}}\|^{2} +  \| W_{G^{r}_{n}}\|^{2}  \sim \chi^2_{D_n},
\]
with non-centrality parameter $ \|\mu_{G^{r}_n}\|^2+\|\mu_{G^{l}_n}\|^2$,  degree of freedom $D_n = 2(n-1) d$. And 
\[ 
	\|W_{gap,n}\|^{2} \sim 2\chi^2_{N_n},
\]
with non-centrality parameter  $\|\mu_{gap,n}\|^2$, degree of freedom $N_n = d $.
Note that the mean spanning distance for graph $G_{2n} $ under $H_0$ is
\begin{equation}\label{eq:mu}
\|\mu_{G_{2n}}\|^2 := \mathbb{E}_0 [\| W_{G_{2n}}\|^{2}] =2( \|\mu_{G^{r}_n}\|^2+\|\mu_{G^{l}_n}\|^2)+\|\mu_{gap,n}\|^2. 
 \end{equation}
\[
R_{\mu, n} = \frac{ \|W_{G_{2n}}\|^{2}}{ (\|W_{G^{l}_{n}}\|^{2} +  \|W_{G^{r}_{n}}\|^{2})}\ -2 \sim 2\frac{N_n }{D_n} F_{N_n, D_n}.
\]

Thus, the test-statistics $R_{\mu,n}$ follows Fisher distribution with $N_n$ and $D_n$ degrees of freedom. Hence, by \cite{baraud2003adaptive}
\begin{align*}
 P(n) &= \mathds{P} \Bigg( \frac{  \|W_{G^{gap}_{2n}}\|^2 }{\Big( \|W_{G^{l}_{n}}\|^2 + \|W_{G^{r}_{n}}\|^2 \Big)}  \le 2\frac{N_n }{D_n} F^{-1}_{N_n,D_n}(\alpha_{\mu,n})   \Bigg)  \\
 &\le \mathds{P} \Bigg(  \|W_{G^{gap}_{2n}}\|^{2} \leq  2\frac{N_n}{D_n} F^{-1}_{N_n, D_n}(\alpha_{\mu,n}) Q \Bigg(  \| \mu^{l}_{G_n} \|^2+\| \mu^{r}_{G_n} \|^2, D_n, \frac{\beta}{2}\Bigg)\Bigg) + \frac{\beta}{2}.
\end{align*}
Therefore, 
\[
\mathds{P} (\mathds{T}_{\mu}  \le 0) \le \beta, 
\]
if for some $n$ in $\mathfrak{N}$
\begin{equation}\label{eq:quotion}
2\frac{N_n}{D_n} F^{-1}_{N_n, D_n} (\alpha_{\mu,n}) Q\Bigg(\|\mu^{l}_{G_n}\|^2+\|\mu^{r}_{G_n}\|^2, D_n, \frac{\beta}{2}\Bigg) \le Q\Bigg(\|\mu_{gap,n}\|^2, N_n, 1-\frac{\beta}{2}\Bigg).
\end{equation}
By Lemma 3 from \cite{birge2001alternative}, we obtain
\[
Q(a, D, u) \le D+a+2 \sqrt{(D+2a) \log(1/u)} + 2 \log (1/u),
\]
\[
Q(a, D, 1-u) \geq  D+a-2 \sqrt{(D+2a) \log(1/u)}.
\]
Therefore,
\begin{align*}
&Q \Bigg( \| \mu^{l}_{G_n} \|^2+\| \mu^{r}_{G_n} \|^2, D_n, \frac{\beta}{2}\Bigg)\\
& \le D_n+ ( \| \mu^{l}_{G_n} \|^2+\| \mu^{r}_{G_n} \|^2) 
+ 2 \sqrt{D_n +2(  \| \mu^{l}_{G_n} \|^2+\| \mu^{r}_{G_n} \|^2) \log(2/\beta)}  
+ 2 \log(2/\beta) \\
&= D_n +(\| \mu^{l}_{G_n} \|^2+\| \mu^{r}_{G_n} \|^2) 
 +2 \sqrt{D_n \log(2/\beta)+2 (  \| \mu^{l}_{G_n} \|^2+\| \mu^{r}_{G_n} \|^2) \log(2/\beta)} 
 +2 \log(2/\beta).
\end{align*}
By the inequality $\sqrt{u+v} \le \sqrt{u} + \sqrt{v}$, and $2 \sqrt{uv} \le 1/2u+2v$, 
\begin{align}\label{eq:4log}
&Q \Bigg(  \| \mu^{l}_{G_n} \|^2+\| \mu^{r}_{G_n} \|^2, D_n, \frac{\beta}{2}\Bigg)  \\
& \le  D_n +  (\| \mu^{l}_{G_n} \|^2+\| \mu^{r}_{G_n} \|^2) + 2 \sqrt{D_n \log \Big(\frac{2}{\beta}\Big)} + 2 \sqrt{(\| \mu^{l}_{G_n} \|^2+\| \mu^{r}_{G_n} \|^2) 2 \log\Big(\frac{2}{\beta}\Big)} \\
&\le  D_n + 2(\| \mu^{l}_{G_n} \|^2+\| \mu^{r}_{G_n} \|^2) + 2 \sqrt{D_n \log \Big(\frac{2}{\beta}\Big)} +4 \log \Big(\frac{2}{\beta}\Big),
\end{align}
which follows 
\[
Q(a, D, 1-u) \geq D+a-2 \sqrt{(D+1a)\log(2/ \beta)}. \\
\]
We obtain
\begin{align*}
Q\Bigg(\|\mu_{gap,n}\|^2, N_n, 1-\frac{2}{\beta}\Bigg) &  \geq N_n + \|\mu_{gap,n}\|^2 - 2\sqrt{\Big(N_n + 2 \|\mu_{gap,n}\|^2 \Big) \log\Big(\frac{2}{\beta}\Big)} \\
\text{by the inequality: } & \sqrt{u+ v} \leq \sqrt{u} + \sqrt{v}\\
&\geq N_n + \|\mu_{gap,n}\|^2 - 2 \sqrt{N_n \log \Big(\frac{2}{\beta}\Big)}-2 \sqrt{2 \|\mu_{gap,n}\|^2 \log \Big(\frac{2}{\beta}\Big)}.\\
\text{by the inequality: } &2\sqrt{u v} \leq \theta u + \theta^{-1} v, \text{choose} \theta = 1/5\\
&\geq N_n +\|\mu_{gap,n}\|^2 - 2 \sqrt{N_n \log\Big(\frac{2}{\beta}\Big)} - \frac{1}{5} \|\mu_{gap,n}\|^2 - 10 \log\Big(\frac{2}{\beta}\Big)
\end{align*}
\begin{equation}\label{eq:10log}
= N_n +\frac{4}{5} \|\mu_{gap,n}\|^2 - 2 \sqrt{N_n \log(\frac{2}{\beta})} - 10 \log \Big(\frac{2}{\beta}\Big).
\end{equation}

Based on Equation (\ref{eq:quotion},) we have the following relation satisfied 
\[
2\frac{N_n}{D_n} F^{-1}_{N_n, D_n} (\alpha_{\mu,n}) Q\Bigg(\| \mu^{l}_{G_n} \|^2+\| \mu^{r}_{G_n} \|^2, D_n, \frac{\beta}{2}\Bigg) \le Q \Bigg(\|\mu_{gap,n}\|^2, N_n, 1-\frac{\beta}{2}\Bigg).
\]

Plug Equation (\ref{eq:4log}) and Equation (\ref{eq:10log}) into Equation (\ref{eq:quotion}), we obtain the following relation:

\begin{align*}
2\frac{N_n}{D_n} F^{-1}_{N_n, D_n} (\alpha_{\mu,n})  \Bigg( D_n + 2( \| \mu^{l}_{G_n} \|^2+\| \mu^{r}_{G_n} \|^2) + 2 \sqrt{D_n \log\Big (\frac{2}{\beta}\Big)} +4 \log \Big(\frac{2}{\beta}\Big) \Bigg)  \\
\leq N_n + \frac{4}{5} \|\mu_{gap,n}\|^2 - 2 \sqrt{D_n \log\Big(\frac{2}{\beta}\Big)} - 10 \log \Big(\frac{2}{\beta}\Big),
\end{align*}
\begin{align*}
 \frac{5}{2}\frac{N_n}{D_n}  F^{-1}_{N_n, D_n} (\alpha_{\mu,n}) \Bigg( D_n + 2 \Big( \| \mu^{l}_{G_n} \|^2+\| \mu^{r}_{G_n} \|^2 \Big) + 2 \sqrt{D_n \log (\frac{2}{\beta})} +4 \log \Big(\frac{2}{\beta}\Big) \Bigg)  \\
\leq  \frac{5}{4} \Bigg( N_n  - 2 \sqrt{N_n \log \Big(\frac{2}{\beta}\Big)} - 10 \log \Big (\frac{2}{\beta} \Big) \Bigg)+ \|\mu_{gap,n}\|^2.
\end{align*}
Rearrange the equation, we have
\begin{align*}
\|\mu_{gap,n}\|^2 \geq  &\Bigg( 5 \frac{N_n}{D_n} F^{-1}_{N_n, D_n} (\alpha_{\mu,n}) \Bigg) \Bigg(\| \mu^{l}_{G_n} \|^2+\| \mu^{r}_{G_n} \|^2 \Bigg) \\
&+\Bigg( 5\frac{N_n}{D_n} F^{-1}_{N_n, D_n} (\alpha_{\mu,n}) \Bigg)\Bigg( D_n + 2 \sqrt{D_n \log\Big(\frac{2}{\beta}\Big)} +4 \log\Big(\frac{2}{\beta}\Big) \Bigg) \\
 &-  \frac{5}{4} \Bigg(N_n  - 2 \sqrt{N_n \log\Big(\frac{2}{\beta}\Big)} - 10 \log \Big (\frac{2}{\beta}\Big) \Bigg).
\end{align*} 
Apply Equation (\ref{eq:mu}) we have derived the quantity
\begin{align*}
\|\mu_{G_{2n}}\|^2 \geq & \Bigg( 2 + 5 \frac{N_n}{D_n} F^{-1}_{N_n, D_n} (\alpha_{\mu,n}) \Bigg) \Bigg(\| \mu^{l}_{G_n} \|^2+\| \mu^{r}_{G_n} \|^2 \Bigg) \\
&+\Bigg( 5\frac{N_n}{D_n} F^{-1}_{N_n, D_n} (\alpha_{\mu,n}) \Bigg)\Bigg( D_n + 2 \sqrt{D_n \log\Big(\frac{2}{\beta}\Big)} +4 \log\Big(\frac{2}{\beta}\Big) \Bigg)\\
  &-  \frac{5}{4} \Bigg(N_n  - 2 \sqrt{N_n \log\Big(\frac{2}{\beta}\Big)} - 10 \log \Big (\frac{2}{\beta}\Big) \Bigg).
\end{align*} 
\end{proof}
Recall that we define $\| W_{gap,n} \|^2 $ as the sum of the quadratic 
distance between the nodes of $G^{l}_{n}$ and $G^{r}_{n}$. This metric tells the separation gap between data before and after the change point candidate.
		
\begin{align*}
\| W_{gap,n} \|^2 &=\|W_{G_{2n}}\|^2 - \|W_{G^{l}_{n}}\|^2 - \|W_{G^{r}_{n}}\|^2 \\
&= \sum\limits_{ i \in G^{l}_{n},  j \in G^{r}_{n}}\|Y_i -Y_j\|^{2}.
\end{align*}
Denote$  \| \mu_{gap,n}\|^2 = \mathbb{E}( \| W_{gap,n} \|^2)$.

\begin{corollary} \label{sp_gap}
		Assume $Y_i \overset{i.i.d}\sim \mathcal{N}(\mu, \sigma^2 \mathbb{I}_d)$, define the expected distance spanned between two complete graphs $ G^{l}_{n} $ and $ G^{r}_{n} $ as\[ 
		     \| \mu_{gap,n}\|^2 = 2\sigma^2 n^2d. 
		 \]
\end{corollary}
\begin{corollary}
 Given window size $n$,  $ \mathds{P} (  T_{\mu,n} > 0) \ge 1-\beta$, if
 \[  \| \mu_{gap,n} \| ^2 > (C_1-1) \Big( \| \mu^{l}_{G_{n}}\|^2+\| \mu^{r}_{G_{n}} \|^2 \Big) + C_2 \sigma^2, \]
where $\| \mu_{gap} \| ^2  \in \mathbb{R}$ is the mean separation spanning distance between graph $G^{l}_n$ and  $G^{r}_n$.
\end{corollary}
In other prospects, if the separation between before and after graphs is greater than the described quantity, then the detection power of $1-\beta$ is guaranteed.
\begin{proposition}[Power of the test- $\sigma+$]\label{prop:powertest_sigma}
Let $ T_{\sigma+,n}$ be the test statistics specified in the main paper and $\beta \in (0,1)$. Then, for any given fixed window length $n$, $ \mathds{P} (T_{\sigma+,n} > 0) \ge 1-\beta$, if 
\[  \| \mu^{r}_{G_{n}} \| ^2 \geq  C_1 \Big( \| \mu^{l}_{G_{n}} \|^2 \Big) + C_2 \sigma^2,\]
where $\| \mu^{r}_{G_{n}} \| ^2  \in \mathbb{R}$ is the mean spanning distance of graph $G^{r}_n$,  and
\begin{align*}
 C_1 =& \Bigg( \frac{5}{2} \frac{d^{r}_{n}}{d^{l}_{n}} F^{-1}_{d^{r}_{n}, d^{l}_{n}} (\alpha_{\sigma+,n}) \Bigg), \\
 C_2 =& \frac{5}{4}\frac{d^{r}_{n}}{d^{l}_{n}}  \Bigg( F^{-1}_{d^{r}_{n}, d^{l}_{n}} (\alpha_{\sigma+,n}) \Bigg) \\
 & \Bigg( d^{l}_{n} + 2 \sqrt{d^{l}_{n}\log\Big(\frac{2}{\beta}\Big)} +4 \log\Big(\frac{2}{\beta}\Big) \Bigg) \\
& -\frac{5}{4} \Bigg(d^{r}_{n}  - 2 \sqrt{d^{r}_{n}\ log \Big(\frac{2}{\beta}\Big)} - 10 \log \Big(\frac{2}{\beta}\Big) \Bigg), 
\end{align*}
where $d^{r}_{n} = d^{l}_{n} = (n-1)d$.
\end{proposition}\label{prop: vari+}
\begin{proof}

By definition of $ \mathds{T}_{\sigma+, n}$.  Let
\begin{align*}
 P(n) = \mathds{P}(  \mathds{T}_{\sigma+,n} \le 0)) &= \mathds{P} \Bigg(\frac{ \|W_{G^r_{n}}\|^2 }{ \|W_{G^{l}_{n}}\|^2 }  \le  \rho_{\sigma+,n}(\alpha_{\sigma+,n}) \Bigg). \\
\end{align*}
The goal is to show $P(n) \le \beta$. Denote $Q (a, D, u)$ the $1-u$ quantile of a non-central $\chi^2$ random variable with $D$ degree of freedom and non-centrality parameter $a$. For each $n  \in \mathfrak{N}$, we have
\[ 
	\| W_{G^{l}_{n}}\|^{2}   \sim \chi^2_{d^l_n}
\]
with non-centrality parameter $ \|\mu_{G^{l}_n}\|^2$. And 
\[ 
	\| W_{G^{r}_{n}}\|^{2}   \sim \chi^2_{d^r_n}
\]
with non-centrality parameter  $\|\mu_{G^{r}_n}\|^2$.
\[
R_{\sigma+, n} = \frac{ \|W_{G^{r}_{n}}\|^{2}}{ \|W_{G^{l}_{n}}\|^{2}} \sim \frac{d^r_n }{dl_n} F_{d^r_n, d^l_n}.
\]
Thus, the test-statistics $R_{\sigma+,n}$ follows Fisher distribution with $d^r_n$ and $d^l_n$ degrees of freedom. Hence, by \cite{baraud2003adaptive}
\begin{align*}
 P(n) &= \mathds{P} \Bigg( \frac{ \|W_{G^{r}_{n}}\|^{2}}{ \|W_{G^{l}_{n}}\|^{2}}  \le \frac{d^r_n }{dl_n} F^{-1}_{d^r_n, d^l_n} (\alpha_{\sigma+,n}) \Bigg)  \\
 &\le \mathds{P} \Bigg(  \|W_{G^{r}_{n}}\|^{2} \leq  \frac{d^r_n }{dl_n} F^{-1}_{d^r_n, d^l_n}(\alpha_{\sigma+,n}) Q \Bigg(  \| \mu^{l}_{G_n} \|^2 , d^l_n, \frac{\beta}{2}\Bigg)\Bigg) + \frac{\beta}{2}.
\end{align*}
Therefore, 
\[
\mathds{P} (\mathds{T}_{\mu}  \le 0) \le \beta, 
\]
if for some $n$ in $\mathfrak{N}$
\begin{equation}\label{eq:quotionS+}
\frac{d^r_n }{dl_n} F^{-1}_{d^r_n, d^l_n}(\alpha_{\sigma+,n}) Q \Bigg(  \| \mu^{l}_{G_n} \|^2 , d^l_n, \frac{\beta}{2}\Bigg)\Bigg) \le Q\Bigg(\|\mu^{r}_{G_{n}}\|^2, d^r_n, 1-\frac{\beta}{2}\Bigg).
\end{equation}
By Lemma 3 from \cite{birge2001alternative}, we obtain
\[
Q(a, D, u) \le D+a+2 \sqrt{(D+2a) \log(1/u)} + 2 \log (1/u),
\]
\[
Q(a, D, 1-u) \geq  D+a-2 \sqrt{(D+2a) \log(1/u)}.
\]
Therefore
\begin{align*}
Q \Bigg( \| \mu^{l}_{G_n} \|^2,  d^l_n, \frac{\beta}{2}\Bigg)
& \le d^l_n+ ( \| \mu^{l}_{G_n} \|^2) \\
& \;  \; \; \;+ 2 \sqrt{d^l_n +2(  \| \mu^{l}_{G_n} \|^2) \log(2/\beta)}  \\
& \;  \; \; \;+ 2 \log(2/\beta) \\
&= d^l_n +(\| \mu^{l}_{G_n} \|^2+\| \mu^{r}_{G_n} \|^2) \\
& \;  \; \; \;+2 \sqrt{D_n \log(2/\beta)+2 (  \| \mu^{l}_{G_n} \|^2+\| \mu^{r}_{G_n} \|^2) \log(2/\beta)} \\
& \;  \; \; \; +2 \log(2/\beta). \\
\end{align*}
By the inequality $\sqrt{u+v} \le \sqrt{u} + \sqrt{v}$, and $2 \sqrt{uv} \le 1/2u+2v$, 

\begin{equation*}
Q \Bigg(  \| \mu^{l}_{G_n} \|^2,  d^l_n, \frac{\beta}{2}\Bigg)  
 \le  d^l_n +  (\| \mu^{l}_{G_n} \|^2) + 2 \sqrt{d^l_n \log \Big(\frac{2}{\beta}\Big)} + 2 \sqrt{(\| \mu^{l}_{G_n} \|^2) 2 \log\Big(\frac{2}{\beta}\Big)} \\
\end{equation*}

\begin{equation}\label{eq:4logS}
\le  d^l_n + 2(\| \mu^{l}_{G_n} \|^2) + 2 \sqrt{d^l_n \log \Big(\frac{2}{\beta}\Big)} +4 \log \Big(\frac{2}{\beta}\Big).
\end{equation}
Based on the inequality
\[
Q(a, D, 1-u) \geq D+a-2 \sqrt{(D+1a)\log(2/ \beta)}, \\
\]
we obtain
\begin{align*}
Q\Bigg(\|\mu^{r}_{G_{n}}\|^2, d^r_n, 1-\frac{2}{\beta}\Bigg) &  \geq d^r_n + \|\mu^{r}_{G_{n}}\|^2 - 2\sqrt{\Big(d^r_n + 2 \|\mu^{r}_{G_{n}}\|^2 \Big) \log\Big(\frac{2}{\beta}\Big)} \\
\text{by the inequality } &\sqrt{u+ v} \leq \sqrt{u} + \sqrt{v}\\
&\geq d^r_n + \|\mu^{r}_{G_{n}}\|^2 - 2 \sqrt{d^r_n \log \Big(\frac{2}{\beta}\Big)}-2 \sqrt{2 \|\mu^{r}_{G_{n}}\|^2 \log \Big(\frac{2}{\beta}\Big)}\\
\text{by the inequality } &2\sqrt{u v} \leq \theta u + \theta^{-1} v, \text{ choose } \theta = 1/5\\
&\geq d^r_n +\|\mu^{r}_{G_{n}}\|^2 - 2 \sqrt{d^r_n \log\Big(\frac{2}{\beta}\Big)} - \frac{1}{5} \|\mu^{r}_{G_{n}}\|^2 - 10 \log\Big(\frac{2}{\beta}\Big)
\end{align*}
\begin{equation}\label{eq:10logS}
= d^r_n +\frac{4}{5} \|\mu^{r}_{G_{n}}\|^2 - 2 \sqrt{d^r_n \log(\frac{2}{\beta})} - 10 \log \Big(\frac{2}{\beta}\Big).
\end{equation}
From Equation (\ref{eq:quotionS+}) we have
\[
\frac{d^r_n}{d^l_n} F^{-1}_{d^r_n, d^l_n} (\alpha_{\sigma+,n}) Q\Bigg(\| \mu^{l}_{G_n} \|^2, d^l_n, \frac{\beta}{2}\Bigg) \le Q \Bigg(\|\mu^{r}_{G_{2}}\|^2, d^r_n, 1-\frac{\beta}{2}\Bigg).
\]

Plugging in Equation (\ref{eq:4logS}) and Equation (\ref{eq:10logS}) into Equation (\ref{eq:quotionS+}), we obtain the following relation:

\begin{align*}
\frac{d^r_n}{d^l_n} F^{-1}_{d^r_n, d^l_n} (\alpha_{\sigma+,n})  \Bigg( d^l_n + 2( \| \mu^{l}_{G_n} \|^2) + 2 \sqrt{D_n \log\Big (\frac{2}{\beta}\Big)} +4 \log \Big(\frac{2}{\beta}\Big) \Bigg)  \\
\leq d^r_n + \frac{4}{5} \|\mu^{r}_{G_{n}}\|^2 - 2 \sqrt{d^l_n \log\Big(\frac{2}{\beta}\Big)} - 10 \log \Big(\frac{2}{\beta}\Big),
\end{align*}
\begin{align*}
 \frac{5}{4}\frac{d^r_n}{d^l_n}  F^{-1}_{d^r_n, d^l_n} (\alpha_{\sigma+,n}) \Bigg( d^l_n + 2 \Big( \| \mu^{l}_{G_n} \|^2 \Big) + 2 \sqrt{d^l_n \log (\frac{2}{\beta})} +4 \log \Big(\frac{2}{\beta}\Big) \Bigg)  \\
\leq  \frac{5}{4} \Bigg( d^r_n  - 2 \sqrt{d^r_n \log \Big(\frac{2}{\beta}\Big)} - 10 \log \Big (\frac{2}{\beta} \Big) \Bigg)+ \|\mu^{r}_{G_{n}}\|^2.
\end{align*}
Rearrange the equation, we have derived the quantity 
\begin{align*}
\|\mu^{r}_{G_{n}}\|^2 \geq  &\Bigg( \frac{5}{2} \frac{d^r_n}{d^l_n} F^{-1}_{d^r_n, d^l_n} (\alpha_{\sigma+,n}) \Bigg) \Bigg(\| \mu^{l}_{G_n} \|^2 \Bigg) \\
&+\Bigg( \frac{5}{4}\frac{d^r_n}{d^l_n} F^{-1}_{d^r_n, d^l_n} (\alpha_{\sigma+,n}) \Bigg)\Bigg( d^l_n + 2 \sqrt{d^l_n \log\Big(\frac{2}{\beta}\Big)} +4 \log\Big(\frac{2}{\beta}\Big) \Bigg)\\
 &-  \frac{5}{4} \Bigg(d^r_n  - 2 \sqrt{d^r_n \log\Big(\frac{2}{\beta}\Big)} - 10 \log \Big (\frac{2}{\beta}\Big) \Bigg),
\end{align*} 
where $d^r_n = d^l_n= (n-1)d$. 
\end{proof}
\begin{proposition}[Power of the test- $\sigma-$]
Let $ T_{\sigma-,n}$ be the test statistics specified as Equation (5) and $\beta \in (0,1)$. Then, for any given fixed window length $n$, $ \mathds{P} (  T_{\sigma-,n} > 0) \ge 1-\beta$, if 
 
\[ \| \mu^{l}_{G_{n}} \| ^2 \geq  C_1 \Big( \| \mu^{r}_{G_{n}} \|^2 \Big) + C_2 \sigma^2, \]

where $\| \mu^{l}_{G_{n}} \| ^2  \in \mathbb{R}$ is the mean spanning distance of graph $G^{l}_n$, and
\begin{align*}
 C_1 =& \Bigg( \frac{5}{2} \frac{d^{l}_{n}}{d^{r}_{n}} F^{-1}_{d^{l}_{n}, d^{r}_{n}} (\alpha_{\sigma-,n}) \Bigg), \\
 C_2 = & \frac{5}{4}\frac{d^{l}_{n}}{d^{r}_{n}}  \Bigg( F^{-1}_{d^{l}_{n}, d^{r}_{n}} (\alpha_{\sigma-,n}) \Bigg) \\
 & \Bigg( d^{r}_{n} + 2 \sqrt{d^{r}_{n}\log\Big(\frac{2}{\beta}\Big)} +4 \log\Big(\frac{2}{\beta}\Big) \Bigg) \\
& -\frac{5}{4} \Bigg(d^{l}_{n}  - 2 \sqrt{d^{l}_{n}\log \Big(\frac{2}{\beta}\Big)} - 10 \log \Big(\frac{2}{\beta}\Big) \Bigg). 
\end{align*}
\end{proposition}
\begin{proof}
It can be shown in a similar fashion by exchange $\| W^l_{G_n} \|^2$  with $\| W^r_{G_n} \|^2$ in the proof in Proposition \ref{prop: vari+}.
\end{proof}
\subsubsection{Minimum radius of the mean separation}
We derive the minimal radius, that is, the lower bound of minimax separation rate, based on the result from \cite{baraud2003adaptive}, and \cite{baraud2002non}. 
To measure the performance of the test at a fix window size $n$, we denote a quantity $\rho_n(\mathcal{F}_1,\phi_{\alpha},\delta)$ by
\begin{align*}
\rho_n(\mathcal{F}_1,\phi_{\alpha},\delta) &= \inf \{\rho > 0,  \inf_{ {\mathcal{F}_1}, \| \mu_{gap,n}\| \geq \rho }  P[\phi_{\alpha}=1] \geq 1- \delta \}  \\
  &= \inf \{\rho > 0,  \sup_{ {\mathcal{F}_1}, \| \mu_{gap,n}\|  \geq \rho }  P[\phi_{\alpha}=0] \leq \delta \}, 
\end{align*} 
where $\phi_{\alpha}$ is the test result that corresponds to the test statistics
 \begin{align*}
 T_{\mu,n} &=  \frac{ \|W_{G_{n}}\|^2 }{\Big( \|W_{G^{l}_{n}}\|^2 + \|W_{G^{r}_{n}}\|^2 \Big)} - 2 - \frac{N_n}{D_n}F^{-1}_{N_n, D_n} (\alpha_{\mu,n}) \\
  &=  \frac{\|W_{gap,n}\|^2 }{\Big( \|W_{G^{l}_{n}}\|^2 + \|W_{G^{r}_{n}}\|^2 \Big)}  - \frac{N_n}{D_n}F^{-1}_{N_n, D_n} (\alpha_{\mu,n}). 
 \end{align*}

Let us introduce a test statistic $\hat{T}_{\mu, n} $ for window size $n$
\[\hat{T}_{\mu, n} = \| W_{gap,n}\|^2 - \sigma^2 \chi^2_{N_n}(\alpha_{\mu,n}),\] 
and denote its corresponding test as $\hat{\phi}_{\alpha}$.
The following lemma gives an analogous argument for the test, which will enable us to derive the lower bound of the minimal radius in a concise way. We assume that the spanning distance of subgraphs $G^l_n$ and $G^r_n$ is greater than 1, and the variance remains unchanged, then we can derive the following properties.

\begin{lemma} \label{lemma:simpleTn}
The minimal radius of $ \| \mu_{gap,n}\| $ derived from test $\phi_{\alpha}$ is the same as that from test $\hat{\phi}_{\alpha}$ 
\[\rho_n(\mathcal{F}_1,\phi_{\alpha},\delta) = \rho_n(\mathcal{F}_1,\hat{\phi}_{\alpha},\delta). \]
 \end{lemma}

\begin{proof}
Let $U, V$  be independent random variables,   $U, \geq 0,  V\geq 0$ and  bounded. Let $Z=U V$ be the product of the two random variables.
Define $\mathfrak{z}_\alpha, \mathfrak{v}_{\alpha}$ as the $1-\alpha$ quantile of the random variable $Z,$ and $V$. 
Then we have the following relation between the cumulative distribution functions,
\begin{align*}
F_Z(Z \leq U  \mathfrak{v}_{\alpha}) &= \mathds{P}(Z \leq  U  \mathfrak{v}_{\alpha}) \\
&= \mathds{P}(UV\leq  U  \mathfrak{v}_{\alpha},U\geq 0) +\mathds{P}(UV\leq U  \mathfrak{v}_{\alpha},U\leq 0) \\
&= \int_{0}^{\infty} f_U(u)\int_{-\infty}^{\mathfrak{v}_{\alpha}} f_V(v) dvdu \\
&= \int_{0}^{\infty} f_U(u)  F_V(V\leq \mathfrak{v}_{\alpha}) du \\
&= F_V(V\leq \mathfrak{v}_{\alpha}) \\
&= \alpha. 
\end{align*}
Thus $F_Z(Z \leq  \mathfrak{z}_{\alpha})=F_Z(Z \leq U  \mathfrak{v}_{\alpha})=F_V(V\leq \mathfrak{v}_{\alpha})$. \\
Analogously, let $Z= \tilde{T}_{\alpha,n}$ and $V=\hat{T}_{\alpha,n}$.
Then $\rho_n(\mathcal{F}_1,\phi_{\alpha},\delta) = \rho_n(\mathcal{F}_1,\hat{\phi}_{\alpha},\delta)$ is satisfied.
\end{proof}
\begin{proposition}\label{prop:Baraud}
 Let
\[ \rho^2_{N_n} = \sqrt{2 \log (1+4(1-\alpha_{\mu,n}-\beta)^2){N_n}} \sigma^2. \]
Then, for all $\rho \leq \rho_{N_n}$
\[ \beta(\{\{Y_i, i> t\} \sim {\mathcal{F}_1} , \| \mu_{gap,n}\| =\rho\}) \geq \delta. \]
\end{proposition}
According to \cite{baraud2002non}, whatever the level-$\alpha$ test $\hat{\phi}_{\alpha_n}$, there exist some observation $\{Y_i, i> t\}$ satisfying $\| \mu_{gap,n}\| =\rho_{D_n}\ $ for which the error of the second kind $P[\hat{\phi}_{\alpha}=0] $ is at least $\delta$. This implies the lower bound, 
\[ \rho_n(\mathcal{F}_1,\hat{\phi}_{\alpha},\delta) \geq \rho_{D_n}. \]
\begin{proof}
The idea of the proof is based on \cite{baraud2002non}.
Let $Y_i, i=1,\ldots,n \sim \mathcal{F}_0=\mathcal{N}(0, I_d),  Y_i, i =n+1,\ldots,2n \sim \mathcal{F}_1$.
Let $\mu_{\rho}$ be some joint probability measure on 
\[ \mathcal{F}_1[\rho] = \{ \| \mu_{gap,n}\| =\rho \}. \]
Setting  $P_{\mu_{\rho}} = \int P d\mu_{\rho}$ and denoting by $\Hat{\Phi}_{\alpha}$ the set of level-$\alpha$ tests, we have
\begin{align*}
\beta(\mathcal{F}_1[\rho]) &=\inf_{\hat{\phi}_{\alpha} \in \Hat{\Phi}_{\alpha}} \sup_{ \mathcal{F}_1[\rho] } P[\phi_{\alpha}=0] \\
& \geq \inf_{\hat{\phi}_{\alpha} \in \Hat{\Phi}_{\alpha}} P_{\mu_{\rho}} [\phi_\alpha =0] \\ 
 &\geq 1-\alpha - \sup_{A|P_0(A) \leq \alpha} | P_{\mu_{\rho}} (A) -P_0(A)| \\
 &\geq 1-\alpha - \sup_{A\in \mathcal{A}} | P_{\mu_{\rho}} (A) -P_0(A)| \\
 &= 1- \alpha - \frac{1}{2}\| P_{\mu_{\rho}} - P_0 \|,
\end{align*}
where  $\|P_{\mu_{\rho}} - P_0 \|$ denotes the total variation nor between the probabilities $P_{\mu_{\rho}}$ and $P_0$.  Assume $P_{\mu_{\rho}}$ is absolutely continuous with respect to $P_0$.  We denote 
\[ L_{\mu_{\rho}}(y)=\frac{dP_{\mu_{\rho}}}{dP_0},  \]
then
\begin{align*}
 \|P_{\mu_{\rho}} -P_0 \| &= \int |L_{\mu_{\rho}}(y)-1|dP_0(y), \\
 \mathbb{E}_0[L_{\mu_{\rho}}(y)-1] &\leq (\mathbb{E}_0[L^2_{\mu_{\rho}}(y)]-1)^{1/2}. 
\end{align*}
We obtain 
\begin{align*} \label{eq:beta}
 \beta(\mathcal{F}_1[\rho]) &\geq 1- \alpha-\frac{1}{2}(\mathbb{E}_0[L^2_{\mu_{\rho}}(y)]-1)^{1/2}\\
 &\geq 1- \alpha-(\mathbb{E}_0[L^2_{\mu_{\rho}}(y)]-1)^{1/2}\\
 &\geq 1 - \alpha-\eta,
\end{align*}
where we set $\eta  \geq (\mathbb{E}_0[L^2_{\mu_{\rho}}(y)]-1)^{1/2}$, equivalently,  $\mathbb{E}_0[L^2_{\mu_{\rho}}(y)] \geq 1+ {\eta}^2.$

Next step is to find some $\rho^*(\eta)$ such that for all $\rho \leq \rho^*(\eta)$,
\begin{equation}\label{eq:Erho}
\mathbb{E}_0[L^2_{\mu_{\rho}}(y)] \geq 1+ {\eta}^2,
\end{equation}

so that 
\[ \beta(\mathcal{F}_1[\rho]) \geq 1-\alpha -\eta =\beta \]
is satisfied.

Let $\epsilon=(\epsilon_j)_{j\in I}, I=\{1,...,N_n\}$ be a sequence of Rademacher random variables, i.e., for each $m$, $\epsilon_j$ are independent and identically distributed random variables taking values form $\{-1,1\}$ with probability $\frac{1}{2}$.  Let $\rho$ be given and $\mu_{\rho}$ be the distribution of the random variable $\sum_{j\in I} \lambda \epsilon_j e_j$, where $\lambda=\rho/\sqrt{N_n}.$ Clearly $\mu_{\rho}$ supports $\mathcal{F}_1[\rho]$.  We derive $L_{\mu_{\rho}}$ as
\begin{align*}
L_{\mu_{\rho}}(y)&=\frac{dP_{\mu_{\rho}}}{dP_0} \\
&= \mathbb{E}_\epsilon \Bigg[ \frac{\exp(-\frac{1}{2}\sum_{j\in I}(y_i-\lambda \epsilon_j)^2}{\exp(-\frac{1}{2}\sum_{j\in I}y^2_i)} \Bigg] \\
&=\mathbb{E}_\epsilon \Bigg[ \exp(-\frac{1}{2}\rho^2+\lambda\sum_{j\in I} \epsilon_j y_j) \Bigg] \\
&= e^{-\rho^{2} /2}  \prod_{j\in I} \cosh(\lambda y_j),
\end{align*}
where $y_i \sim \mathbf{N}(0,1)$.
Next, we compute $\mathbb{E}_0[L^2_{\mu_{\rho}(y)}]$
\begin{align*}
\mathbb{E}_0[L^2_{\mu_{\rho}(y)}] &= e^{-\rho^{2} /2}  \mathbb{E}_0 \Bigg[ \prod_{j\in I} \cosh^2(\lambda y_j) \Bigg]\\
&= {\cosh(\lambda^2)}^{N_n} \\
& \leq (\exp(\frac{\lambda^4}{2}))^{N_n}\\
&= \exp (\frac{\rho^4}{2N_n}).
\end{align*}
By Equation \ref{eq:Erho}, we set 
\begin{align*}
\ln \mathbb{E}_0[L^2_{\mu_{\rho}(y)}]  \leq \frac{\rho^4}{2N_n} = \ln(1+ \eta^2).
\end{align*}

Therefore, for $\rho \leq \rho_{N_n} = \sqrt{2N_n \ln (1+ \eta^2)}$,  $\eta = 1- \alpha-\beta$, 
we ensure that 
\[ \beta(\mathcal{F}_1[\rho]) \geq 1-\alpha -\eta =\beta. \]

\end{proof}

\textbf{Proof of Proposition \ref{Prop:lower_bound}}\label{Proof:lower_bound}
\begin{proposition}[$(\alpha,\beta)$ minimum radius]
Let $\beta \in (0, 1-\alpha_{\mu,n})$ and fix some window size $n \in \mathfrak{N}$
\[\theta(\alpha_{\mu,n}, \beta) = \sqrt{2 \log (1+4(1-\alpha_{\mu,n}-\beta)^2)}. \]
If
$\| \mu_{gap,n}\|^2 \leq \theta(\alpha_{\mu,n}, \beta) \sqrt{nd} \sigma^2 $
then $\mathds{P}(T_{\mu,n}(t) \geq 0 ) \leq 1- \beta$.
\end{proposition}
\begin{proof}
The proof of the result for the test statistics $T_{\mu,n}$ is based on analogous arguments assuming $\sigma^2=1$.   Since the distribution of the numerator and de-numerator are independent and $\chi^2$ distribution is a non-negative distribution, by Lemma \ref{lemma:simpleTn},  it is equivalent to consider the following distribution: 
For fix window size $n \in \mathfrak{N}$, we consider the test statistic
\[\hat{T}_{\mu, n} = \| W_{gap,n}\|^2 - \chi^2_{N_n}(\alpha_{\mu,n}). \]
By Proposition \ref{prop:Baraud}, for all $Y_i \in R^d$ such that
\begin{equation}\label{eq:Tn}
\| \mu^{gap}_{G_n}\|^2 \leq \theta(\alpha_{\mu,n}, \beta) \sqrt{nd},
\end{equation}
we then obtain $\mathds{P}(T_{\mu,n} \leq 0) \geq \beta.$

\end{proof}

\section{Extending the power of test to unknown distributions - Gaussian approximation}\label{Sec:nonGaussian}
\subsection{Graph-spanning ratio (GSR): ratio of quadratic forms}
For further generalization to GSR test statistics of mean and variance, we consider the graph-spanning ratio in the quadratic form of $\| W_{G_1}\|^2 / \|W_{G_2 }\|^2$, where $\| W_{G_1}\|^2$ and $\| W_{G_2}\|^2$ can be replaced with the quadratic terms used in GSRs for mean or variance. Recall that the graph-spanning for a graph $G_n$ is defined as
$\| W_{G_n}\|^2= \sum_{i,j \in G_n} \|Y_i-Y_j \|^2 I_{i,j}$. This represents the sum of the quadratic Euclidean distances between some of the nodes in the graph $V[n]$. Note that the graph $V[n]$ consists of $n$ i.i.d. observations, $Y_1, \ldots, Y_n$, where $Y_i$ follows an unknown distribution $\mathcal{F}$. For simplicity, assume $Y_1, \ldots, Y_n$ are i.i.d. centered random vectors in $\mathbb{R}^d$, with $Y_i= ( y_{i,1}, \ldots, y_{i,d})^T$. Define a pooled random vector $\mathbb{Y} \in \mathbb{R}^{n d}$ by concatenating the random vectors $Y_1, \ldots, Y_n$, that is
\begin{equation}
\label{eq:baseVector}
   \mathbb{Y} = ( y_{1,1}, \ldots, y_{1,d}, y_{2,1},\ldots,y_{2,d},\ldots, y_{n,1}, \dots, y_{n,d} )^T.
\end{equation}
Let $B_{i,j}$ be a $\mathbb{R}^{nd\times nd}$ matrix. We write
$Y_i -Y_j =  \mathbb{Y}^T B_{i,j} \mathbb{Y}$.

Let us also define a positive definite matrix $\mathbb{B} \in \mathbb{R}^{nd\times nd}$,
where $\mathbb{B}$ is the sum of $B_{i,j}$ times the connectivity indicator $I_{i,j}$ as
\begin{equation}
    \mathbb{B}= \sum_{i,j \in G_n}  B_{i,j} I_{i,j}.
\end{equation}
Then the quadratic graph-spanning can be decomposed in matrix form as
\begin{equation}
\label{eq:SumSquaredErrors}
\| W_{G_n}\|^2= \sum_{i,j \in G_n} \|Y_i-Y_j \|^2 I_{i,j} = \sum_{i,j \in G_n} \mathbb{Y}^T ( B_{i,j} I_{i,j}) \mathbb{Y}
=\mathbb{Y}^T \mathbb{B} \mathbb{Y}.
\end{equation}
Under the null hypothesis (i.e., no change point), the observations $Y_1, Y_2, \ldots, Y_n$ are i.i.d. from the same null distribution. Since the quadratic graph-spanning weights $\|W_{G_1}\|^2 = \sum_{i,j \in G_1} \|Y_i - Y_j\|^2 \cdot I_{i,j}$ and $\|W_{G_2}\|^2 = \sum_{i,j \in G_2} \|Y_i - Y_j\|^2 \cdot I_{i,j}$ are derived from the same observations $Y_1, Y_2, \ldots, Y_n$, they may be correlated depending on the graph types. However, because $\|W_{G_1}\|^2$ and $\|W_{G_2}\|^2$ are both based on the same observations, they can be expressed in quadratic forms of the same base vector $\mathbb{Y}$, as specified in Equation \ref{eq:baseVector}.
\begin{equation}
\|W_{G_1}\|^2 = \mathbb{Y}^T \mathbb{B}_1 \mathbb{Y}, \ \ \ \ \ \
\|W_{G_2}\|^2 = \mathbb{Y}^T \mathbb{B}_2 \mathbb{Y}.
\end{equation}
Equivalently, let us define Gaussian quadratic forms based on $ \mathbb{B}_1 $ and $ \mathbb{B}_2$.

Let $\gamma_1 \ldots \gamma_n$ be i.i.d. standard normal in $\mathbb{R}^d$ for $d \le \infty$. Denote $\gamma_i = (\gamma_{i,1}, \ldots, \gamma_{i,d})^T$.  Also let $\Gamma = (\gamma_{1,1}, \ldots, \gamma_{1,d},\ldots, \gamma_{n,1}, \ldots, \gamma_{n,d})^T$ and let $\mathbb{B} \in \mathbb{R}^{nd \times nd}$ be a positive definite matrix. Denote the Gaussian quadratic forms as $ ({\Gamma}/\sqrt{n})^T\mathbb{B}_1(\Gamma/\sqrt{n})$ and $({\Gamma}/\sqrt{n})^T\mathbb{B}_2(\Gamma/\sqrt{n})$.
\subsection{Sub-Gaussian condition}
Let us assume the following condition on the random vector $Y$: \\
\noindent \hypertarget{SubGaussian}{[\textcolor{blue}{Sub-Gaussian condition}]}
Let $Y \in \mathbb{R}^d$ satisfy $\mathbb{E}(Y) =0$. Let $\Var(Y) \le \mathbb{I}_d$. For some $C_{Y} >0$ and $g>0$, assume that the characteristic function of $Y$ is well defined and fulfills:
\begin{equation}
    |\log \mathbb{E}e^{i\langle  u, Y \rangle} | \le \frac{C_{Y}\|u\|^2}{2}, \ \ \ \ \ \ u \in \mathbb{R}^d, \  \|u\| < g,
\end{equation}
where $i=\sqrt{-1}$.
The sub-Gaussian condition states that the logarithm of the characteristic function is bounded on a ball. We have the following setting similar to \cite{spokoiny2023concentration} except for the characteristic function.
For $w \in \mathbb{R}^d $,  define a measure $\mathds{P}_w$ and its corresponding expectation $\mathbb{E}_w$. For any random variable $\eta$,
\[ \mathbb{E}_{\sqrt{i}\omega} (\eta) \myeq \frac{\mathbb{E}\left( \eta e^{\langle \sqrt{i}\omega,Y \rangle}\right)}{\mathbb{E} e^{\langle \sqrt{i}\omega, Y \rangle}}.
\]
Moreover, let us fix some $g > 0$ and define $\tau_3$, and $\tau_4$ as
\begin{align}\label{eq:tau3&4}
    &\tau_3 \myeq \sup_{\|w\| \le g} \sup_{u \in \mathbb{R}^p}  \frac{1}{\|u\|^3} |\mathbb{E}_{\sqrt{i}\omega} \langle \sqrt{i}u, Y-\mathbb{E}_{\sqrt{i}\omega} Y \rangle^3 |, \\
    &\tau_4 \myeq \sup_{\|w\| \le g} \sup_{u \in \mathbb{R}^p}  \frac{1}{\|u\|^4} |\mathbb{E}_{\sqrt{i}\omega} \langle \sqrt{i}u, Y-\mathbb{E}_{\sqrt{i}\omega} Y \rangle^4 - 3\{ \mathbb{E}_{\sqrt{i}\omega} \langle \sqrt{i}u, Y - \mathbb{E}_{\sqrt{i}\omega} Y \rangle^2\}^2|.
\end{align}
$\tau_3, \tau_4$ are typically very small and depend on the distribution of $Y$ and $g$.
For the approximation of the distribution, we consider the joint distribution of $\|W_{G_1}\|^2$, $\|W_{G_2}\|^2$. Let $i \Lambda_1, i \Lambda_2$ be complex numbers. The characteristic function of the joint distribution $\| W_{G_1}\|^2$ and $\| W_{G_2}\|^2$ is $\mathbb{E} \exp \left(i \Lambda_1 \|W_{G_1}\|^2 + i \Lambda_2 \|W_{G_2}\|^2  \right)$.

\begin{corollary}[Gaussian approximation of the graph-spanning ratio]
\label{corollary:F-approx}
  Let $Y_1, \ldots, Y_n$ $\in \mathbb{R}^d$ be centered i.i.d. random vectors satisfying $\mathbb{E} Y_i=0$ and $\Var(Y_i)\le \mathbb{I_d}$, and the \hyperlink{SubGaussian}{sub-Gaussian condition}. Let $\mathbb{Y}$ be defined as in Equation \ref{eq:baseVector}. Let $\|W_{G_1}\|^2 = \mathbb{Y}^T \mathbb{B}_1 \mathbb{Y}$ and $\|W_{G_1}\|^2 = \mathbb{Y}^T \mathbb{B}_2 \mathbb{Y}$, where $\mathbb{B}_1$, and $\mathbb{B}_2$ are positive definite matrices. Denote $\lambda_1 =\| \mathbb{B}_1 \|_{op}$, and $\lambda_2 =\| \mathbb{B}_2\|_{op}$. Assume $n \varrho^2 \ge 3 d$, where $\varrho$ is given in Equation \ref{eq:c3}, on page \pageref{eq:c3}. If $\Lambda_1 < n(2\lambda_1)^{-1}, \Lambda_2 < n(2\lambda_2)^{-1}$ satisfy $C_Y |\Lambda_1| < 1/6$ and $C_X |\Lambda_2| < 1/6$, then it holds
   \begin{equation}
      \mathbb{E} \exp \left( \frac{i \Lambda}{2n} \|W_{G_1}\|^2+ \frac{i \Lambda}{2n} \|W_{G_2}\|^2\right) \approx  \mathbb{E} \left[ \exp \left(\frac{i\Lambda_1}{2n} \Gamma^T \mathbb{B}_1 \Gamma + \frac{i\Lambda_2}{2n} \Gamma^T \mathbb{B}_2 \Gamma  \right)  \right], 
   \end{equation}
  under $d \gg 1$, and $ d^2 \ll n$.
\end{corollary}
See proof in Section \ref{Proof:F-approx} on page \pageref{Proof:F-approx}. Under the criteria of $d \gg 1$ and $d^2 \ll n$, the joint characteristic function of the sub-Gaussian quadratic forms approximates that of the Gaussian case. By contraction, the distribution of the sub-Gaussian quadratic ratio approximates that of the Gaussian case. This result can be applied to quadratic ratios that meet the above criteria. Hence, under the sub-Gaussian condition, $d \gg 1$ and $d^2 \ll n$, the distribution of the GSR test statistic approximates the distribution of the ratio of Gaussian quadratic forms, that is,
\begin{equation*}
    \frac{\|W_{G_1}\|^2}{\|W_{G_2}\|^2} \myApproxD \frac{\Gamma^T \mathbb{B}_1 \Gamma}{\Gamma^T \mathbb{B}_2 \Gamma}\myeqD \mathcal{F} \ \text{distribution}.
\end{equation*}
This result implies that if the constant $g$ in \hyperlink{SubGaussian}{sub-Gaussian condition} is sufficiently large and $d \gg 1$ and $d^2 \ll n$, the tail behavior of $\|W_{G_1}\|^2 / \|W_{G_2} \|^2$ is similar to that of the Gaussian case. When the criteria are satisfied, the theoretical results derived based on Gaussian data can be extended to observations of non-Gaussian data.

\subsection{Gaussian approximation of quadratic graph-spanning}\label{sec:GaussianApprox}

Our aim is to approximate the distribution of the quadratic graph-spanning $\| W_{G_n}\|^2$ which is consisted of $n$ i.i.d. random vectors to that of the Gaussian case in a non-asymptotic way. By the Laplace approximation from \cite{spokoiny2023concentration}, we can evaluate the approximation errors and find the conditions under which the distribution approximates the Gaussian case.

\subsubsection{Characteristic function of the Gaussian quadratic form}
For a later reference, let us first present the characteristic function of the sum of i.i.d. Gaussian quadratic forms. Define the imaginary part $i=\sqrt{-1}$, we state the following Lemma.
 
\begin{lemma}
\label{Lemma:GaussianSum}
Let $\gamma_1 \ldots \gamma_n$ be i.i.d. standard normal in $\mathbb{R}^d$ for $d \le \infty$. Denote $\gamma_i = (\gamma_{i,1}, \ldots, \gamma_{i,d})^T$.  Also, let $\Gamma = (\gamma_{1,1}, \ldots, \gamma_{1,d},\ldots, \gamma_{n,1}, \ldots, \gamma_{n,d})^T$ and let $\mathbb{B} \in \mathbb{R}^{nd \times nd}$ be a positive definite matrix. Denote the Gaussian quadratic form as $ ({\Gamma}/\sqrt{n})^T\mathbb{B}(\Gamma/\sqrt{n})$. For $\Lambda \in \mathbb{R}$, the characteristic function of the Gaussian quadratic form is
\begin{equation}
\label{eq:GaussianSumQuad}
    \mathbb{E}\exp{ \left( \frac{i \Lambda}{2}({\Gamma}/\sqrt{n})^T \mathbb{B} (\Gamma/\sqrt{n}) \right)}  = \det \left(\mathbb{I}_{nd} -  i \Lambda  \mathbb{B}/n\right)^{-1/2}.
\end{equation}
\end{lemma}
\begin{proof}
    The result can be obtained by applying Lemma A.2. from \cite{spokoiny2023concentration}.
\begin{align*}
     \mathbb{E}\exp{ \left( \frac{i \Lambda}{2}({\Gamma}/\sqrt{n})^T \mathbb{B} (\Gamma/\sqrt{n}) \right)}  
    = \mathbb{E}\exp{\left( \frac{i \Lambda}{2n}\langle \mathbb{B}\Gamma, \Gamma \rangle \right) }
    =\det \left(\mathbb{I}_{nd} -  i \Lambda  \mathbb{B}/n \right)^{-1/2}. 
\end{align*}
\end{proof}

\subsubsection{Characteristic function of the sub-Gaussian quadratic form}
We show that under the sub-Gaussian condition and certain criteria, the distributions of the sum of i.i.d. quadratic forms of random vectors are similar to that of the Gaussian case. We first approximate the characteristic function of the quadratic form for one sub-Gaussian random vector to the Gaussian case. The approximation error is derived based on the error bound of the local Laplace approximation of \cite{spokoiny2023concentration}. 

Let us assume the following condition on the random vector $X$:
\noindent \hypertarget{SubGaussianX}{[\textcolor{blue}{Sub-Gaussian condition}]}
Let $X \in \mathbb{R}^p$ satisfy $\mathbb{E}(X) =0$. Let $\Var(X) \le \mathbb{I}_p$. For some $C_{X} >0$ and $g>0$, assume that the characteristic function of $X$ is well defined and fulfills:
\begin{equation}
    |\log \mathbb{E}e^{i\langle  u, X \rangle} | \le \frac{C_{X}\|u\|^2}{2}, \ \ \ \ \ \ u \in \mathbb{R}^p, \  \|u\| < g.
\end{equation}

The sub-Gaussian condition states that the logarithm of the characteristic function is bounded on a ball. We have the following setting similar to \cite{spokoiny2023concentration} but instead for the characteristic function. Let $Q$ be a linear mapping $Q:\mathbb{R}^p \rightarrow \mathbb{R}^q$ and define
\begin{equation}
\label{eq:pQ}
 p_Q \myeq \frac{\mathbb{E} \| Q \gamma \|^2}{\|QQ^T \|} = \frac{tr\{ QQ^T\}}{\|QQ^T\|}= \frac{\|Q\|_{Fr}}{\|QQ^T\|}.   
\end{equation}
Define function $\phi_X(u)$ as
\[
\phi_{X}(u) \myeq \log e^{\sqrt{i} \langle u, X \rangle}.
\]
For $w \in \R $,  define a measure $\mathds{P}_w$ and its corresponding expectation $\mathbb{E}_w$. For any random variable $\eta$,
\[ \mathbb{E}_{\sqrt{i}\omega} (\eta) \myeq \frac{\mathbb{E}\left( \eta e^{\langle \sqrt{i}\omega,X \rangle}\right)}{\mathbb{E} e^{\langle \sqrt{i}\omega, X \rangle}}.
\]
Moreover, let us fix some $g > 0$ and define $\tau_3$, and $\tau_4$ as
\begin{align}\label{eq:tau3&4A}
    &\tau_3 \myeq \sup_{\|w\| \le g} \sup_{u \in \mathbb{R}^p}  \frac{1}{\|u\|^3} |\mathbb{E}_{\sqrt{i}\omega} \langle \sqrt{i}u, X-\mathbb{E}_{\sqrt{i}\omega} X \rangle^3 |, \\
    &\tau_4 \myeq \sup_{\|w\| \le g} \sup_{u \in \mathbb{R}^p}  \frac{1}{\|u\|^4} |\mathbb{E}_{\sqrt{i}\omega} \langle \sqrt{i}u, X-\mathbb{E}_{\sqrt{i}\omega} X \rangle^4 - 3\{ \mathbb{E}_{\sqrt{i}\omega} \langle \sqrt{i}u, X - \mathbb{E}_{\sqrt{i}\omega} X \rangle^2\}^2|.
\end{align}
$\tau_3, \tau_4$ are typically very small and they depend on the distribution of $X$ and $g$.
We now show  $ \mathbb{E} \{\exp{i \Lambda \| QX\|^2/2} \} \approx \det (\mathbb{I}_q - i \Lambda \mathbb{B})^{-1/2}$, where $\mathbb{B}=Q\Var(X)Q^T$.

Analogy to Proposition 4.1. in \cite{spokoiny2023concentration}, we can approximate the characteristic function of the sub-Gaussian quadratic form to that of the Gaussian case. 
\begin{theorem}
\label{theorem:CfSubGauss}[Approximate the c.f. of a quadratic form to Gaussian]
Let random vector $X \in \mathbb{R}^p$ satisfy $\mathbb{E}X=0$, $\Var(X) \le \mathbb{I}_p$ and the \hyperlink{SubGaussian}{sub-Gaussian condition}. For any linear mapping $Q:\mathbb{R}^p \rightarrow \mathbb{R}^q$. Define $\mathbb{B} = Q\Var(X)Q^T$, and $\lambda = \|B\|_{op}$. Also, let $g$ and $\tau_3$ of \ref{eq:tau3&4A} satisfy $g^2 \ge 3 p_Q$ and $g \tau_3 \le 2/3$. If $\Lambda \le \lambda^{-1}$ satisfies $C_{X} |\Lambda|\le 1/3$, then it holds
\begin{equation}
|\mathbb{E}\exp (i \Lambda\|QX\|^2/2) -\det(\mathbb{I}_q - i \Lambda \mathbb{B})^{-1/2}| \le (\diamondsuit +\rho_\Lambda) |\det(\mathbb{I}_q - i \Lambda \mathbb{B})^{-1/2}| + \S,
\end{equation}
for some $\diamondsuit$, $\rho_\Lambda$, and $\S$ giving explicitly in the proof.

Furthermore, under $p_Q \gg 1$, and $(\tau_3^2+\tau_4) p_Q^2 \ll 1$
\begin{equation}
\mathbb{E}\exp (i \Lambda\|QX\|^2/2) \approx \det(\mathbb{I}_q - i \Lambda \mathbb{B})^{-1/2}.
\end{equation}
\end{theorem}

\begin{proof}
Normalizing by $\|Q\|$ reduces the statement to $\|Q\|=1$ and $p_Q = tr(QQ^T)$ throughout the proof. Applying the local Laplace approximation from \cite{spokoiny2023concentration}, we approximate the characteristic function of the quadratic form to that of the Gaussian case, that is,
\[
\mathbb{E} e^{i\Lambda\|QX\|^2} \approx \det (\mathbb{I}_q-i\Lambda B)^{-1/2},
\]
where $\mathbb{B}=Q\Var(X)Q^T$.
And define a function 
\begin{equation}
   \phi_X(u) = \log \mathbb{E}\exp{\langle \sqrt{i} u,X\rangle}. 
\end{equation}
It holds $\phi_X(0)=0$, and $\nabla \phi_X(0)=0$. We follow the same steps as for proof Proposition 4.1 in \cite{spokoiny2023concentration} with $\phi_X(u)$ defined for the characteristic function as above.
Let $\gamma$ be standard Gaussian in $\mathbb{R}^d$ under $\mathbb{E}_r$ conditional on $X$. Denote $\mathbb{E}_\gamma=\mathbb{E}_{\gamma\approx \mathcal{N}(0,\mathbb{I}_p)}$. 
Then, we decompose the characteristic function as
\begin{align*}
\mathbb{E} \exp{i\Lambda \|QX\|^2/2} =& \mathbb{E}\mathbb{E}_{\gamma} \exp{\sqrt{i\Lambda}\langle Q^T \gamma, X\rangle}\\
=&\mathbb{E}_{\gamma} \mathbb{E} \exp{\sqrt{i}\sqrt{\Lambda}\langle Q^T\gamma,X\rangle}\\
=& \mathbb{E}_{\gamma}\exp{\phi_X(\sqrt{\Lambda}Q^T\gamma)} \mathbbm{1}{(\| \sqrt{\Lambda}Q^T \gamma\|\le g)}\\
& +\mathbb{E}_{\gamma}\exp{\phi_X(\sqrt{\Lambda}Q^T\gamma)} \mathbbm{1}{(\| \sqrt{\Lambda}Q^T \gamma\|> g)}.
\end{align*}
The approximation errors are from Part a: $\mathbb{E}_\gamma \exp{\phi_X(\sqrt{\Lambda}Q^T\gamma)} \mathbbm{1}{(\| \sqrt{\Lambda}Q^T \gamma\|\le g)}$, and Part b: $\mathbb{E}_\gamma \exp{\phi_X(\sqrt{\Lambda}Q^T\gamma)} \mathbbm{1}{(\| \sqrt{\Lambda}Q^T \gamma\| > g)}$.

With $\tau_3$, and $\tau_4$ defined in Equation \ref{eq:tau3&4A}, $\phi_X(u)$ satisfies the smoothness conditions:

\begin{equation}
    |\nabla^3\phi_X(x),u^{\otimes 3} | \le \tau_3 \|u\|^3, \ \ \ u \in  \mathbb{R}^p,
\end{equation}
and
\begin{equation}
    | \delta_4(u)| \myeq | \phi_X(u) - \frac{1}{2} \langle \phi_X^{''}(0),u^{\otimes 2}\rangle -\frac{1}{6}
 \langle \phi_X^{(3)}(0), u^{\otimes 3}\rangle | \le \frac{\tau_4}{24}\|u\|^4, \ \ \ \|u\| \le g. \end{equation}

\textbf{[Approximation error from Part a]} \\
Define $\mathcal{W}=\{ \w \in \mathbb{R}^p: \| \Lambda^{1/2} Qw \| \le g\}$.
Then with $\gamma \sim \mathcal{N}(0,\mathbb{I}_p)$
\[
\mathbb{E}_r \exp{\phi_X(\Lambda^{1/2}Q^T\gamma)\mathbbm{1}(\|\Lambda^{1/2}Q^T\gamma\| \le g) }= c_q \int_{\mathcal{W}} e^{f_\Lambda (w)} dw,
\]
where $c_q = (2\pi)^{-q/2} $ and for $w \in \mathbb{R}^q$
\[ 
f_{\Lambda}(w) = \phi(\Lambda^{1/2} Q^Tw) - \|w\|^2/2.
\]
$f_\Lambda(0)=0$, and $\nabla f_\Lambda(0)=0$. 
The function $f_\Lambda$ also satisfies the smoothness properties as for $\phi_X(\Lambda^{1/2}Q^Tw)$ with a factor, that is for any $w$ satisfying $\|\Lambda^{1/2} Q^Tw \| \le g$,
\begin{align*}
    &| \langle \nabla^3 f_\Lambda(w), u^{\otimes3}\rangle | \le \tau_3 \|\Lambda^{1/2} Q^T u\|^3,\\
    &| \langle \nabla^4 f_\Lambda(w), u^{\otimes4}\rangle | \le \tau_4 \|\Lambda^{1/2} Q^T u\|^4.
\end{align*}
W.l.o.g we assume $\| Q\|=1$, we define and evaluate the following quantities.
\begin{align*}
    &D^2_\Lambda \myeq - \nabla^2 f_\Lambda(0)= i \Lambda Q \Var(X) Q^T + \mathbb{I}_q = \mathbb{I}_q - i \Lambda B,\\
    &\mathbb{P}_\Lambda \myeq tr\{ D^{-2}_\Lambda (i\Lambda QQ^T) \} = i \Lambda tr(D^{-2}_\Lambda QQ^T) = \frac{i\Lambda}{1-i \Lambda} tr(QQ^T) = \frac{i\Lambda}{1-i \Lambda} p_Q,\\
    &\alpha_\Lambda \myeq \| D^{-1}_{\Lambda}(i\Lambda QQ^T) D^{-1}_\Lambda \| = i \Lambda \|D^{-2}_\Lambda  QQ^T\|.
\end{align*}
With $|\Lambda| \le \frac{1}{3C_X} <1/3$, and $\|B\| \le \|Q\| =1$, we can show that
\begin{equation}
\label{eq:p&alpha}
|\mathbb{P}_\Lambda| \le \sqrt{2} |\Lambda| p_Q \le (\sqrt{2}/3)p_Q , \ \ \ \  |\alpha_\Lambda| \le  \left|\frac{i \Lambda}{1+i \Lambda} \right| \le \sqrt{2}|\Lambda| \le \sqrt{2}/3.  
\end{equation}

The marginal difference is bounded.
\[
\left| \frac{\int_\mathcal{W} e^{f_\Lambda}(w)dw}{\int e^{-\| D_\Lambda w\|^2}dw} -1 \right| \le \diamondsuit + \rho_{\Lambda},
\]
where $\diamondsuit \le \frac{1}{2} (\sigma_G+\delta_{4,G)^2})^2+ \frac{5}{3} \epsilon^3_\Lambda \exp{(\epsilon^2_\Lambda)}$ by Proposition 3.1 in \cite{spokoiny2023concentration}.

We estimate $\diamondsuit$ by evaluating these quantities. Let $\mathcal{T}(u) = \langle \nabla^3 f_\Lambda(0), u^{\otimes 3} \rangle, \gamma_\Lambda \sim \mathcal{N}(0,D^{-2}_\Lambda)$. Under $g^2 = 3 p_Q$ and \ref{eq:p&alpha}, 

\begin{align*}
    &\epsilon_\Lambda = \frac{\tau_3 g^2 \sqrt{|\alpha_\Lambda|}}{2} \le \frac{12}{11}\tau_3p_Q,\\
    &\sigma^2_\Lambda =\mathbb{E} |\mathcal{T}^2(\gamma_\Lambda)| \le \sqrt{5/12} \tau_3 \mathbb{P}_{\Lambda} \le \frac{1}{3}\tau_3 p_Q, \\
    & \delta_{4,\Lambda} = \mathbb{E}_\mathcal{U} |\delta^2_4(\gamma_\Lambda)|\le \frac{1}{24} \tau_4 (\mathbb{P}_\Lambda + 3 |\alpha_\Lambda|)^2 \le \frac{1}{108} \tau_4(p_Q+3)^2. 
\end{align*}

\begin{align*}
    \rho_\Lambda 
    &\myeq \left|1- \frac{\int_\mathcal{W} e^{-\| D_\Lambda w\|^2}dw}{\int e^{-\| D_\Lambda w\|^2}dw}\right|
    =\mathds{P}(\|\sqrt{i\Lambda} Q^T D_\Lambda^{-1} \gamma\| >g)\\
    &=\mathds{P}(\|\sqrt{i\Lambda} Q^T D_\Lambda^{-1} \gamma\|^2 >g^2)
    =\mathds{P}\bigg(\| Q^T\gamma\|^2 > \sqrt{\frac{\Lambda(1+\Lambda^2)}{(1+\Lambda)^2}}g^2\bigg).
\end{align*}
We then evaluate $ \rho_\Lambda$ with $|\Lambda| \le 1/(3C_X) \le 1/3$, $g^2 \ge 3p_Q$, and by Theorem B.3 from \cite{spokoiny2023concentration}. 
\begin{align*}
    \rho_\Lambda =\mathds{P}(\|Q^T \gamma\|^2 > \sqrt{\frac{216}{5}} p_Q)
    \le \mathds{P}(\| Q^T\gamma\|^2 > 6 p_Q)
    \le e^{-p_Q/2}. 
\end{align*}
With the marginal error $\diamondsuit$ and $\rho_\Lambda$, we have 
\[
|\mathbb{E}e^{f_X(\sqrt{\Lambda}Q^T \gamma)} \mathbbm{1}(\|\sqrt{\Lambda}Q^T \gamma \| < g) -{det(\mathbb{I}_q+i\Lambda B)^{-1/2}} | \le (\diamondsuit+\rho_\Lambda){det(\mathbb{I}_q + i\Lambda B)^{-1/2}.}
\]

\textbf{[Approximation error from Part b]}\\
We now estimate the approximation error from 
$\mathbb{E}_\gamma [\exp{\phi_X(\sqrt{\Lambda}Q^T \gamma)} \mathbbm{1}(\| \sqrt{\Lambda}Q^T \gamma\| > g]$. 
Let us first show the proof for $\Lambda>0$. By the definition of the function $\phi_X(u)$ and  $\sqrt{i} = \frac{1}{\sqrt{2}} + \frac{i}{\sqrt{2}}$, we decompose $\phi_X(u)$ into real part and imaginary part.
\[\phi_X(u) = \log \mathbb{E}\exp(\sqrt{i}\langle u, X\rangle ) = \log \mathbb{E}\exp\bigg(\frac{1}{\sqrt{2}} \langle u, X\rangle  +\frac{i}{\sqrt{2}}  \langle u, X \rangle \bigg)
\]
We derive the following
\begin{align*}
 \mathbb{E}_\gamma & \bigg[ \exp{\phi_X(\sqrt{\Lambda}Q^T \gamma)} \cdot \mathbbm{1}(\| \sqrt{\Lambda}Q^T \gamma\| > g ) \bigg] \\
 &=  \mathbb{E}_\gamma \mathbb{E}\bigg[ \exp \bigg(\frac{1}{\sqrt{2}} \sqrt{\Lambda} \langle Q^T \gamma, X \rangle +\frac{i}{\sqrt{2}} \sqrt{\Lambda}   \langle Q^T \gamma, X \rangle \bigg)\cdot \mathbbm{1}(\| \sqrt{\Lambda}Q^T \gamma\| > g ) \bigg].\\ 
\end{align*}
Taking absolute value of the approximate error. Since $\big|\exp{(\frac{i}{\sqrt{2}} \langle u, X \rangle)} \big| = 1 $, we then obtain
\begin{align*}
 \bigg|\mathbb{E}_\gamma & \big[ \exp{\phi_X(\sqrt{\Lambda}Q^T \gamma)} \cdot \mathbbm{1}(\| \sqrt{\Lambda}Q^T) \gamma\| > g ) \big]\bigg| \\
 &= \left| \mathbb{E}_\gamma \mathbb{E}\bigg[ \exp \bigg(\frac{1}{\sqrt{2}} \sqrt{\Lambda} \langle Q^T \gamma, X \rangle +\frac{i}{\sqrt{2}} \sqrt{\Lambda}   \langle Q^T \gamma, X \rangle \bigg)\cdot \mathbbm{1}(\| \sqrt{\Lambda}Q^T \gamma\| > g ) \bigg] \right|\\
  & \le  \mathbb{E}_\gamma \mathbb{E}\bigg[ \exp \bigg(\frac{1}{\sqrt{2}} \sqrt{\Lambda} \langle Q^T \gamma, X \rangle  \bigg) \left| \exp \bigg(\frac{i}{\sqrt{2}} \sqrt{\Lambda}   \langle Q^T \gamma, X \rangle \bigg) \right| \cdot \mathbbm{1}(\| \sqrt{\Lambda}Q^T \gamma\| > g ) \bigg] \\
  & =  \mathbb{E}_\gamma \mathbb{E}\bigg[ \exp \bigg(\frac{1}{\sqrt{2}} \sqrt{\Lambda} \langle Q^T \gamma, X \rangle  \bigg) \cdot \mathbbm{1}(\| \sqrt{\Lambda}Q^T \gamma\| > g ) \bigg]. \\
\end{align*}

Following Proposition 4.1 in \cite{spokoiny2023concentration}, we can show that the approximation error from Part b is small, that is with $|\Lambda| \le 1/3$,

\begin{align*}
\mathbb{E}_\gamma &\mathbb{E}\bigg[ \exp \bigg(\frac{1}{\sqrt{2}} \sqrt{\Lambda} \langle Q^T \gamma, X \rangle  \bigg) \cdot \mathbbm{1}(\| \sqrt{\Lambda}Q^T \gamma\| > g ) \bigg] \\
&\le \mathbb{E}_\gamma \bigg[ \exp\bigg( \frac{1}{\sqrt{2}}C_X\Lambda \|Q^T \gamma \|^2\bigg) \cdot \mathbbm{1}(\| Q^T \gamma\|^2 > \Lambda^{-1} g^2 ) \bigg]\\
&\le \mathbb{E}_\gamma \bigg[ \exp\bigg( C_X\Lambda \|Q^T \gamma \|^2\bigg) \cdot \mathbbm{1}(\| Q^T \gamma\|^2 > \Lambda^{-1} g^2 ) \bigg]\\
&\le \frac{1}{1-\omega_\Lambda}\exp \bigg( C_X \Lambda p_Q/2 -(1-\omega_\Lambda)\mathfrak{Z}_\Lambda\bigg),
\end{align*}
where 
\begin{align*}
\mathfrak{Z}_\Lambda \myeq \frac{1}{4} ( \sqrt{C_X^{-1} \Lambda^{-1} g^2}-\sqrt{p_Q})^2, \ \ \ \omega_\Lambda \myeq C_X\Lambda+ C_X\Lambda\sqrt{p_Q/\mathfrak{Z}_\Lambda}. 
\end{align*}
As shown in Proposition 4.1 in \cite{spokoiny2023concentration}, $|\Lambda| \le 1/3$ ensures that $\mathfrak{Z}_\Lambda \ge (\sqrt{9p_Q}-\sqrt{p_Q})^2/4=p_Q$ and $\omega_\Lambda \le 2/3$ so that the above quantity is small as $p_Q$ large, that is
\[
\S \myeq \frac{1}{1-\omega_\Lambda}\exp \bigg( C_X \Lambda p_Q/2 -(1-\omega_\Lambda)\mathfrak{Z}_\Lambda\bigg)  \le 3 e^{-p_Q/6}.
\]

For $\Lambda \le 0$, similarly we use $\sqrt{-i} = \frac{1}{\sqrt{2}} - \frac{i}{\sqrt{2}}$, the result then follows.

In summary, with the approximation error $\diamondsuit$, and $\rho_\Lambda$ from Part a and $\S$ from Part b, we obtained,
\begin{equation}
\left|\mathbb{E}\exp (i \Lambda\|QX\|^2/2) -\det(\mathbb{I}_q - i \Lambda \mathbb{B})^{-1/2} \right|  \le (\diamondsuit +\rho_\Lambda) \left|\det(\mathbb{I}_q - i \Lambda \mathbb{B})^{-1/2} \right| + \S,
\end{equation}
The approximation error from Part a and Part b diminished as $p_Q \gg 1$ and $(\tau_3^2+\tau_4)p_Q \ll 1$ such that
\begin{equation}
\mathbb{E}\exp (i \Lambda\|QX\|^2/2) \approx \det(\mathbb{I}_q - i \Lambda \mathbb{B})^{-1/2}.
\end{equation}
\end{proof}

We are now ready to apply the above theorem to the sum of i.i.d. squared random vectors. Let $Y_1, \ldots, Y_n \in \mathbb{R}^d$ be centered i.i.d. random vectors. We assume each $Y_i$ to satisfy the sub-Gaussian condition and specify constants $c_3$ and $c_4$ as the following:
\begin{itemize}
    \item Major condition: $Y_i$ satisfies \hyperlink{SubGaussianX}{sub-Gaussian condition}, that is, $\mathbb{E}Y =0, \Var(Y) \le \mathbb{I}_d$, and the logarithm of the characteristic function $\log \mathbb{E} \langle u, Y_1\rangle $ is finite and satisfies that for some $C_Y$:
    \begin{equation}\label{eq:Cx}
    |\log \mathbb{E} \exp^{i\langle u, Y_1\rangle} | \le \frac{C_Y \| u||^2}{2}, \ \ u \in \mathbb{R}^p.        
    \end{equation}
    \item $c_3$, $c_4$ : For $\varrho>0 $ and some constant $c_3$ and $c_4$, it holds with $\mathbb{E}_\omega$,
    \begin{align}
    \label{eq:c3}
        &\sup_{\|w\| \le \varrho} \sup_{u \in \mathbb{R}^d}  \frac{1}{\|u\|^3} |\mathbb{E}_{\sqrt{i}\omega} \langle \sqrt{i} u, Y_1-\mathbb{E}_{\sqrt{i}\omega} Y_1 \rangle^3 | \le c_3.\\
        &\sup_{\|w\| \le \varrho} \sup_{u \in \mathbb{R}^d}   \frac{1}{\|u\|^4} |\mathbb{E}_\omega \langle \sqrt{i} u, Y_1-\mathbb{E}_{\sqrt{i}\omega} Y_1 \rangle^4 - 3\{ \mathbb{E}_{\sqrt{i}\omega} \langle \sqrt{i} u, Y_1 - \mathbb{E}_{\sqrt{i}\omega} Y_1 \rangle^2\}^2| \le c_4.
    \end{align}
\end{itemize}
We now state the result for the sum of i.i.d. squared random vectors.
\begin{theorem}
\label{Theorem:nd2criteria}
   Let $Y_1, \ldots, Y_n$ be i.i.d. in $\mathbb{R}^d$ satisfying $\mathbb{E} Y_i = 0$, and $\Var(Y_i) \le \mathbb{I}_d$ and the \hyperlink{SubGaussianX}{sub-Gaussian condition}. Denote $Y_i= ( y_{i,1}, \ldots, y_{i,d})^T$. Let 
   \[\mathbb{Y} = ( y_{1,1}, \ldots, y_{1,d}, y_{2,1},\ldots,y_{2,d},\ldots, y_{n,1}, \ldots, y_{n,d} )^T.\]
   Also let $\mathbb{B}$ be positive definite matrix and $\lambda = \|B\|_{op}$. Assume $n \varrho^2 \ge 3 p_Q$. For any $\Lambda < n\lambda^{-1}$ that satisfies $ C_{Y} |\Lambda|\le 1/3$, where $C_{Y}$ as in Equation \ref{eq:Cx}, it holds
      \begin{align*}
       \Bigl|\mathbb{E} \exp \Bigl( \frac{i \Lambda}{2} (\mathbb{Y}/\sqrt{n})^T \mathbb{B}  (\mathbb{Y}/\sqrt{n} ) \Bigr) &-\det \left(  \mathbb{I}_{nd}  - i \Lambda \mathbb{B}/n \right)^{-1/2} \Bigr| \\
       &\le (\diamondsuit +\rho_\Lambda) \Bigl| \det \left(  \mathbb{I}_{nd}  - i \Lambda \mathbb{B}/n \right)^{-1/2} \Bigr| +\S,   
      \end{align*}
   for some $\diamondsuit$, $\rho_\Lambda$ and $\S$. Furthermore, under $d \gg 1$, and $ {p_Q}^2 \ll n$,
   \begin{equation}
   \label{eq:ExpMomSubG}
       \mathbb{E} \exp \left( \frac{i \Lambda}{2} (\mathbb{Y}/\sqrt{n})^T \mathbb{B}  (\mathbb{Y}/\sqrt{n} ) \right)\approx \det \left(  \mathbb{I}_{nd}  - i \Lambda \mathbb{B}/n \right)^{-1/2}. 
   \end{equation}
  $p_Q$ is given in Equation \ref{eq:pQ} with $QQ^T=\mathbb{B}$.
\end{theorem}

\begin{proof}
We apply Theorem \ref{theorem:CfSubGauss} to approximate the characteristic function using the Laplace approximation. 
Recall that function $\phi_{Y}(u) \myeq \log \mathbb{E}e^{\sqrt{i}\langle u, Y_1 \rangle}$. Random vector $\mathbb{Y}$ is structured by concatenating $Y_1,\ldots, Y_n$, $Y_i \in \mathbb{R}^d$. Setting $v = (u_1^T \frown u_2 \frown  \ldots \frown  u_n^T)^T, u_i \in \mathbb{R}^d$, $ v \in \mathbb{R}^{nd}$, let us define the function 

\[\phi_{\mathbb{Y}}(v)= \log \mathbb{E} e^{\sqrt{i}\langle v, \mathbb{Y} \rangle} = n \log \mathbb{E}e^{\sqrt{i}\langle u, Y_1 \rangle} = n \phi_Y(u).\]
     
Similarly to Theorem \ref{theorem:CfSubGauss}, the approximation error of the Laplace approximation can be evaluated by taking the derivative of the function $\phi_{\mathbb{Y}}(v)$.
     
$k^{th}$ derivative of the function $\phi_\mathbb{Y}(v)$:
\begin{equation}
\label{eq:fk}
  \phi_\mathbb{Y}^{(k)} (v / \sqrt{n}) = n^{1-k/2} \phi_Y^{(k)} (u/\sqrt{n}).  
\end{equation}
Based on Equation \ref{eq:fk}, we obtain $\tau_3$ and $\tau_4$ (Equation \ref{eq:tau3&4A}) related to $c_3$ and $c_4$. For any $g $ with $g/\sqrt{n} \le \varrho$, we have
\[ 
\tau_3=n^{-1/2}c_3,  \ \ \ \ \ \ \tau_4 = n^{-1} c_4.
\]

By Theorem \ref{theorem:CfSubGauss}, $\rho_\Lambda \ll 1$ and $\S \ll 1$ under $p_Q \gg 1$, and $(\tau_3^2+\tau_4) p_Q^2 \ll 1$. Here, $p_Q$ depends on $\mathbb{B}$. Thus, $\diamondsuit \ll 1$ under $(c_3^2+c_4) n^{-1}  {p_Q}^2 \ll 1 $. Recall that $c_3$ and $c_4$ depend on $g$ and are usually small. For high-dimensional data, $p_Q \gg 1$.  Thus, the criterion ${p_Q}^2 \ll n $ ensures that the approximation error is $\ll 1$ such that Equation \ref{eq:ExpMomSubG} follows.  
\end{proof}

Comparing Equation \ref{eq:GaussianSumQuad} and Equation \ref{eq:ExpMomSubG}, it shows that the characteristic function of a sum of i.i.d squared random vector approaches that of a Gaussian case under the sub-Gaussian condition and the criteria of $p_Q \gg 1$ and ${p_Q}^2 \ll n$. Apply Theorem \ref{Theorem:nd2criteria} to the quadratic graph-spanning. Since $p_Q$ depends on $\mathbb{B}$, $p_Q$ is determined by the connectivity of the graph. For verification, we now check the $p_Q$ for the sum of i.i.d. quadratic norms and for different graph types. 
\begin{corollary}\label{corolary:Yn}
    Let $Y_1, \ldots,Y_n$ be i.i.d. in $\mathbb{R}^d$ satisfying $\mathbb{E}Y_i = 0$, $\Var(Y_i) \le \mathbb{I}_d $ and the \hyperlink{SubGaussianX}{sub-Gaussian condition}. Assume $n \varrho^2 \ge 3 d$. If $\Lambda < n$ satisfies $ C_{Y} |\Lambda| \le 1/3$, where $C_{Y}$ as in Equation \ref{eq:Cx}.
    It holds
    \begin{equation}
    \label{eq:AvgSubGaussianSumQuad}
     \mathbb{E}\Bigl(\exp{  \frac{i \Lambda}{2} \frac{1}{n} \sum_{j=1}^n \| Y_j \|^2}  \Bigr)
    \approx \det \Bigl( \Bigl( 1-  i  \frac{\Lambda}{n} \Bigr)  \mathbb{I}_{nd} \Bigr)^{-1/2},
\end{equation}
under $d \gg 1$ and $d^2 \ll n$.
\end{corollary}

\begin{proof}
Let $\mathbb{Y}$ be defined as in \ref{eq:baseVector}. The characteristic function is then
\[
\mathbb{E}\Bigl(\exp{  \frac{i \Lambda}{2} \frac{1}{n} \sum_{j=1}^n \| Y_j \|^2}  \Bigr)=\mathbb{E}\Bigl( \exp{\frac{i \Lambda}{2n} \mathbb{Y}^T \mathbb{B} \mathbb{Y}}\Bigr),
\]
where $\mathbb{B}=\mathbb{I}_{nd}$. Apply Theorem \ref{Theorem:nd2criteria}, we have 
\begin{equation*}
 p_Q =  \frac{ \mathbb{E}\langle \mathbb{B}\gamma /\sqrt{n}, \gamma/
 \sqrt{n} \rangle}{\| \mathbb{B} \|} = \frac{tr\{ \mathbb{B}\}}{n \|\mathbb{B}\|} =\frac{nd}{n}=d.
\end{equation*}
Thus, under a sub-Gaussian condition, the characteristic function of $(1/n) \mathbb{Y}^T \mathbb{Y}$ is similar to that of the Gaussian case with the criteria of $d \gg 1$ and $d^2 \ll n$. The result is consistent with Theorem 2.4 in \cite{spokoiny2023concentration}.
\end{proof}

We now extend the result to the quadratic graph-spanning $\|W_{G_n} \|^2$ for various graph types, including the complete graph and the MST graph. 
\begin{corollary}\label{corollary:WGn}
    Let $\|W_{G_n}\|^2$ be specified as Equation \ref{eq:SumSquaredErrors}, where $Y_1, \ldots, Y_n$ are i.i.d. in $\mathbb{R}^d$ satisfying $\mathbb{E}Y_i$, $\Var(Y_i)\le \mathbb{I}_d$ and the \hyperlink{SubGaussianX}{sub-Gaussian condition}. Assume $n \varrho^2 \ge 3 d$. $\lambda = \| \mathbb{B}\|_{op}$ and $\lambda = \|B\|_{op}$. Assume $n \varrho^2 \ge 3 p_Q$. For any $\Lambda < n\lambda^{-1}$ that satisfies $ C_{Y} |\Lambda|\le 1/3$, where $C_{Y}$ as in Equation \ref{eq:Cx}, it holds
\begin{equation}
\mathbb{E} \exp \left( \frac{i \Lambda}{2} \|W_{G_n}\|^2/n \right) \approx  \det(\mathbb{I}_{nd}- i \Lambda \mathbb{B}/n ) )^{-1/2},
\end{equation}
under $d \gg 1$, and $ d^2 \ll n$.
\end{corollary}

\begin{proof}
    The graph-spanning distance can be decomposed to a quadratic form consisted of n i.i.d. squared random vectors, as stated in Equation \ref{eq:SumSquaredErrors}
\begin{equation*}
\| W_{G_n}\|^2= \mathbb{Y}^T \mathbb{B} \mathbb{Y}.
\end{equation*}
    Thus, we can apply the result of Theorem \ref{Theorem:nd2criteria}. Regardless of the type of graph, it has the same base vector $\mathbb{Y}$ as stated in Equation \ref{eq:SumSquaredErrors}. Thus, $\tau_3$ and $\tau_4$ are the same for all types of graph. The remaining is to find $p_Q$ that depends on $\mathbb{B}$.
For complete graph, w.l.o.g., let d=1, then
   \begin{equation*}
       \mathbb{B}= \begin{bmatrix}
(n-1) & -1 & -1 & \ldots & -1 \\
 -1 & (n-1) & -1 & \ldots & -1 \\
 \vdots & \vdots & \ddots &  & \vdots \\
  -1 & -1 & \ldots & (n-1) &-1   \\
 -1 & -1 & \ldots & -1 & (n-1)
\end{bmatrix}.
   \end{equation*} 
Thus, for all $d$
\begin{equation*}
\label{eq:pQ_CG}
 p_Q =  \frac{ \mathbb{E}\langle \mathbb{B}\gamma /\sqrt{n}, \gamma/
 \sqrt{n} \rangle}{\| \mathbb{B} \|} = \frac{tr(\mathbb{B})}{n \|\mathbb{B}\|} =\frac{n(n-1)d}{n\sqrt{n^2+n-1}} \approx d.
\end{equation*}

Similarly, for the MST graph, \begin{equation*}
\label{eq:pQ_MST}
 p_Q =  \frac{ \mathbb{E}\langle \mathbb{B}\gamma /\sqrt{n}, \gamma/
 \sqrt{n} \rangle}{\| \mathbb{B} \|} = \frac{tr(\mathbb{B})}{n \|\mathbb{B}\|}  \approx d.
\end{equation*}   
Apply Theorem \ref{Theorem:nd2criteria}, with $d \gg 1$, and $ d^2 \ll n$, it yields $\diamondsuit \ll 1$, $\rho_\Lambda \ll 1$, $\S \ll 1$ such that
\[   
\mathbb{E} \exp \left( \frac{i \Lambda}{2} \|W_{G_n}\|^2/n \right) \approx  \det(\mathbb{I}_{nd}- i \Lambda \mathbb{B}/n ) )^{-1/2}.
\] 
\end{proof}
We can then conclude that under the sub-Gaussian condition, with the criteria $d \gg 1 $ and $d^2 \ll n$, the distribution of the quadratic graph-spanning distance $\| W_{G_n}\|^2$ is approximately the same as that of the Gaussian quadratic form, that is, $\| W_{G_n}\|^2 \myApproxD  \langle \mathbb{B} \Gamma,\Gamma \rangle$, where $\|W_{G_n} \|^2$ defined as in Corollary \ref{corollary:WGn}. 

\subsection{Gaussian approximation of the quadratic graph-spanning ratio}
 We now study the distribution of the ratio of the quadratic graph-spanning $\|W_{G_1}\|^2/\|W_{G_2}\|^2$ applying the result of the Laplace approximation from the last section. Our aim is to show that the ratio of two quadratic forms of general random vectors $\|W_{G_1}\|^2/\|W_{G_2}\|^2$ follows a distribution similar to that of the Gaussian quadratic forms. To such an extent that the theoretical results derived based on the Gaussian distribution can be extended to general vectors, that is 

\begin{equation}
    \label{eq:distSimilar}
    \frac{\|W_{G_1}\|^2}{\|W_{G_2}\|^2} \myApproxD \frac{\langle \mathbb{B}_1 \Gamma,\Gamma \rangle}{\langle \mathbb{B}_2 \Gamma,\Gamma \rangle}.
\end{equation}



Under the null hypothesis (i.e., no change point), the observations
$Y_1, Y_2, \ldots, Y_n$ are i.i.d. from the same null distributions. Since the quadratic graph-spanning $\|W_{G_1}\|^2 = \sum_{i,j \in G_1} \|Y_i - Y_j\|^2 \cdot I_{i,j}$ and $\|W_{G_2}\|^2 = \sum_{i,j \in G_2} \|Y_i - Y_j\|^2 \cdot I_{i,j}$ is derived from the same observations $Y_1, Y_2, \ldots Y_n$, they could be correlated depending on the graph types. However, because $\|W_{G_1}\|^2$ and $\|W_{G_2}\|^2$ are based on the same observations, we can represent $\|W_{G_1}\|^2, \|W_{G_1}\|^2$ in quadratic forms of the same base vector $\mathbb{Y}$ as specified in Equation \ref{eq:baseVector}.
\begin{equation}
\|W_{G_1}\|^2 = \mathbb{Y}^T \mathbb{B}_1 \mathbb{Y}, \ \ \ \ \ \
\|W_{G_2}\|^2 = \mathbb{Y}^T \mathbb{B}_2 \mathbb{Y}.
\end{equation}
In order to show Equation \ref{eq:distSimilar}, we consider the joint distribution of $\|W_{G_1}\|^2, \|W_{G_2}\|^2$. Let $i \Lambda_1, i \Lambda_2$ be complex numbers, where $i=\sqrt{-1}$. The characteristic function of the joint distribution $\| W_{G_1}\|^2$ and $\| W_{G_1}\|^2$ is
\[ \mathbb{E} \exp \left(i \Lambda_1 \|W_{G_1}\|^2 + i \Lambda_2 \|W_{G_2}\|^2  \right).
\]
The objective now is to show that the joint characteristic function of squared norms of random vectors is similar to that of Gaussian quadratic forms, that is, \[ \mathbb{E} \left[ \exp \left(i \Lambda_1 \|W_{G_1}\|^2 + i \Lambda_2 \|W_{G_2}\|^2  \right)  \right] \approx  \mathbb{E} \left[ \exp \left(i \Lambda_1 \langle \mathbb{B}_1 \Gamma,\Gamma \rangle + i \Lambda_2 \langle \mathbb{B}_2 \Gamma,\Gamma \rangle \right)  \right].  \] In the following section, we first state the joint characteristic function of the Gaussian quadratic forms. Then we derive a joint characteristic function for quadratic norms of random vectors under the sub-Gaussian condition. We conclude that the joint characteristic function of the sub-Gaussian case is similar to that of the Gaussian case. 

\subsubsection{Joint characteristic function of Gaussian quadratic forms}\label{Sec:ExpMomentGaussian}
Let $\gamma_1,\ldots,\gamma_n$ be the i.i.d. standard normal in $\mathbb{R}^d$ for $d<\infty$.
Denote $\gamma_i = (\gamma_{i,1}, \ldots, \gamma_{i,d})^T$.  Also, let $\Gamma = (\gamma_{1,1}, \ldots, \gamma_{1,d},\ldots, \gamma_{n,1}, \ldots, \gamma_{n,d})^T$ and let $\mathbb{B}_1$ and $\mathbb{B}_2 \in \mathbb{R}^{nd \times nd}$ be a positive definite matrix. Denote two Gaussian quadratic forms as
\[
\Gamma^T\mathbb{B}_1\Gamma, \ \ \ \ \ \ \Gamma^T\mathbb{B}_2\Gamma.
\]

The joint distribution of the above Gaussian quadratic forms is stated as follows.

\begin{lemma}[Extension of Lemma A.2 in \cite{spokoiny2023concentration}]\label{lemma:ExtA.2}
  Let $\Gamma^T\mathbb{B}_1\Gamma$, $\Gamma^T\mathbb{B}_2\Gamma$ be the Gaussian quadratic forms stated above. And let $\mathbb{B}_1$ and $\mathbb{B}_2$ be positive definite matrices. For $\Lambda_1$, $\Lambda_2 \in \mathbb{R}$, it holds
\begin{equation}\label{eq:GaussianFormＡ}
  \mathbb{E} \exp \left(\frac{i \Lambda_1}{2n} \Gamma^T\mathbb{B}_1\Gamma + \frac{i \Lambda_2}{2n} \Gamma^T\mathbb{B}_2 \Gamma \right) \\
  = det \left(\mathbb{I}_{nd}-i \Lambda_1 \mathbb{B}_1/n -i \Lambda_1 \mathbb{B}_2/n \right) ^{-1/2}.
\end{equation}
\end{lemma}

\begin{proof}
Let us define $\mathbb{B} \myeq \left({ \frac{\Lambda_1}{\Lambda} \mathbb{B}_1 + \frac{\Lambda_2}{\Lambda}\mathbb{B}_2 } \right)$, where $\Lambda = |\Lambda_1| +|\Lambda_2|$. We apply Lemma \ref{Lemma:GaussianSum}. 
\begin{align*}
   \mathbb{E} \exp \left(\frac{i \Lambda_1}{2n} \Gamma^T\mathbb{B}_1\Gamma + \frac{i \Lambda_2}{2n} \Gamma^T\mathbb{B}_2 \Gamma \right)
   &= \mathbb{E} \exp  \left( \frac{i \Lambda}{2n} \Gamma^T \mathbb{B} \Gamma \right)
   = det \left(\mathbb{I}_{nd}-i \Lambda_1 \mathbb{B}/n\right) ^{-1/2}\\
   &= det \left(\mathbb{I}_{nd}-i \Lambda_1 \mathbb{B}_1/n -i \Lambda_1 \mathbb{B}_2/n \right) ^{-1/2},
\end{align*}
for $\Lambda_1$, $\Lambda_2 \in \mathbb{R}$.
\end{proof}
\subsubsection{Joint characteristic function of sub-Gaussian quadratic forms}
In this section, we show that under the sub-Gaussian condition, the joint characteristic function of the general quadratic forms is similar to that of the Gaussian case, as shown in Section \ref{Sec:ExpMomentGaussian}.
\begin{theorem}\label{Theorem:2CfSubGauss}[Extension of Theorem \ref{theorem:CfSubGauss}]
Let random vector $Y \in \mathbb{R}^p$ satisfy $\mathbb{E}(Y)=0$, $\Var(Y) \le \mathbb{I}_d$, and the \hyperlink{SubGaussianX}{sub-Gaussian condition}. Let $g$, $\tau_3$, and $\tau_4$ from \ref{eq:tau3&4A} defined for $X$ satisfy $g \tau_3 \le 2/3$. And let $\mathbb{B}_1$ and $\mathbb{B}_2$ be positive definite matrices, $\lambda_1 = \|B_1\|_{op} $, and $\lambda_2 = \|B_2\|_{op} $. Also, let $\Lambda = |\Lambda_1| + |\Lambda_2|$, define 
\[
B \myeq \left({ \frac{\Lambda_1}{\Lambda}} B_1 + \frac{\Lambda_2}{\Lambda}B_2  \right).
\]
Assume $g^2 \ge 3 p_Q$, where $p_Q$ is defined in \ref{eq:pQ} with $QQ^T = B$. 
If $\Lambda_1 < (2\lambda_1)^{-1}, \Lambda_2 < (2\lambda_2)^{-1}$ satisfy $C_X |\Lambda_1| < 1/6$ and $C_X |\Lambda_2| < 1/6$, it holds 
\begin{align*}
      \Bigl| \mathbb{E} &\exp \Bigl( \frac{i \Lambda_1}{2} Y^T B_1 Y + \frac{i \Lambda_2}{2} Y^T B_2 Y \Bigr) -\det  \left(\mathbb{I}_{d}  - i \Lambda_1 B_1 -i \Lambda_2 B_2 \right)^{-1/2}\Bigr| \\
       &\le (\diamondsuit + \rho_\Lambda) \Bigl|\det  \left(\mathbb{I}_{d}  - i \Lambda_1 B_1 -i \Lambda_2 B_2 \right)^{-1/2}\Bigr| + \S, 
\end{align*}
for some $\diamondsuit$, $\rho_\Lambda$ and $\S$. Furthermore, under $p_Q \gg 1$, and $(\tau_3^2+\tau_4) p_Q^2 \ll 1$.
\begin{equation}
\label{eq:CfSumSubGaussianForm}
   \mathbb{E} \exp \left( \frac{i \Lambda_1}{2} Y^T B_1 Y + \frac{i \Lambda_2}{2} Y^T B_2 Y \right) \approx \det  \left(\mathbb{I}_{d}  - i \Lambda_1 B_1 -i \Lambda_2 B_2 \right)^{-1/2}.
\end{equation}
\end{theorem}

\begin{proof}
We represent the quadratic form in terms of $\Lambda$ and $B$,
    \begin{align*}
    \frac{\Lambda_1}{2} Y^T B_1 Y   + \frac{\Lambda_2}{2} Y^T B_2 Y  = Y^T( \Lambda_1 B_1 + \Lambda_2 B_2) Y = \Lambda Y^T (B) Y,
    \end{align*}

For any $\Lambda_1 < (2 \lambda_1)^{-1}$ and $\Lambda_2 < (2 \lambda_2)^{-1}$, $\Lambda \|B\|_{op} <1$ such that the integral in the characteristic function converges. With $\Lambda = |\Lambda_1| + |\Lambda_2| < (3C_X)^{-1}$, we apply Theorem \ref{theorem:CfSubGauss} by analogy $\| QX\|^2 = \langle B Y, Y \rangle$, with $B$ replacing $(QQ^T)$. It holds for  $\Lambda_1 < (2 \lambda_1)^{-1}$ and $\Lambda_2 < (2 \lambda_2)^{-1}$ that 
    \begin{align*}
    \mathbb{E}&\exp \left( \frac{i \Lambda_1}{2} Y^T B_1 Y + \frac{i \Lambda_2}{2} Y^T B_2 Y \right)
    = \mathbb{E} \exp \left( \frac{i \Lambda_1}{2}  \Lambda Y^T (B) Y \right)\\
    &\approx \det  \left(\mathbb{I}_{d}  - i \Lambda B \right)^{-1/2} \\
    &\approx \det  \left(\mathbb{I}_{d}  - i \Lambda_1 B_1 -i \Lambda_2 B_2 \right)^{-1/2}, 
    \end{align*}
under $p_Q \gg 1$, and $(\tau_3^2+\tau_4) p_Q^2 \ll 1$.
\end{proof}
Now we extend the result to random vectors $\mathbb{Y}$. 

\begin{corollary}\label{Corollary:ndcriteria2}[Extension of Theorem \ref{Theorem:nd2criteria}]
Let $Y_1, \ldots, Y_n \in \mathbb{R}^d$ be centered i.i.d. random vectors satisfying $\mathbb{E} Y_i=0$ and $\Var(Y_i)\le \mathbb{I}_d$, and the \hyperlink{SubGaussianX}{sub-Gaussian condition}. Let $\mathbb{Y}$ be defined as in Equation \ref{eq:baseVector}. Let $g \tau_3 \le 2/3$. Given positive definite matrices $\mathbb{B}_1$, $\mathbb{B}_2$, let $\lambda_1 =\| \mathbb{B}_1 \|_{op}$, and $\lambda_2 =\| \mathbb{B}_2\|_{op}$. Assume $n \varrho^2 \ge 3 p_Q$.
Also, let $\Lambda = |\Lambda_1| + |\Lambda_2|$, define
$\mathbb{B} \myeq \left({ \frac{\Lambda_1}{\Lambda}}\mathbb{B}_1 + \frac{\Lambda_2}{\Lambda} \mathbb{B}_2  \right).
$
Assume $g^2 \ge 3 p_Q$, where $p_Q$ is defined in \ref{eq:pQ} with $QQ^T = \mathbb{B}$.
If $\Lambda_1 < n(2\lambda_1)^{-1}, \Lambda_2 < n(2\lambda_2)^{-1}$ satisfy $C_X |\Lambda_1| < 1/6$ and $C_X |\Lambda_2| < 1/6$, it holds 
\begin{equation}
   \mathbb{E} \exp \left( \frac{i \Lambda_1}{2n} \mathbb{Y}^T \mathbb{B}_1 \mathbb{Y} + \frac{i \Lambda_2}{2n} \mathbb{Y}^T \mathbb{B}_2 \mathbb{Y} \right)
   \approx \det  \left(\mathbb{I}_{nd}  - i \Lambda_1 \mathbb{B}_1/n -i \Lambda_2 \mathbb{B}_2/n \right)^{-1/2}, 
\end{equation}
under $p_Q \gg 1$, and $ p_Q^2 \ll n$.
\end{corollary}

\begin{proof}
Similar to Theorem \ref{Theorem:2CfSubGauss}, with $\Lambda_1 < n(2 \lambda_1)^{-1}$ and $\Lambda_2 < n(2 \lambda_2)^{-1}$ so that $n^{-1}\Lambda \|\mathbb{B}\|_{op} <1$ the integral in the characteristic function converges. We then rewrite the quadratic form in terms of $\Lambda$ and $\mathbb{B}$,
    \begin{align*}
    \frac{\Lambda_1}{2} \mathbb{Y}^T \mathbb{B}_1 \mathbb{Y}   + \frac{\Lambda_2}{2} \mathbb{Y}^T \mathbb{B}_2 \mathbb{Y}  = \mathbb{Y}^T( \Lambda_1 \mathbb{B}_1 + \Lambda_2 \mathbb{B}_2) \mathbb{Y} = \Lambda \mathbb{Y}^T (\mathbb{B}) \mathbb{Y},
    \end{align*}
With $C_X |\Lambda| < 1/3$, we apply Theorem \ref{Theorem:nd2criteria}. It holds for $\Lambda_1 < n(2 \lambda_1)^{-1}$ and $\Lambda_2 < n(2 \lambda_2)^{-1}$ that 
    \begin{align*}
    \mathbb{E}&\exp \left( \frac{i \Lambda_1}{2n} \mathbb{Y}^T \mathbb{B}_1 \mathbb{Y} + \frac{i \Lambda_2}{2n} \mathbb{Y}^T \mathbb{B}_2 \mathbb{Y} \right)
    = \mathbb{E} \exp \left( \frac{i \Lambda_1}{2n}  \Lambda \mathbb{Y}^T (\mathbb{B}) \mathbb{Y} \right)\\
    &\approx \det  \left(\mathbb{I}_{nd}  - i \Lambda \mathbb{B} \right)^{-1/2}\\
    &\approx \det  \left(\mathbb{I}_{nd}  - i \Lambda_1 \mathbb{B}_1/n -i \Lambda_2 \mathbb{B}_2/n \right)^{-1/2}, 
    \end{align*}
 under $p_Q \gg 1$, and $ p_Q^2 \ll n$.
\end{proof}
Therefore, under high-dimensional condition $p_Q \gg 1$, and criteria $ {p_Q}^2 \ll n$, the joint characteristic function of the quadratic forms approximates that of the Gaussian case (\ref{eq:GaussianFormＡ}).
We now check the Gaussian approximation of the GSR test statistics.

\textbf{Proof of Corollary \ref{corollary:F-approx}}\label{Proof:F-approx}
\begin{corollary}[F-distribution approximation of the graph-spanning ratio]
  Let $Y_1, \ldots, Y_n \in \mathbb{R}^d$ be centered i.i.d. random vectors satisfying $\mathbb{E} Y_i=0$ and $\Var(Y_i)\le \mathbb{I_d}$, and the \hyperlink{SubGaussian}{sub-Gaussian condition}. Let $\mathbb{Y}$ be defined as in Equation \ref{eq:baseVector}. Let $\|W_{G_1}\|^2 = \mathbb{Y}^T \mathbb{B}_1 \mathbb{Y}$ and $\|W_{G_1}\|^2 = \mathbb{Y}^T \mathbb{B}_2 \mathbb{Y}$, where $\mathbb{B}_1$, and $\mathbb{B}_2$ are positive definite matrices. Denote $\lambda_1 =\| \mathbb{B}_1 \|_{op}$, and $\lambda_2 =\| \mathbb{B}_2\|_{op}$. Assume $n \varrho^2 \ge 3 d$. If $\Lambda_1 < n(2\lambda_1)^{-1}, \Lambda_2 < n(2\lambda_2)^{-1}$ satisfy $C_Y |\Lambda_1| < 1/6$ and $C_X |\Lambda_2| < 1/6$, then it holds
   \begin{equation}
      \mathbb{E} \exp \left( \frac{i \Lambda}{2n} \|W_{G_1}\|^2+ \frac{i \Lambda}{2n} \|W_{G_2}\|^2\right) \approx  \mathbb{E} \left[ \exp \left(\frac{i\Lambda_1}{2n} \Gamma^T \mathbb{B}_1 \Gamma + \frac{i\Lambda_2}{2n} \Gamma^T \mathbb{B}_2 \Gamma  \right)  \right], 
   \end{equation}
  under $d \gg 1$, and $ d^2 \ll n$.
\end{corollary}
\begin{proof}
We define a positive definite matrix $\mathbb{B}$ as in Corollary \ref{Corollary:ndcriteria2}.
According to Corollary \ref{corollary:WGn}, ${p_Q}_1=d$ and ${p_Q}_2=d$ for the graph structure. Therefore, $p_Q=d$ for matrix $QQ^T = \mathbb{B}$. 
Then by Corollary \ref{corollary:WGn}, the equation holds under $d \gg 1$, and $ d^2 \ll n$.
\end{proof}

Therefore by contraction, under the sub-Gaussian condition, $d \gg 1$ and $d^2 \ll n$, the distribution of the GSR test statistic approximates the distribution of the ratio of Gaussian quadratic forms.

\end{appendix}
\begin{funding}
KP acknowledges the partial support of the project PNRR - M4C2 - Investimento 1.3, Partenariato Esteso PE00000013 - ``FAIR - Future Artificial Intelligence Research'' - Spoke 1 ``Human-centered AI'', funded by the European Commission under the NextGeneration EU programme. 
\end{funding}

\bibliographystyle{plainnat}
\bibliography{ref}

\begin{thebibliography}{30}
\providecommand{\natexlab}[1]{#1}
\providecommand{\url}[1]{\texttt{#1}}
\expandafter\ifx\csname urlstyle\endcsname\relax
  \providecommand{\doi}[1]{doi: #1}\else
  \providecommand{\doi}{doi: \begingroup \urlstyle{rm}\Url}\fi

\bibitem[Aminikhanghahi and Cook(2017)]{aminikhanghahi2017survey}
Samaneh Aminikhanghahi and Diane~J Cook.
\newblock A survey of methods for time series change point detection.
\newblock \emph{Knowledge and information systems}, 51\penalty0 (2):\penalty0 339--367, 2017.

\bibitem[Arlot et~al.(2019)Arlot, Celisse, and Harchaoui]{arlot2019kernel}
Sylvain Arlot, Alain Celisse, and Zaid Harchaoui.
\newblock A kernel multiple change-point algorithm via model selection.
\newblock \emph{Journal of machine learning research}, 20\penalty0 (162):\penalty0 1--56, 2019.

\bibitem[Avanesov et~al.(2018)Avanesov, Buzun, et~al.]{avanesov2018change}
Valeriy Avanesov, Nazar Buzun, et~al.
\newblock Change-point detection in high-dimensional covariance structure.
\newblock \emph{Electronic Journal of Statistics}, 12\penalty0 (2):\penalty0 3254--3294, 2018.

\bibitem[Baraud(2002)]{baraud2002non}
Yannick Baraud.
\newblock Non-asymptotic minimax rates of testing in signal detection.
\newblock \emph{Bernoulli}, pages 577--606, 2002.

\bibitem[Baraud et~al.(2003)Baraud, Huet, and Laurent]{baraud2003adaptive}
Yannick Baraud, Sylvie Huet, and B{\'e}atrice Laurent.
\newblock Adaptive tests of linear hypotheses by model selection.
\newblock \emph{The Annals of Statistics}, 31\penalty0 (1):\penalty0 225--251, 2003.

\bibitem[Baringhaus and Gaigall(2017)]{baringhaus2017hotelling}
Ludwig Baringhaus and Daniel Gaigall.
\newblock Hotelling’s t2 tests in paired and independent survey samples: An efficiency comparison.
\newblock \emph{Journal of Multivariate Analysis}, 154:\penalty0 177--198, 2017.

\bibitem[Birg{\'e}(2001)]{birge2001alternative}
Lucien Birg{\'e}.
\newblock An alternative point of view on lepski's method.
\newblock \emph{Lecture Notes-Monograph Series}, pages 113--133, 2001.

\bibitem[Boracchi et~al.(2018)Boracchi, Carrera, Cervellera, and Maccio]{boracchi2018quanttree}
Giacomo Boracchi, Diego Carrera, Cristiano Cervellera, and Danilo Maccio.
\newblock Quanttree: Histograms for change detection in multivariate data streams.
\newblock In \emph{International Conference on Machine Learning}, pages 639--648. PMLR, 2018.

\bibitem[Chen et~al.(2015)Chen, Zhang, et~al.]{chen2015graph}
Hao Chen, Nancy Zhang, et~al.
\newblock Graph-based change-point detection.
\newblock \emph{The Annals of Statistics}, 43\penalty0 (1):\penalty0 139--176, 2015.

\bibitem[Chen and Gupta(2011)]{chen2011parametric}
Jie Chen and Arjun~K Gupta.
\newblock \emph{Parametric statistical change point analysis: with applications to genetics, medicine, and finance}.
\newblock Springer Science \& Business Media, 2011.

\bibitem[Enikeeva and Harchaoui(2019)]{enikeeva2019high}
Farida Enikeeva and Zaid Harchaoui.
\newblock High-dimensional change-point detection under sparse alternatives.
\newblock \emph{The Annals of Statistics}, 47\penalty0 (4):\penalty0 2051--2079, 2019.

\bibitem[Friedman and Rafsky(1979)]{friedman1979multivariate}
Jerome~H Friedman and Lawrence~C Rafsky.
\newblock Multivariate generalizations of the wald-wolfowitz and smirnov two-sample tests.
\newblock \emph{The Annals of Statistics}, pages 697--717, 1979.

\bibitem[Girshick and Rubin(1952)]{girshick1952bayes}
Meyer~A Girshick and Herman Rubin.
\newblock A bayes approach to a quality control model.
\newblock \emph{The Annals of mathematical statistics}, 23\penalty0 (1):\penalty0 114--125, 1952.

\bibitem[Grundy et~al.(2020)Grundy, Killick, and Mihaylov]{grundy2020high}
Thomas Grundy, Rebecca Killick, and Gueorgui Mihaylov.
\newblock High-dimensional changepoint detection via a geometrically inspired mapping.
\newblock \emph{Statistics and Computing}, 30\penalty0 (4):\penalty0 1155--1166, 2020.

\bibitem[Harchaoui et~al.(2009)Harchaoui, Moulines, and Bach]{harchaoui2009kernel}
Zaid Harchaoui, Eric Moulines, and Francis~R Bach.
\newblock Kernel change-point analysis.
\newblock In \emph{Advances in neural information processing systems}, pages 609--616, 2009.

\bibitem[James et~al.(1992)James, James, and Siegmund]{james1992asymptotic}
Barry James, Kang~Ling James, and David Siegmund.
\newblock Asymptotic approximations for likelihood ratio tests and confidence regions for a change-point in the mean of a multivariate normal distribution.
\newblock \emph{Statistica Sinica}, pages 69--90, 1992.

\bibitem[Kirch(2008)]{kirch2008bootstrapping}
Claudia Kirch.
\newblock Bootstrapping sequential change-point tests.
\newblock \emph{Sequential Analysis}, 27\penalty0 (3):\penalty0 330--349, 2008.

\bibitem[Kuleshov et~al.(2016)Kuleshov, Jiang, Zhou, Jahanbani, Batzoglou, and Snyder]{kuleshov2016synthetic}
Volodymyr Kuleshov, Chao Jiang, Wenyu Zhou, Fereshteh Jahanbani, Serafim Batzoglou, and Michael Snyder.
\newblock Synthetic long read sequencing reveals the composition and intraspecies diversity of the human microbiome.
\newblock \emph{Nature biotechnology}, 34\penalty0 (1):\penalty0 64, 2016.

\bibitem[Liu et~al.(2021)Liu, Gao, and Samworth]{liu2021minimax}
Haoyang Liu, Chao Gao, and Richard~J Samworth.
\newblock Minimax rates in sparse, high-dimensional change point detection.
\newblock \emph{The Annals of Statistics}, 49\penalty0 (2):\penalty0 1081--1112, 2021.

\bibitem[Lorden(1971)]{lorden1971procedures}
Gary Lorden.
\newblock Procedures for reacting to a change in distribution.
\newblock \emph{The annals of mathematical statistics}, pages 1897--1908, 1971.

\bibitem[Page(1954)]{page1954continuous}
Ewan~S Page.
\newblock Continuous inspection schemes.
\newblock \emph{Biometrika}, 41\penalty0 (1/2):\penalty0 100--115, 1954.

\bibitem[Rosenbaum(2005)]{rosenbaum2005exact}
Paul~R Rosenbaum.
\newblock An exact distribution-free test comparing two multivariate distributions based on adjacency.
\newblock \emph{Journal of the Royal Statistical Society: Series B (Statistical Methodology)}, 67\penalty0 (4):\penalty0 515--530, 2005.

\bibitem[Shiryaev(1961)]{shiryaev1961problem}
Albert~Nikolaevich Shiryaev.
\newblock The problem of quickest detection of a violation of stationary behavior.
\newblock In \emph{Doklady Akademii Nauk}, volume 138, pages 1039--1042. Russian Academy of Sciences, 1961.

\bibitem[Siegmund et~al.(2011)Siegmund, Yakir, and Zhang]{siegmund2011detecting}
David Siegmund, Benjamin Yakir, and Nancy~R Zhang.
\newblock Detecting simultaneous variant intervals in aligned sequences.
\newblock \emph{The Annals of Applied Statistics}, pages 645--668, 2011.

\bibitem[Spokoiny(2009)]{spokoiny2009multiscale}
Vladimir Spokoiny.
\newblock Multiscale local change point detection with applications to value-at-risk.
\newblock \emph{The Annals of Statistics}, 37\penalty0 (3):\penalty0 1405--1436, 2009.

\bibitem[Spokoiny(2023)]{spokoiny2023concentration}
Vladimir Spokoiny.
\newblock Concentration of a high dimensional sub-gaussian vector.
\newblock \emph{arXiv preprint arXiv:2305.07885}, 2023.

\bibitem[Spokoiny(1996)]{spokoiny1996adaptive}
Vladimir~G Spokoiny.
\newblock Adaptive hypothesis testing using wavelets.
\newblock \emph{The Annals of Statistics}, 24\penalty0 (6):\penalty0 2477--2498, 1996.

\bibitem[Wang and Samworth(2018)]{wang2018high}
Tengyao Wang and Richard~J Samworth.
\newblock High dimensional change point estimation via sparse projection.
\newblock \emph{Journal of the Royal Statistical Society: Series B (Statistical Methodology)}, 80\penalty0 (1):\penalty0 57--83, 2018.

\bibitem[Wellner et~al.(2013)]{wellner2013weak}
Jon Wellner et~al.
\newblock \emph{Weak convergence and empirical processes: with applications to statistics}.
\newblock Springer Science \& Business Media, 2013.

\bibitem[Zhilova(2022)]{zhilova2022new}
Mayya Zhilova.
\newblock New edgeworth-type expansions with finite sample guarantees.
\newblock \emph{The Annals of Statistics}, 50\penalty0 (5):\penalty0 2545--2561, 2022.

\end{thebibliography}

\end{document}